\documentclass[11pt]{article}
\usepackage{fullpage}

\usepackage[T1]{fontenc}
\usepackage{kpfonts}
\usepackage{libertine}
\usepackage[scaled=0.85]{beramono}

\usepackage[titletoc,page]{appendix}

\usepackage{dirtytalk}

\usepackage{graphicx}
\usepackage{amsmath,amsfonts,amscd, amssymb}
\usepackage{amsthm,thmtools,enumerate,caption}
\usepackage{xspace}
\usepackage{setspace}
\usepackage{perpage}
\usepackage{fancyhdr}
\usepackage{sectsty}
\usepackage{bbm}

\usepackage{algorithm, algorithmicx, algpseudocode}
\usepackage{varwidth}

\let\oldReturn\Return
\renewcommand{\Return}{\State\oldReturn}

\usepackage{booktabs}
\usepackage[caption=false,font=footnotesize]{subfig}

\usepackage[usenames,dvipsnames,svgnames,table]{xcolor}
\definecolor{darkgreen}{rgb}{0.0,0,0.9}

\usepackage[colorlinks=true,pdfpagemode=UseNone,citecolor=OliveGreen,linkcolor=black,urlcolor=BrickRed,
pagebackref]{hyperref}

\newtheorem{theorem}{Theorem}

\newtheorem{lemma}{Lemma}

\newtheorem{proposition}[theorem]{Proposition}

\newtheorem{corollary}[theorem]{Corollary}
\theoremstyle{definition}
\newtheorem{definition}{Definition}
\newtheorem{problem}{Problem}
\newtheorem{assumption}[theorem]{Assumption}

\newtheorem{remark}{Remark}

\usepackage[numbers]{natbib}
\bibpunct{(}{)}{;}{}{}{,}
\usepackage{multicol}

\newcommand{\Acal}{\mathcal{A}}
\newcommand{\Bcal}{\mathcal{B}}

\newcommand{\Scal}{\mathcal{S}}

\DeclareMathOperator*{\argmax}{arg\,max}
\newcommand{\diag}{\mathop{\mathrm{diag}}}
\newcommand{\rad}{\mathop{\mathrm{rad}}}
\newcommand{\m}{\mathop{\mathrm{m}}}
\newcommand{\dBm}{\mathop{\mathrm{dBm}}}

\newcommand{\EV}[1]{\mathbb{E}[#1]} 
\newcommand{\Var}[1]{\mathbb{V}[#1]} 

\algrenewcommand\textproc{}%


\usepackage{enumitem}

\title{\vspace{1cm}\bf Sampling-based Incremental Information Gathering with Applications to Robotic Exploration and Environmental Monitoring\thanks{Submitted to IJRR. 
    \url{mani.ghaffari@gmail.com} -- \url{http://maanighaffari.com}}}
\author{Maani Ghaffari Jadidi\thanks{College of Engineering, \mbox{University of Michigan}, Ann Arbor, MI 48109 USA} \and
Jaime Valls Miro\thanks{Centre for Autonomous Systems (CAS), University
of Technology Sydney, NSW, Australia} \and  Gamini Dissanayake\footnotemark[3]}
\begin{document}

\maketitle
\thispagestyle{empty}
\vspace{1.5cm}
\begin{abstract}
In this article, we propose a sampling-based motion planning algorithm equipped with an information-theoretic convergence criterion for incremental informative motion planning. The proposed approach allows dense map representations and incorporates the full state uncertainty into the planning process. The problem is formulated as a constrained maximization problem. Our approach is built on rapidly-exploring information gathering algorithms and benefits from advantages of sampling-based optimal motion planning algorithms. We propose two information functions and their variants for fast and online computations. We prove an information-theoretic convergence for an entire exploration and information gathering mission based on the least upper bound of the average map entropy. A natural automatic stopping criterion for information-driven motion control results from the convergence analysis. We demonstrate the performance of the proposed algorithms using three scenarios: comparison of the proposed information functions and sensor configuration selection, robotic exploration in unknown environments, and a wireless signal strength monitoring task in a lake from a publicly available dataset collected using an autonomous surface vehicle.
\end{abstract}
\clearpage
\thispagestyle{empty}
\tableofcontents
\clearpage

\section{Introduction}
\label{intro}

Exploration in unknown environments is a major challenge for an autonomous robot and has numerous present and emerging applications ranging from search and rescue operations to space exploration programs. While exploring an unknown environment, the robot is often tasked to monitor a quantity of interest through a cost or an information quality measure~\citep{dhariwal2004bacterium,singh2010modeling,marchant2012bayesian,dunbabin2012robots,lan2013planning,yu2015persistent}. Robotic exploration algorithms usually rely on geometric frontiers~\citep{yamauchi1997frontier,strom2015predictive} or visual targets~\citep{kim2014active} as goals to solve the planning problem using geometric/information gain-based \emph{greedy} action selection or planning for a limited horizon. The main drawback of such techniques is that a set of targets constrains the planner search space to the paths that start from the current robot pose to targets. In contrast, a more general class of robotic navigation problem known as robotic information gathering~\citep{singh2009efficient,binney2012branch,binney2013optimizing}, and in particular the Rapidly-exploring Information Gathering (RIG)~\citep{hollinger2013sampling,hollinger2014sampling} technique exploits a sampling-based planning strategy to calculate the cost and information gain while searching for traversable paths in the entire space. This approach differs from classic path planning problems as there is no goal to be found. Therefore, a solely information-driven robot control strategy can be formulated. Another important advantage of RIG methods is their \emph{multi-horizon planning} nature through representing the environment by an incrementally built graph that considers both the information gain and cost.

Rapidly-exploring information gathering algorithms are suitable for non-myopic robotic exploration techniques. However, the developed methods are not online, the information function calculation is often a bottleneck, and they do not offer an automated convergence criterion, i.e.{\@} they are \emph{anytime}~\footnote{Being anytime is typically a good feature, but we are interested in knowing when the algorithm converges.}. To be able to employ RIG for online navigation tasks, we propose an Incrementally-exploring Information Gathering (IIG) algorithm built on the RIG. In particular, the following developments are performed:

\begin{figure}[t]
  \centering 
  \includegraphics[width=0.95\columnwidth,trim={0cm 0cm 0cm 0cm},clip]{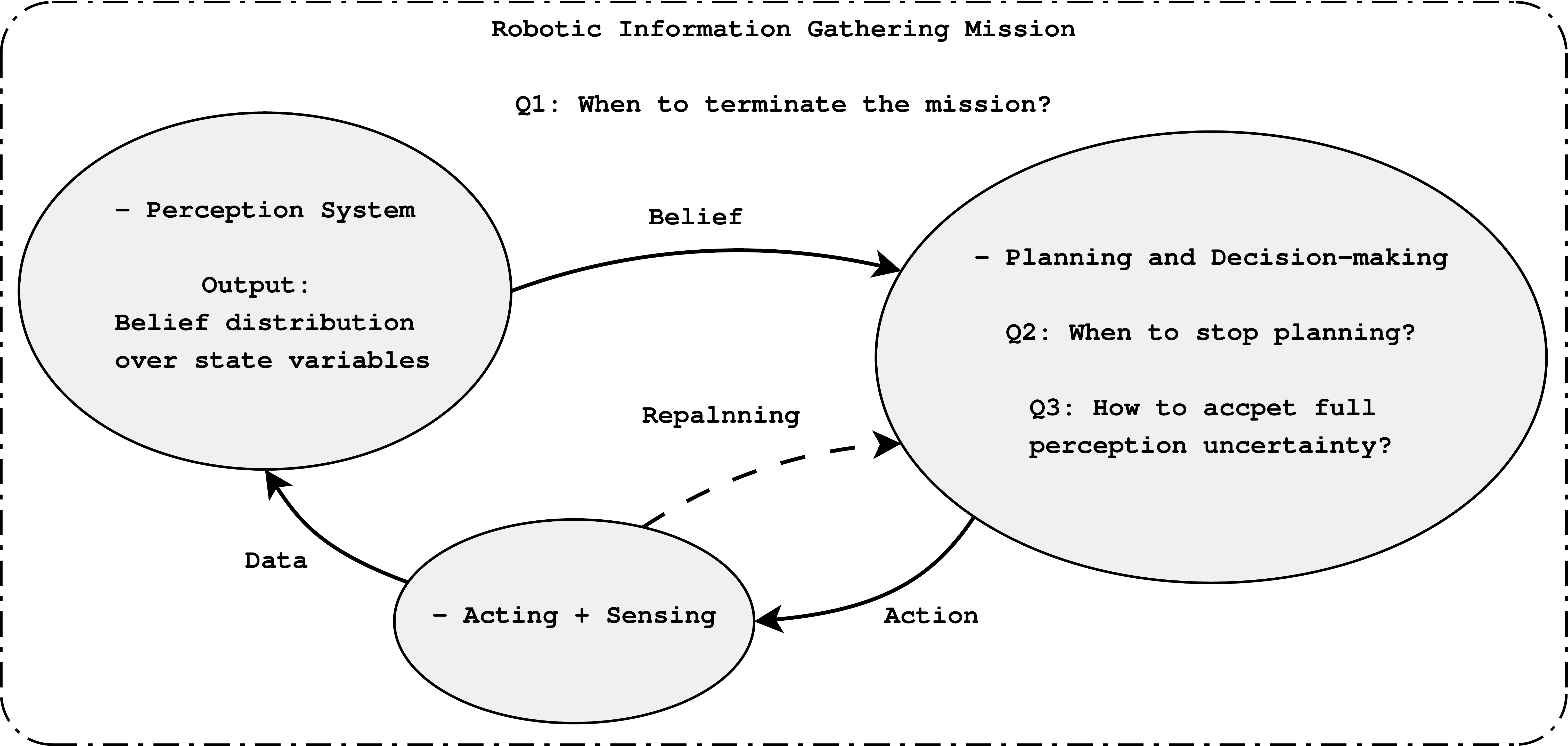}
  \caption{The abstract system for robotic information gathering. The three important questions that we try to answer in this work are asked (Q1-Q3). Note that the diagram is conceptual and, in practice, different modules can share many attributes and methods such as parameters and sensor models. The replanning problem is not studied in this work, that is the robot remains commited to the planned action until it enters the planning state again.}
  \label{fig:exploration}
\end{figure}

\begin{enumerate}
 \item[(i)] We allow the full state including the robot pose and a dense map to be partially observable.
 \item[(ii)] We develop a convergence criterion based on relative information contribution. This convergence criterion is a necessary step for incremental planning to make it possible for the robot to execute planned actions autonomously.
 \item[(iii)] We propose two classes of information functions that can approximate the information gain for online applications. The algorithmic implementation of the proposed functions is also provided.
 \item[(iv)] We develop a heuristic algorithm to extract the most informative trajectories from a RIG/IIG tree. This algorothm is necessary to extract an exutable trajectory (action) from RIG/IIG graphs (trees).
 \item[(v)] We prove an information-theoretic automatic stopping criterion for the entire mission based on the least upper bound of the average map (state variables) entropy.
 \item[(vi)] We provide results in batch and incremental experiments as well as in publicly available robotic datasets and discuss potential applications of the developed algorithms.
\end{enumerate}

Figure~\ref{fig:exploration} shows a conceptual illustration of the problem studied in this article. The three questions that we try to address are: Q1) When to terminate the entire information gathering mission? Q2) When to stop planning by automatically detecting the convergence of the planner, i.e. planning and acting incrementally? And Q3) How to incorporate the uncertainty of the perception system into the planner?

This article is organized as follows. In the next section, the related work is discussed. Section~\ref{sec:prelim} includes required preliminaries. In Section~\ref{sec:rigprob}, we present the problem definition and highlight the differences between RIG and IIG algorithms. We also explain RIG algorithms to establish the required basis. In Section~\ref{sec:incrig}, we present the novel IIG algorithm. In Section~\ref{sec:infofunc}, we propose two information functions that approximate the information quality of nodes and trajectories of the associated motion planning problem. In Section~\ref{sec:pathselect}, the problem of path extraction and selection is explained, and a heuristic algorithm is proposed to solve the problem. Section~\ref{sec:theoretic} presents the proof for an information-theoretic automatic stopping criterion for exploration and information gathering missions. In Section~\ref{sec:iigresults}, we present an extensive evaluation of the proposed strategies including a comparison with relevant techniques in the literature. A lake monitoring scenario based on a publically available dataset is also presented together with a discussion on the limitations of this work. Finally, Section~\ref{sec:conclusion} concludes the article and discusses possible extensions of the proposed algorithms as future work.

\section{Related work}
\label{sec:literature}
Motion planning algorithms~\citep{latombe1991robot,lavalle2006planning} construct a broad area of research in the robotic community. In the presence of uncertainty, where the state is not fully observable, measurements and actions are stochastic. The sequential decision-making under uncertainty, in the most general form, can be formulated as a Partially Observable Markov Decision Processes~(POMDP) or optimal control with imperfect state information~\citep{astrom1965optimal,smallwood1973optimal,bertsekas1995dynamic,kaelbling1998planning}. Unfortunately, when the problem is formulated using a dense belief representation, a general purpose POMDP solver is not a practical choice~\citep{binney2013optimizing}.

The sampling-based motion planning algorithms~\citep{horsch1994motion,kavraki1996probabilistic,lavalle2001randomized,bry2011rapidly,karaman2011sampling,lan2013planning} have proven successful applications in robotics. An interesting case is where the environment is fully observable (known map) and the objective is to make sequential decisions under motion and measurement uncertainty. This problem is also known as \emph{active localization}; for a recent technique to solve such a problem see~\citet[and references therein]{Aghamohammadi01022014}. A closely related term that is used in the literature is \emph{belief space planning~\citep{kurniawati2008sarsop,huynh2009iclqg,Prentice01112009,platt2010belief,kurniawati2011motion,van2011lqg,bry2011rapidly,valencia2013planning}.} In this series of works, the assumption of a known map can be relaxed and the environment is often represented using a set of features. The objective is to find a trajectory that minimizes the state uncertainty (the total cost) with respect to a fixed or variable (and bounded) planning horizon that can be set according to the \emph{budget}~\citep{indelman2015planning}.

In the context of feature-based representation of the environment, planning actions for Simultaneous Localization And Mapping (SLAM) using Model Predictive Control (MPC) is studied in~\citet{leung2006active}. To enable the robot to explore, a set of ``attractors'' are defined to resemble informative regions for exploration. In particular, it is concluded as expected that a multi-step (three steps) look-ahead MPC planner outperforms the greedy approach. In~\citet{atanasov2014information}, assuming the sensor dynamics is linear-Gaussian, the stochastic optimal control problem is reduced to a deterministic optimal control problem which can be solved offline. The deterministic nature of the problem has been exploited to solve it using forward value iteration~\citep{bertsekas1995dynamic,le2009trajectory}; furthermore, it is shown that the proposed solution, namely reduced value iteration, has a lower computational complexity than that of forward value iteration and its performance is better than the greedy policy. The work is also extended from a single robot to decentralized \emph{active information acquisition}, and using attractors as dummy exploration landmarks, i.e. frontiers, it has been successfully applied to the \emph{active SLAM} problem~\citep{atanasov2015decentralized}.

Another approach to study the problem of robotic navigation under uncertainty, is known as informative motion planning or robotic information gathering~\citep{leung2006planning,singh2009efficient,levine2010information,binney2012branch,binney2013optimizing,hollinger2013sampling,hollinger2014sampling,atanasov2015active}. Rapidly-exploring information gathering~\citep{hollinger2014sampling}, is a technique that solves the problem of informative motion planning using incremental sampling from the workspace and by \emph{partial ordering} of nodes builds a graph that contains informative trajectories.

In the problem we study in this article, the state variables consist of the robot trajectory and a dense map of the environment.  We build on rapidly-exploring information gathering technique by considering both the robot pose and the map being partially observable. We also develop an information-theoretic criterion for the convergence of the search which allows us to perform online robotic exploration in unknown environments. The Rapidly-exploring Adaptive Search and Classification (ReASC)~\citep{hollinger2015long} also improves on the rapidly-exploring information gathering by allowing for real-time optimization of search and classification objective functions. However, ReASC relies on discrete target locations and, similar to~\cite{platt2010belief}, resorts to the maximum likelihood assumption for future observations.

The technique reported in~\citet{charrow2015information,charrow2015informationrss} is closely related to this work but assumes that the robot poses are known, i.e. fully observable, to solve the problem of information gathering for occupancy mapping with the help of geometric frontiers~\citep{yamauchi1997frontier,strom2015predictive}. The computational performance of the information gain estimation is increased by using Cauchy-Schwarz Quadratic Mutual Information (CSQMI). It is shown that the behavior of CSQMI is similar to that of mutual information while it can be computed faster~\citep[Subsection 6.2.2 and Figure 6.1]{charrow2015informationphd}. It is argued in~\citet[Subsection V.C]{charrow2015informationrss} that: \say{\emph{It is interesting that the human operator stopped once they believed they had obtained a low uncertainty map and that all autonomous approaches continue reducing the map's entropy beyond this point, as they continue until no frontiers are left. However, the final maps are qualitatively hard to differentiate, suggesting a better termination condition is needed.}} We agree with this argument and relax such a constraint by exploiting a sampling-based planning strategy to calculate the cost and information gain while searching for traversable paths in the entire space. We also prove an information-theoretic automatic stopping criterion for the entire mission which can alleviate this issue. In Subsection~\ref{subsec:oniigres}, we conduct robotic exploration experiments to show the effectiveness of the proposed termination condition. We show that the proposed condition enables the robot to produce comparable results while collecting less measurements, but sufficient information for the inference.

In particular, the main features of this work that differentiate the present approach from the literature mentioned above can be summarized as follows. 
\begin{itemize}
 \item[(i)] We allow for dense belief representations. Therefore, we incorporate full state uncertainty, i.e.{\@} the robot pose and the map, into the planning process. As a result, the robot behavior has a strong correlation with its perception uncertainty.
 \item[(ii)] We take into account all possible future observations and do not resort to maximum likelihood assumptions. Therefore, the randomness of future observations is addressed.
 \item[(iii)] We take into account both cost and information gain. The cost is included through a measure of distance, and the information gain quantifies the sensing quality and acting uncertainty. Therefore, the planning algorithm runs with respect to available sensing resources and acting limitations.
 \item[(iv)] We propose an information-theoretic notion of planning horizon which leads to an infinite-horizon planning technique and provides a general framework for incremental planning and acting.
 \item[(v)] We offer an automatic stopping criterion for the entire information gathering mission that can be easily incorporated in many available algorithms mentioned above.
\end{itemize}

\section{Preliminaries}
\label{sec:prelim}

In this section, we briefly explain the mathematical notation and required preliminaries throughout this article.

\subsection{Mathematical notation}
\label{subsec:notation}
In the present article probabilities and probability densities are not distinguished in general. Matrices are capitalized in bold, such as in $\boldsymbol X$, and vectors are in lower case bold type, such as in $\boldsymbol x$. Vectors are column-wise and $1\colon n$ means integers from $1$ to $n$. The Euclidean norm is shown by $\lVert \cdot \rVert$. $\lvert \boldsymbol X \rvert$ denotes the determinant of matrix $\boldsymbol X$. For the sake of compactness, random variables, such as $X$, and their realizations, $x$, are sometimes denoted interchangeably where it is evident from context. $x^{[i]}$ denotes a reference to the $i$-th element of the variable. An alphabet such as $\mathcal{X}$ denotes a set, and the cardinality of the set is denoted by $\lvert \mathcal{X} \rvert$. A subscript asterisk, such as in $\boldsymbol x_*$, indicates a reference to a test set quantity. The $n$-by-$n$ identity matrix is denoted by $\boldsymbol I_{n}$. $\mathrm{vec}(x^{[1]},\dots,x^{[n]})$ denotes a vector such as $\boldsymbol x$ constructed by stacking $x^{[i]}$, $\forall i \in \{1\colon n\}$. The function notation is overloaded based on the output type and denoted by $k(\cdot)$, $\boldsymbol k(\cdot)$, and $\boldsymbol K(\cdot)$ where the outputs are scalar, vector, and matrix, respectively. Finally, $\EV\cdot$ and $\Var\cdot$ denote the expected value and variance of a random variable, respectively.

\subsection{Information theory}
\label{subsec:Information}
Entropy is a measure of the uncertainty of a random variable \citep{cover2012elements}. The entropy $H(X)$ of a discrete random variable $X$ is defined as $H(X) = \mathbb{E}_{p(x)}[\log \frac{1}{p(x)}] = -\sum_{\mathcal{X}} p(x) \log{p(x)}$ which implies that $H(X) \geq 0$. The joint entropy $H(X,Y)$ of discrete random variables $X$ and $Y$ with a joint distribution $p(x,y)$ is defined as $H(X,Y) = -\sum_{x\in\mathcal{X}} \sum_{y\in\mathcal{Y}} p(x,y) \log{p(x,y)}$. The conditional entropy is defined as $H(Y| X) = -\sum_{x\in\mathcal{X}} \sum_{y\in\mathcal{Y}} p(x,y) \log{p(y| x)}$.
\begin{theorem}[Chain rule for entropy~\citep{cover2012elements}]
\label{th:entchainrule}
 Let $X_1,X_2,...,X_n$ be drawn according to $p(x_1,x_2,...,x_n)$. Then
 \begin{equation}
  H(X_1,X_2,...,X_n) = \sum_{i=1}^n H(X_i|X_{i-1},...,X_1) 
 \end{equation}
\end{theorem}

The relative entropy or Kullback-Leibler Divergence (KLD) is a measure of distance between two distributions $p(x)$ and $q(x)$. It is defined as $D(p||q) = \mathbb{E}_{p(x)}[\log{\frac{p(x)}{q(x)}}]$.
\begin{theorem}[Information inequality~\citep{cover2012elements}]
\label{th:infoineq}
 Let $X$ be a discrete random variable. Let $p(x)$ and $q(x)$ be two probability mass functions. Then
 \begin{equation}
  D(p||q) \geq 0
\end{equation}
with equality if and only if $p(x) = q(x)\ \forall \ x$. 
\end{theorem}

The mutual information (MI), $I(X;Y) = D(p(x,y)|| p(x)p(y)) = H(X) - H(X|Y)$, is the reduction in the uncertainty of one random variable due to the knowledge of the other.
\begin{corollary}[Nonnegativity of mutual information]
 For any two random variables $X$ and $Y$,
 \begin{equation}
  I(X;Y) \geq 0
\end{equation}
with equality if and only if $X$ and $Y$ are independent.
\end{corollary}
\begin{proof}
 $I(X;Y) = D(p(x,y)|| p(x)p(y)) \geq 0$, with equality if and only if $p(x,y) = p(x)p(y)$.
\end{proof}

Some immediate consequences of the provided definitions are as follows.
\begin{theorem}[Conditioning reduces entropy]
\label{th:condentineq}
 For any two random variables $X$ and $Y$,
 \begin{equation}
  H(X|Y) \leq H(X)
\end{equation}
with equality if and only if $X$ and $Y$ are independent.
\end{theorem}
\begin{proof}
 $0 \leq I(X;Y) = H(X) - H(X|Y)$.
\end{proof}

We now define the equivalent of the functions mentioned above for probability density functions. Let $X$ be a continuous random variable whose support set is $\mathcal{S}$. Let $p(x)$ be the probability density function for $X$. The differential entropy $h(X)$ of $X$ is defined as $h(X) = -\int_{\mathcal{S}} p(x) \log{p(x)} dx$. Let $X$ and $Y$ be continuous random variables that have a joint probability density function $p(x,y)$. The conditional differential entropy $h(X|Y)$ is defined as $h(X|Y) = -\int p(x,y) \log{p(x|y)} dx dy$. The relative entropy (KLD) between two probability density functions $p$ and $q$ is defined as $D(p||q) = \int p \log{\frac{p}{q}}$.
The mutual information $I(X;Y)$ between two continuous random variables $X$ and $Y$ with joint probability density function $p(x,y)$ is defined as $I(X;Y) = \int p(x,y) \log{\frac{p(x,y)}{p(x)p(y)}}dx dy$.

\subsection{Submodular functions}
\label{subsec:submodular}
A set function $f$ is said to be \emph{submodular} if \mbox{$\forall \Acal \subseteq \Bcal \subseteq \Scal$} and $\forall s \in \Scal \backslash \Bcal$, then \mbox{$f(\Acal \cup s) - f(\Acal) \geq f(\Bcal \cup s) - f(\Bcal)$}. Intuitively, this can be explained as: by adding observations to a smaller set, we gain more information. The function $f$ has diminishing return. It is \emph{normalized} if $f(\varnothing) = 0$ and it is \emph{monotone} if $f(\Acal) \leq f(\Bcal)$. The mutual information is normalized, approximately monotone, and submodular~\citep{krause2008near}.

\subsection{Gaussian processes}
\label{subsec:GPs}

Gaussian Processes (GPs) are non-parametric Bayesian regression techniques that employ statistical inference to learn dependencies between points in a data set~\citep{rasmussen2006gaussian}. The joint distribution of the observed target values, $\boldsymbol y$, and the function values (the latent variable), $\boldsymbol f_*$, at the query points can be written as
\begin{equation}
\label{eq:gp_joint}
 \begin{bmatrix}
	\boldsymbol y \\
	\boldsymbol f_*
 \end{bmatrix} \sim \mathcal{N}(\boldsymbol 0,
 \begin{bmatrix}
	\boldsymbol K(\boldsymbol X,\boldsymbol X)+\sigma_n^2 \boldsymbol I_{n} & \boldsymbol K(\boldsymbol X,\boldsymbol X_*) \\
	\boldsymbol K(\boldsymbol X_*,\boldsymbol X)			& \boldsymbol K(\boldsymbol X_*,\boldsymbol X_*) 
 \end{bmatrix})
\end{equation}
where $\boldsymbol X$ is the $d\times n$ design matrix of aggregated input vectors $\boldsymbol x$, $\boldsymbol X_*$ is a $d\times n_*$ query points matrix, $\boldsymbol K(\cdot,\cdot)$ is the GP covariance matrix, and $\sigma_n^2$ is the variance of the observation noise which is assumed to have an independent and identically distributed (i.i.d.) Gaussian distribution. Define a training set \mbox{$\mathcal{D} = \{(\boldsymbol x^{[i]},y^{[i]}) \mid i=1\colon n\}$}. The predictive conditional distribution for a single query point \mbox{$f_*|\mathcal{D},\boldsymbol x_* \sim \mathcal{N}(\EV{f_*},\Var{f_*})$} can be derived as 
\begin{equation}
 \label{eq:gp_mean}
 \mu = \EV{f_*} = \boldsymbol k(\boldsymbol X,\boldsymbol x_*)^{T}[\boldsymbol K(\boldsymbol X,\boldsymbol X)+\sigma_n^2 \boldsymbol I_{n}]^{-1}\boldsymbol y
\end{equation}
\begin{equation}
\label{eq:gp_cov}
 \sigma = \Var{f_*} = k(\boldsymbol x_*,\boldsymbol x_*) - \boldsymbol k(\boldsymbol X,\boldsymbol x_*)^{T}[\boldsymbol K(\boldsymbol X,\boldsymbol X)+\sigma_n^2 \boldsymbol I_{n}]^{-1}\boldsymbol k(\boldsymbol X,\boldsymbol x_*)
\end{equation}

The hyperparameters of the covariance and mean function, $\boldsymbol\theta$, can be computed by minimization of the negative log of the marginal likelihood (NLML) function.
\begin{equation}
\label{eq:nlml}
	\log p(\boldsymbol y|\boldsymbol X,\boldsymbol\theta) = -\frac{1}{2}\boldsymbol y^{T}(\boldsymbol K(\boldsymbol X,\boldsymbol X)+\sigma_n^2 \boldsymbol I_{n})^{-1}\boldsymbol y -\frac{1}{2}\log \arrowvert K(\boldsymbol X,\boldsymbol X)+\sigma_n^2 \boldsymbol I_{n} \arrowvert-\frac{n}{2}\log 2\pi\
\end{equation}

\subsubsection{Covariance function}
\label{subsubsec:kerneldef}

Covariance functions are the main part of any GPs. We define a covariance function using the kernel definition as follows. 
\begin{definition}[Covariance function]
Let $\boldsymbol x \in \mathcal{X}$ and $\boldsymbol x' \in \mathcal{X}$ be a pair of inputs for a function
\mbox{$k:\mathcal{X} \times \mathcal{X} \rightarrow \mathbb{R}$} known as kernel.
A kernel is called a covariance function, as the case in Gaussian processes, if it is
symmetric, $k(\boldsymbol x,\boldsymbol x')=k(\boldsymbol x',\boldsymbol x)$, and positive semidefinite:
\begin{equation}
 \int k(\boldsymbol x,\boldsymbol x') f(\boldsymbol x) f(\boldsymbol x') d\mu(\boldsymbol x) d\mu(\boldsymbol x') \geq 0
\end{equation}
for all $f \in L_2(\mathcal{X},\mu)$. 
\end{definition}
Given a set of input points $\{\boldsymbol x^{[i]}|i=1:n\}$, a covariance matrix can be constructed using \mbox{$\boldsymbol K^{[i,j]}=k(\boldsymbol x^{[i]},\boldsymbol x^{[j]})$} as its entries.

\begin{figure}[t]
  \centering 
  \subfloat{\includegraphics[width=.3\columnwidth]{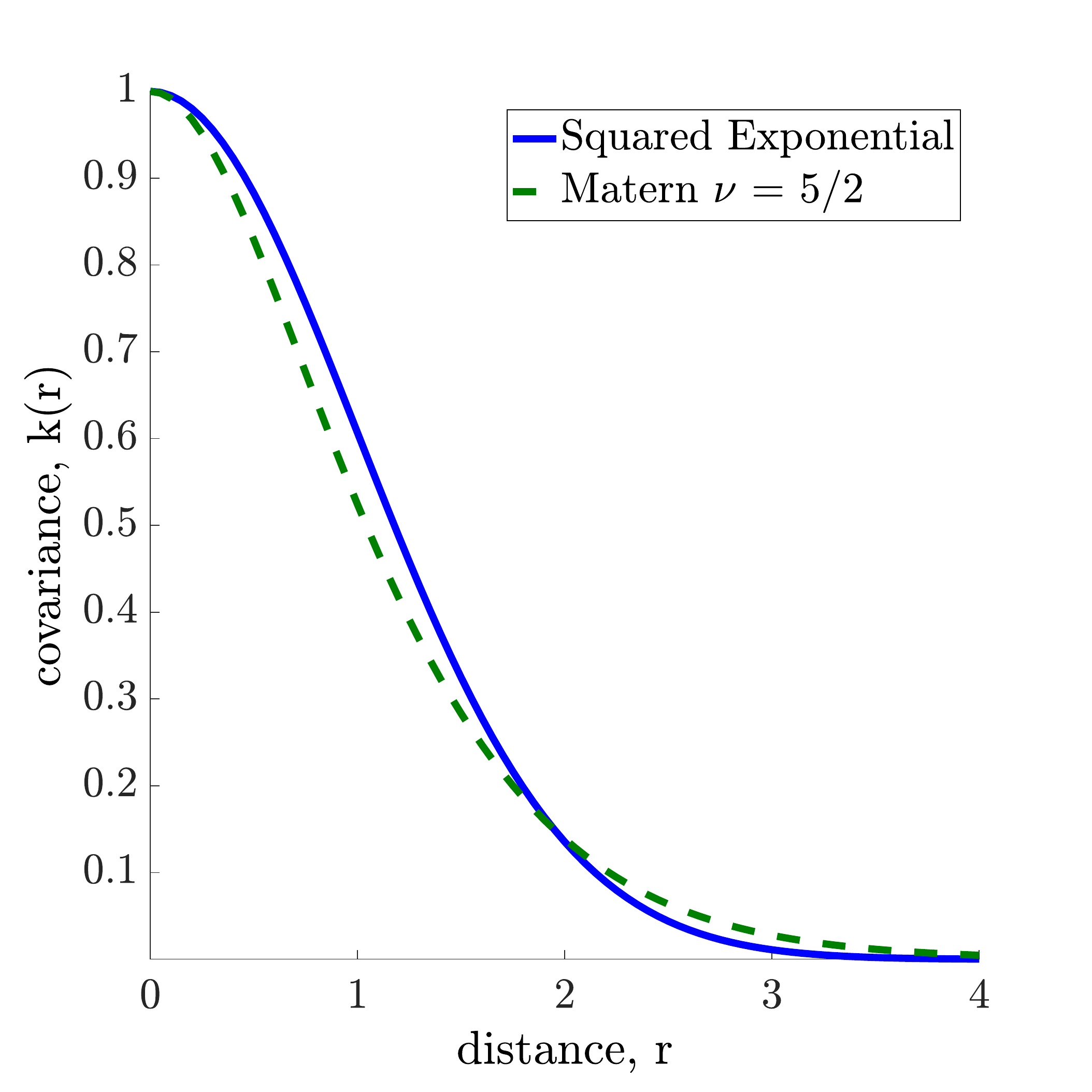}~
 }
  \subfloat{\includegraphics[width=.3\columnwidth]{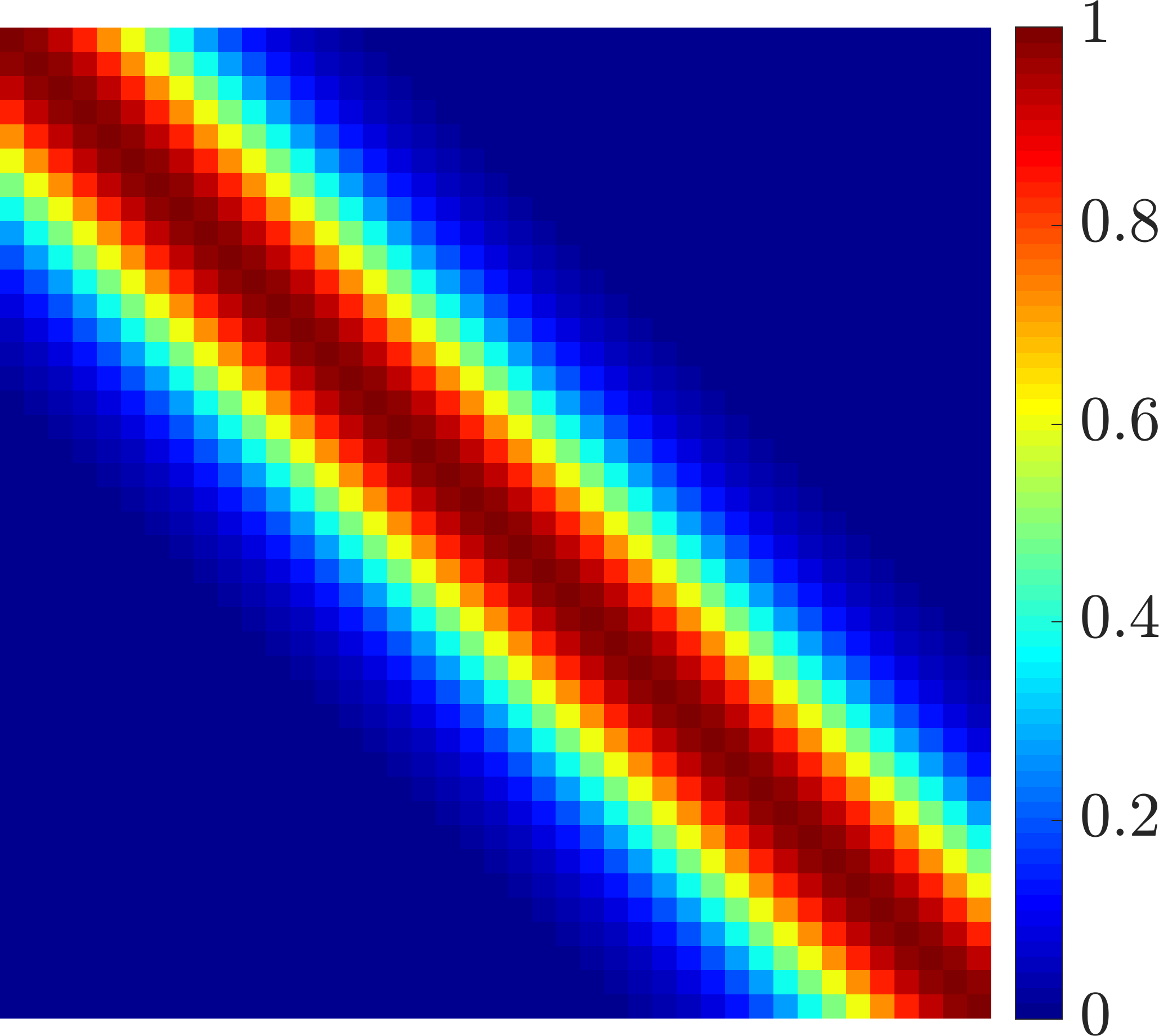}~
 }
  \subfloat{\includegraphics[width=.3\columnwidth]{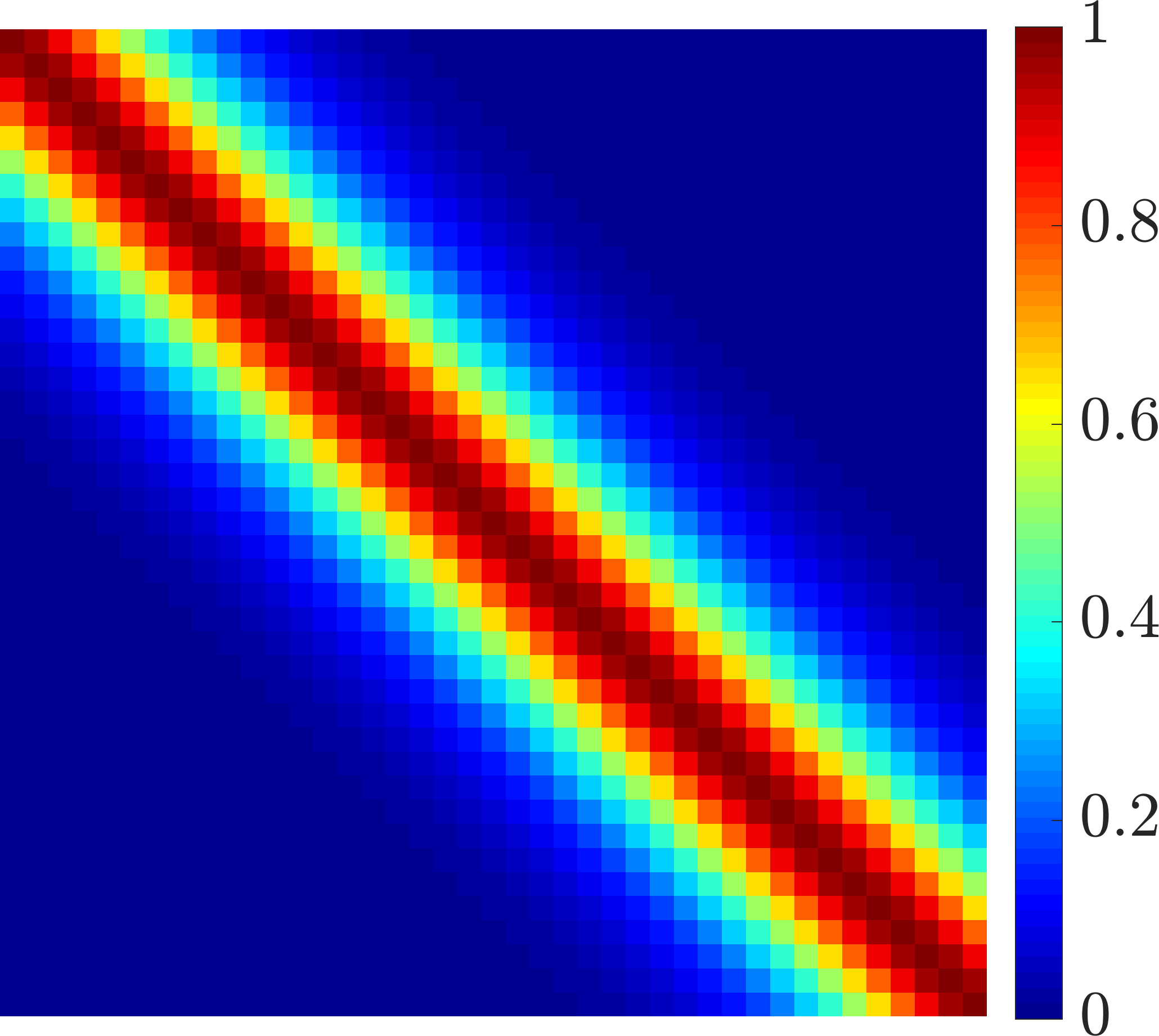}
 } 
  \caption{Illustrative examples of the SE and Mat\'ern ($\nu = 5/2$) covariance function as the distance parameter is increased from $0$ to $4$ and their corresponding function values in the kernel space, from the left respectively. The length-scale parameter is set to one.}
  \label{fig:kernels}
\end{figure}

\subsubsection{Useful kernels}
\label{subsubsec:kernelex}
The \emph{squared exponential} (SE) covariance function has the form $k(r) = \exp(-\frac{r^2}{2l^2})$ where $r=\lVert \boldsymbol x - \boldsymbol x_* \rVert$ is the distance between two input arguments of the covariance function and $l$ is the \emph{characteristic length-scale}. This covariance function is the most common kernel used in GPs and is infinitely differentiable. The Mat\'ern family of covariance functions \citep{stein1999interpolation} has proven powerful features to model structural correlations \citep{maani2014com,kim2015gpmap}. For a single query point $\boldsymbol x_*$ the function is given by 
\begin{equation}
\label{eq:Matern}
k(r) = \frac{1}{\Gamma(\nu) 2^{\nu-1}}\left[\frac{\sqrt{2\nu}r}{l}\right]^{\nu} K_{\nu}\left(\frac{\sqrt{2\nu}r}{l} \right)
\end{equation}
where $\Gamma(\cdot)$ is the Gamma function, $K_{\nu}(\cdot)$ is the modified Bessel function of 
the second kind of order $\nu$, $l$ is the characteristic length scale, 
and $\nu$ is a positive parameter used to control the smoothness of the covariance. In the limit for $\nu \rightarrow \infty$ this covariance function
converges to the SE kernel. 

Examples of the SE and Mat\'ern ($\nu = 5/2$) covariance functions as the distance parameter $r$ increases are shown in Figure~\ref{fig:kernels}. The functions are also plotted in kernel space.

\section{Problem statement}
\label{sec:rigprob}

The problem of robotic information gathering is formulated as a maximization problem subject to finite resources, i.e.{\@} a budget $b$.
In~\cite{hollinger2014sampling}, this problem is defined as follows.

\begin{definition}[Trajectory]
 Let $\mathcal{X}_f$ denotes the free workspace. A trajectory, $\mathcal{P} \in \mathcal{X}_f$, is a sequence of reachable points in which any two consecutive points are connected using a collision-free path and with respect to the robot motion constraints.
\end{definition}
\begin{problem}[Informative motion planning]
\label{prob:rig}
Let $\mathcal{A}$ be the space of all possible trajectories and $f_I(\mathcal{P})$ be a function that quantifies the information
quality along a trajectory $\mathcal{P}$. Let $f_c(\mathcal{P})$ be a function that returns the cost associated with trajectory $\mathcal{P}$.
Given the available budget $b$, the problem can be formulated as follows.
\begin{equation}
 \mathcal{P}^* =  \underset{\mathcal{P} \in \mathcal{A}}\argmax \ f_I(\mathcal{P}) \quad \text{s.t.} \ f_c(\mathcal{P}) \leq b 
\end{equation}
\end{problem}
Now we express the assumptions in RIG algorithms.
\begin{assumption}
 The cost function $f_c(\mathcal{P})$ is strictly positive, monotonically increasing, bounded, and additive such as distance and energy.
\end{assumption}

\begin{remark}
 The information function $f_I(\mathcal{P})$ can be modular, time-varying modular, or submodular.
\end{remark}
The information function assumption follows from~\cite{hollinger2014sampling}, even though we focus our attention on the submodular class of information 
functions as the information gathered at any future time during navigation depends on prior robot trajectories. Another reason to 
consider submodular information functions is to avoid information ``double-counting''. This allows us to develop an information-theoretic 
convergence criterion for RIG/IIG as the amount of available information remains 
bounded.
The following assumptions are directly from~\cite{hollinger2014sampling} which in turn are equivalent or adapted from~\cite{bry2011rapidly,karaman2011sampling}.
The \texttt{Steer} function used in Assumption~\ref{assump:interpoint} extends nodes towards newly sampled points.
\begin{assumption}
\label{assump:interpoint}
 Let $\boldsymbol x_a$, $\boldsymbol x_b$, and $\boldsymbol x_c \in \mathcal{X}_{f}$ be three points within radius $\Delta$ of each other. Let the trajectory $e_1$ be generated
 by $\texttt{Steer}(\boldsymbol x_a, \boldsymbol x_c,\Delta)$, $e_2$ be generated by $\texttt{Steer}(\boldsymbol x_a, \boldsymbol x_b,\Delta)$, and $e_3$ be generated by $\texttt{Steer}(\boldsymbol x_b, \boldsymbol x_c,\Delta)$.
 If $\boldsymbol x_b \in e_1$, then the concatenated trajectory $e_2+e_3$ must be equal to $e_1$ and have equal cost and information.
\end{assumption}
This assumption is required as in the limit drawn samples are infinitely close together, 
and the \texttt{Steer} function, cost, and information need to be consistent for any intermediate point.
\begin{assumption}
\label{assump:ballr}
 There exists a constant $r \in \mathbb{R}_{>0}$ such that for any point $\boldsymbol x_a \in \mathcal{X}_f$ there exists an $\boldsymbol x_b \in \mathcal{X}_f$, such that
 $1)$ the ball of radius $r$ centered at $\boldsymbol x_a$ lies inside $\mathcal{X}_f$ and $2)$ $\boldsymbol x_a$ lies inside the ball of radius $r$ centered at $\boldsymbol x_b$.
\end{assumption}
This assumption ensures that there is enough free space near any point for extension of the graph. Violation of this assumption in practice can lead to failure of the algorithm to find a path. 
\begin{assumption}[Uniform sampling]\footnote{Results extend naturally to any absolutely continuous distribution with density bounded away from zero on workspace $\mathcal{X}$~\citep{karaman2011sampling}.}
 Points returned by the sampling function \texttt{sample} are i.i.d. and drawn from a uniform distribution.
\end{assumption}

\subsection{Incremental informative motion planning}
\label{subsec:incinfmp}
Now we define the problem of incremental informative motion planning as follows.
\begin{problem}[Incremental informative motion planning]
\label{prob:iig}
Let $s_{0:t_s} \in \mathcal{S}$ be the current estimate of the state up to time $t_s$. Let $\mathcal{A}_t$ be the space of all possible trajectories 
at time $t$ and $f_I(\mathcal{P}_t)$ be a function that quantifies the information quality along a trajectory $\mathcal{P}_t$. 
Let $f_c(\mathcal{P}_t)$ be a function that returns the cost associated with trajectory $\mathcal{P}_t$.
Given the available budget $b_t$, the problem can be formulated as follows.
\begin{align}
\label{eq:incinfmp}
 \nonumber \mathcal{P}^*_t =  &\underset{\mathcal{P}_t \in \mathcal{A}_t}\argmax \ f_I(\mathcal{P}_t) \quad \forall \ t > t_s \\
 &\quad \text{s.t.} \ f_c(\mathcal{P}_t) \leq b_t\ \text{and}\ S = s_{0:t_s}
\end{align}
\end{problem}
\begin{remark}
 The state $S$ can include the representation of the environment (map), the robot trajectory, 
 and possibly any other variables defined in the state vector. In general, the information function $f_I(\mathcal{P}_t)$ is responsible 
 for incorporating the state uncertainty in the information gain calculations.
\end{remark}
\begin{remark}
 In practice we solve the problem incrementally and use a planning horizon $T > t_s$ that in the limit goes to $\infty$.
\end{remark}

The main difference between Problem~\ref{prob:rig} and Problem~\ref{prob:iig} is that in the latter the robot does not have the full knowledge of the environment \emph{a priori}. Therefore Problem~\ref{prob:iig} is not only the problem of information gathering but also \emph{planning for estimation} as the robot needs to infer the map (and in general its pose in the SLAM problem) sequentially. Note that we do not impose any assumptions on the observability of the robot pose and the map; therefore, they can be partially observable as is the case in POMDPs.

The aforementioned problems are both in their offline (nonadaptive) and online (adaptive) forms NP-hard~\citep{singh2009efficient}.
We build our proposed incremental information gathering algorithm on top of the RIG to solve the interesting problem of autonomous robotic
exploration in unknown environments. Furthermore, since the ultimate goal is online applications, we only consider the RIG-tree variant to be extended for sequential planning. This conclusion stems from extensive comparisons of RIG variants provided in~\cite{hollinger2014sampling}. However, we acknowledge that the RIG-graph is an interesting case to consider as under a \emph{partial ordering} assumption it is \emph{asymptotically optimal}.

\begin{algorithm}[t]
\small
\caption{\texttt{RIG-tree}()}
\label{alg:rigtree}
\begin{algorithmic}[1]
\Require Step size $\Delta$, budget $b$, free space $\mathcal{X}_{f}$, Environment $\mathcal{M}$, start configuration $\boldsymbol x_{start}$, near radius $r$;
\State // Initialize cost, information, starting node, node list, edge list, and tree
\State $I_{init} \gets \texttt{Information}([\ ],\boldsymbol x_{start},\mathcal{M}), C_{init} \gets 0, n \gets \langle \boldsymbol x_{start}, C_{init}, I_{init} \rangle$ \label{line:rig_inits}
\State $\mathcal{V} \gets \{n\}, \mathcal{V}_{closed} \gets \varnothing, \mathcal{E} \gets \varnothing$ \label{line:rig_inite}
\While {not terminated}
\State // Sample configuration space of vehicle and find nearest node
\State $\boldsymbol x_{sample} \gets \texttt{Sample}(\mathcal{X}_f)$ \label{line:rig_samples}
\State $\boldsymbol x_{nearest} \gets \texttt{Nearest}(\boldsymbol x_{sample}, \mathcal{V} \backslash \mathcal{V}_{closed})$
\State $\boldsymbol x_{feasible} \gets \texttt{Steer}(\boldsymbol x_{nearest}, \boldsymbol x_{sample}, \Delta)$ \label{line:rig_samplee}
\State // Find near points to be extended
\State $\mathcal{V}_{near} \gets \texttt{Near}(\boldsymbol x_{feasible}, \mathcal{V} \backslash \mathcal{V}_{closed}, r)$ \label{line:rig_ballr}
\For{all $n_{near} \in \mathcal{V}_{near}$}
\State // Extend towards new point
\State $\boldsymbol x_{new} \gets \texttt{Steer}(\boldsymbol x_{near}, \boldsymbol x_{feasible}, \Delta)$ \label{line:rig_newnode}
\If{$\texttt{NoCollision}(\boldsymbol x_{near}, \boldsymbol x_{new}, \mathcal{X}_{f})$} \label{line:rig_colfs}
\State // Calculate new information and cost
\State $I_{new} \gets \texttt{Information}(I_{near},\boldsymbol x_{new},\mathcal{M})$
\State $c(\boldsymbol x_{new}) \gets \texttt{Cost}(\boldsymbol x_{near}, \boldsymbol x_{new})$
\State $C_{new} \gets C_{near} + c(\boldsymbol x_{new}), n_{new} \gets \langle \boldsymbol x_{new}, C_{new}, I_{new} \rangle$ \label{line:rig_colfe}
\If{$\texttt{Prune}(n_{new})$} \label{line:rig_prunes}
\State \textbf{delete} $n_{new}$ 
\Else
\State // Add edges and nodes
\State $\mathcal{E} \gets \cup \{(n_{near}, n_{new})\}, \mathcal{V} \gets \cup \{n_{new}\}$
\State // Add to closed list if budget exceeded
\If{$C_{new} > b$}
\State $\mathcal{V}_{closed} \gets \mathcal{V}_{closed} \cup \{n_{new}\}$ \label{line:rig_prunee}
\EndIf
\EndIf
\EndIf
\EndFor
\EndWhile
\Return $\mathcal{T} = (\mathcal{V}, \mathcal{E})$
\end{algorithmic}
\end{algorithm}

\subsection{RIG algorithms}
\label{subsec:rig}

The sampling-based RIG algorithms find a trajectory that maximizes an information quality metric with
respect to a pre-specified budget constraint~\citep{hollinger2014sampling}. 
The RIG is based on RRT*, RRG, and PRM*~\citep{karaman2011sampling} and borrow the notion of informative path planning 
from branch and bound optimization~\citep{binney2012branch}. Algorithm~\ref{alg:rigtree} shows the RIG-tree algorithm.
The functions that are used in the algorithm are explained as follows.

\texttt{Cost} -- The cost function assigns a strictly positive cost to a collision-free path between two points from the free space $\mathcal{X}_f$.

\texttt{Information} -- This function quantifies the information quality of a collision-free path between two points from the free space $\mathcal{X}_f$.

\texttt{Sample} -- This function returns i.i.d. samples from $\mathcal{X}_f$.

\texttt{Nearest} -- Given a graph $\mathcal{G}=(\mathcal{V},\mathcal{E})$, where $\mathcal{V} \subset \mathcal{X}_f$, and a query 
point $\boldsymbol x \in \mathcal{X}_f$, this function returns a vertex $v \in \mathcal{V}$ that has the ``closest'' distance to the query point~\footnote{Here we use Euclidean distance.}.

\texttt{Steer} -- This function extends nodes towards newly sampled points and allows for constraints on motion of the robot~\footnote{Through this function, it is possible to make the planner kinodynamic.}.

\texttt{Near} -- Given a graph $\mathcal{G}=(\mathcal{V},\mathcal{E})$, where $\mathcal{V} \subset \mathcal{X}_f$, a query 
point $\boldsymbol x \in \mathcal{X}_f$, and a positive real number $r \in \mathbb{R}_{>0}$, this function returns a set of vertices 
$\mathcal{V}_{near} \subseteq \mathcal{V}$ that are contained in a ball of radius $r$ centered at $\boldsymbol x$.

\texttt{NoCollision} -- Given two points $\boldsymbol x_a, \boldsymbol x_b \in \mathcal{X}_f$, this functions returns \textbf{true} if the line segment 
between $\boldsymbol x_a$ and $\boldsymbol x_b$ is collision-free and \textbf{false} otherwise.

\texttt{Prune} -- This function implements a pruning strategy to remove nodes that are not ``promising''. This can be achieved through defining 
a \emph{partial ordering} for co-located nodes.

In line~\ref{line:rig_inits}-\ref{line:rig_inite} the algorithm initializes the starting node of the graph (tree). 
In line~\ref{line:rig_samples}-\ref{line:rig_samplee}, a sample point from workspace $\mathcal{X}$ is drawn and is converted to a feasible point, from its nearest neighbor in the graph. Line~\ref{line:rig_ballr} extracts all nodes from the graph that are within radius $r$ of the feasible point. These nodes are candidates for extending the graph, and each node is converted to a new node using the \texttt{Steer} function in line~\ref{line:rig_newnode}. In line~\ref{line:rig_colfs}-\ref{line:rig_colfe}, if there exists a collision free path between the candidate node and the new node, the information gain and cost of the new node are evaluated. In line~\ref{line:rig_prunes}-\ref{line:rig_prunee}, if the new node does not satisfy a partial ordering condition it is pruned, otherwise it is added to the graph. Furthermore, the algorithm checks for the budget constraint violation. The output is a graph that contains a subset of traversable paths with maximum information gain.

\subsection{System dynamics}
\label{subsec:sysdyn}

The equation of motion of the robot is governed by the nonlinear partially observable equation as follows.
\begin{equation}
\label{eq:reom}
\boldsymbol x_{t+1}^- = f(\boldsymbol x_{t}, \boldsymbol u_{t}, \boldsymbol w_{t}) \quad \boldsymbol w_{t} \sim \mathcal{N}(\boldsymbol 0,\boldsymbol Q_{t})
\end{equation}
moreover, with appropriate linearization at the current state estimate, we can predict the state covariance matrix as
\begin{equation}
\label{eq:predcov}
 \boldsymbol \Sigma_{t+1}^- = \boldsymbol F_t \boldsymbol \Sigma_{t} \boldsymbol F_t^T + \boldsymbol W_t \boldsymbol Q_t \boldsymbol W_t^T
\end{equation}
where $\boldsymbol F_t = \frac{\partial f}{\partial \boldsymbol x} \vert_{\boldsymbol x_{t}, \boldsymbol u_{t}}$ and $\boldsymbol W_t = \frac{\partial f}{\partial \boldsymbol w} \vert_{\boldsymbol x_{t}, \boldsymbol u_{t}}$ are the Jacobian matrices calculated with respect to $\boldsymbol x$ and $\boldsymbol w$, respectively.

\section{IIG: Incrementally-exploring information gathering}
\label{sec:incrig}

In this section, we present the IIG algorithm which is essentially RIG with an information-theoretic convergence condition. The algorithmic implementation of IIG is shown in Algorithm~\ref{alg:iigtree}. We employ IIG to solve the robotic exploration problem with the partially observable state. Both RIG and IIG, through incremental sampling, search the space of possible trajectories to find the maximally informative path; however, due to the automatic convergence of the IIG, it is possible to run the algorithm online without the full knowledge of the state, i.e.{\@} the map and robot poses.

\begin{algorithm}[th!]
\small
\caption{\texttt{IIG-tree}()}
\label{alg:iigtree}
\begin{algorithmic}[1]
\Require Step size $\Delta$, budget $b$, free space $\mathcal{X}_{f}$, Environment $\mathcal{M}$, start configuration $\boldsymbol x_{start}$, near radius $r$, relative information contribution threshold $\delta_{RIC}$, averaging window size $n_{RIC}$;
\State // Initialize cost, information, starting node, node list, edge list, and tree
\State $I_{init} \gets \texttt{Information}([\ ],\boldsymbol x_{start},\mathcal{M}), C_{init} \gets 0, n \gets \langle \boldsymbol x_{start}, C_{init}, I_{init} \rangle$
\State $\mathcal{V} \gets \{n\}, \mathcal{V}_{closed} \gets \varnothing, \mathcal{E} \gets \varnothing$
\State $n_{sample} \gets 0$ // Number of samples
\State $I_{RIC} \gets \varnothing$ // Relative information contribution
\While {$\texttt{AverageRIC}(I_{RIC},n_{RIC}) > \delta_{RIC}$} $\label{iigcondition}$
\State // Sample configuration space of vehicle and find nearest node
\State $\boldsymbol x_{sample} \gets \texttt{Sample}(\mathcal{X}_f)$
\State $n_{sample} \gets n_{sample} + 1$ // Increment sample counter
\State $\boldsymbol x_{nearest} \gets \texttt{Nearest}(\boldsymbol x_{sample}, \mathcal{V} \backslash \mathcal{V}_{closed})$
\State $\boldsymbol x_{feasible} \gets \texttt{Steer}(\boldsymbol x_{nearest}, \boldsymbol x_{sample}, \Delta)$
\State // Find near points to be extended
\State $\mathcal{V}_{near} \gets \texttt{Near}(\boldsymbol x_{feasible}, \mathcal{V} \backslash \mathcal{V}_{closed}, r)$
\For{all $n_{near} \in \mathcal{V}_{near}$}
\State // Extend towards new point
\State $\boldsymbol x_{new} \gets \texttt{Steer}(\boldsymbol x_{near}, \boldsymbol x_{feasible}, \Delta)$
\If{$\texttt{NoCollision}(\boldsymbol x_{near}, \boldsymbol x_{new}, \mathcal{X}_{f})$} 
\State // Calculate new information and cost
\State $I_{new} \gets \texttt{Information}(I_{near},\boldsymbol x_{new},\mathcal{M})$
\State $c(\boldsymbol x_{new}) \gets \texttt{Cost}(\boldsymbol x_{near}, \boldsymbol x_{new})$
\State $C_{new} \gets C_{near} + c(\boldsymbol x_{new}), n_{new} \gets \langle \boldsymbol x_{new}, C_{new}, I_{new} \rangle$
\If{$\texttt{Prune}(n_{new})$}
\State \textbf{delete} $n_{new}$
\Else
\State $I_{RIC} \gets \texttt{append}(I_{RIC}, (\frac{I_{new}}{I_{near}} - 1)/n_{sample})$ \label{line:iig_ric} // Equation~\eqref{eq:pric}
\State $n_{sample} \gets 0$ // Reset sample counter
\State // Add edges and nodes
\State $\mathcal{E} \gets \cup \{(n_{near}, n_{new})\}, \mathcal{V} \gets \cup \{n_{new}\}$
\State // Add to closed list if budget exceeded
\If{$C_{new} > b$}
\State $\mathcal{V}_{closed} \gets \mathcal{V}_{closed} \cup \{n_{new}\}$
\EndIf
\EndIf
\EndIf
\EndFor
\EndWhile
\Return $\mathcal{T} = (\mathcal{V}, \mathcal{E})$
\end{algorithmic}
\end{algorithm}

We introduce the Relative Information Contribution (RIC) criterion to detect the convergence of the search. The motivation behind this definition is that the number of nodes constantly increases unless an aggressive pruning strategy is used. However, an aggressive pruning strategy leads to potentially pruning nodes that can be part of optimal solutions~\footnote{Note that more than one optimal trajectory at each time can exist, e.g.{\@} when the robot needs to explore two equally important directions.}. Even though the algorithm continues to add nodes, it is possible to evaluate the contribution of each added node in the relative information sense. In other words, adding nodes does not affect the convergence of the algorithm, but the amount of information the algorithm can collect by continuing the search. We define the RIC of a node as follows.
\begin{definition}[Relative Information Contribution]
In Algorithm~\ref{alg:iigtree}, let $\boldsymbol x_{new} \in \mathcal{X}_{f}$ be a reachable point through a neighboring node $n_{near} \in \mathcal{V}$ returned by the function $\texttt{Near}()$. Let $I_{new}$ and $I_{near}$ be the information values of their corresponding nodes returned by the function $\texttt{Information}()$. The relative information contribution of node $n_{new}$ is defined as
\begin{equation}
\label{eq:ric}
 RIC \triangleq \frac{I_{new}}{I_{near}} - 1
\end{equation}
\end{definition}
Equation~\eqref{eq:ric} is conceptually important as it defines the amount of information gain relative to a neighboring point in the IIG graph. In practice, the number of samples it takes before the algorithm finds a new node becomes important. Thus we define penalized relative information contribution that is computed in line~\ref{line:iig_ric}.
\begin{definition}[Penalized Relative Information Contribution]
 Let $RIC$ be the relative information contribution computed using Equation~\eqref{eq:ric}. Let $n_{sample}$ be the number of samples it takes to find the node $n_{new}$. The penalized relative information contribution is defined as
 \begin{equation}
  \label{eq:pric}
  I_{RIC} \triangleq \frac{RIC}{n_{sample}}
 \end{equation}
\end{definition}

An appealing property of $I_{RIC}$ is that it is non-dimensional, and it does not depend on the actual calculation/approximation of the information values. In practice, as long as the information function satisfies the RIG/IIG requirements, using the following condition, IIG algorithm converges. Let $\delta_{RIC}$ be a threshold that is used to detect the convergence of the algorithm. Through averaging $I_{RIC}$ values over a window of size $n_{RIC}$, we ensure that continuing the search will not add any significant amount of information to the IIG graph. In Algorithm~\ref{alg:iigtree}, this condition is shown in line~\ref{iigcondition} by function $\texttt{AverageRIC}$.
\begin{remark}
 In Algorithm~\ref{alg:iigtree}, $\delta_{RIC}$ sets the planning horizon from the information gathering point of view. Through using smaller values of $\delta_{RIC}$ the planner can reach further points in both spatial and belief space. In other words, if $\delta_{RIC} \rightarrow 0$, then $T \rightarrow \infty$.
\end{remark}

\section{Information functions algorithms}
\label{sec:infofunc}

We propose two classes of algorithms to approximate the information gain at any sampled point from the free workspace. The information function in RIG/IIG algorithms often causes a bottleneck and computationally dominates the other parts. Therefore, even for offline calculations, it is important to have access to functions that, concurrently, are computationally tractable and can capture the essence of information gathering. We emphasize that the information functions are directly related to the employed sensors. However, once the model is provided and incorporated into the estimation/prediction algorithms, the information-theoretic aspects of the provided algorithms remain the same.

The information functions that are proposed are different in nature. First, we discuss MI-based information functions whose calculations explicitly depend on the probabilistic sensor model. We provide a variant of the MI Algorithm in~\citet{maani2015mi,jadidi2016gaussian} that is developed for range-finder sensors and based on the beam-based mixture measurement model and the inverse sensor model map prediction~\citep{thrun2005probabilistic}. We also present an algorithm to approximate MI upper bound which reveals the maximum achievable information gain.

Then, we exploit the property of GPs to approximate the information gain. In Equation~\eqref{eq:gp_cov}, the variance calculation does not explicitly depend on the target vector (measurements) realization. In this case, as long as the underlying process is modeled as GPs, the information gain can be calculated using prior and posterior variances which removes the need for relying on a specific sensor model and calculating the expectation over future measurements. However, note that the hyperparameters of the covariance functions are learned using the training set which contains measurements; therefore, the knowledge of underlying process and measurements is incorporated into the GP through its hyperparameters. Once we established GP Variance Reduction (GPVR) algorithm, we then use the expected kernel notion~\citep{maaniwgpom} to propagate pose uncertainty into the covariance function resulting in Uncertain GP Variance Reduction (UGPVR) algorithm. In particular, these two information functions are interesting for the following reasons:
\begin{itemize}
 \item[(i)] Unlike MI-based (direct information gain calculation), they are non-parametric.
 \item[(ii)] GPVR-based information functions provide a systematic way to incorporate input (state) uncertainty into information gathering frameworks.
 \item[(iii)] In the case of incomplete knowledge about the quantity of interest in an unknown environment, they allow for \emph{active learning}~\footnote{Although this is one of the most interesting aspects of GPVR-based information functions, it is beyond the scope of this article, and we leave it as a possible extension of this work.}.
\end{itemize}

\subsection{Mutual information}
\label{subsec:miinfofunc}

To calculate MI without information ``double-counting'' we need to update the map after every measurement prediction. It is possible to perform map prediction using a forward or inverse sensor model~\citep{thrun2005probabilistic}. Typically using an inverse sensor model results in simpler calculations. We first define two required parameters in the proposed algorithm as follows.

\begin{definition}[Map saturation probability]
 The probability that the robot is completely confident about the occupancy status of a point is defined as $p_{sat}$.
\end{definition}
\begin{definition}[Map saturation entropy]
 The entropy of a point from a map whose occupancy probability is $p_{sat}$, is defined as  $h_{sat} \triangleq H(p_{sat})$.
\end{definition}

The defined parameters are relevant since they prevent the exhaustive search for information in ``less important'' areas. The MI-based information function using an inverse sensor model implementation is shown in Algorithm~\ref{alg:funcmi2}. In line~\ref{line:raycast}, the predicted range measurement for beam $\alpha$, $\hat{z}^{[\alpha]}$, is computed using ray casting in the current map estimate where $\mathcal{I}^{[\alpha]}$ denotes the index set of map points that are in the perception field of the $\alpha$-th sensor beam. In line~\ref{line:skiphib}-\ref{line:skiphie}, the algorithm skips any map point whose entropy surpasses the saturation entropy or adds the map entropy of point $i$ to the information gain, $I$. In line~\ref{line:centb}-\ref{line:cente}, the map conditional entropy by integrating over future measurement is calculated where, in line~\ref{line:invmappredb}-\ref{line:invmapprede}, map prediction is performed using free point belief, $b_{free}$, and occupied point belief, $b_{occ}$. The predicted probabilities are clamped using $\epsilon > 0$ which is a small number relative to $p_{sat}$ to avoid losing numerical computation accuracy. In line~\ref{line:miupdate}, the map conditional entropy is subtracted from initial map entropy ($\bar{h}$ is negative) using an appropriate numerical integration resolution, $s_z$.

\begin{algorithm}[t!]
\caption[InformationMI]{\texttt{InformationMI}()}
\label{alg:funcmi2}
\small{
\begin{algorithmic}[1]
\Require Robot pose or desired location, current map estimate $m$, sensor model, saturation probability $p_{sat}$, free point belief $b_{free}$, occupied point belief $b_{occ}$, numerical integration resolution $s_z$, near node information $I_{near}$;
\State $\bar{m} \gets m$ // Initialize updated map as the current map
\If {$I_{near}$ is not empty} // Initialize information gain\label{line:initinfogain}
\State $I \gets I_{near}$ 
\Else
\State $I \gets 0$
\EndIf\label{line:initinfogaine}
\State // Compute saturation entropy
\State $h_{sat} \gets -[p_{sat} \log(p_{sat}) + (1-p_{sat}) \log(1-p_{sat})]$
\For {all $\alpha$} // Loop over all sensor beams
\State Compute $\hat{z}^{[\alpha]}$ and $\mathcal{I}^{[\alpha]}$ using ray casting in $m$ \label{line:raycast}
\State // Calculate map conditional entropy along beam $\alpha$
\For {$i \in \mathcal{I}^{[\alpha]}$}
\State // Entropy of point $i$
\State $h_i \gets -[\bar{m}^{[i]} \log(\bar{m}^{[i]}) + (1-\bar{m}^{[i]}) \log(1-\bar{m}^{[i]})]$
\If{$h_i < h_{sat}$} \label{line:skiphib}
\State \textbf{continue}
\Else
\State $I \gets I + h_i$ // Add to information gain \label{line:addtoinfo}
\EndIf \label{line:skiphie}
\State $\bar{h} \gets 0$ // Initialize map conditional entropy\label{line:centb}
\State $z \gets s_z^{-1}$ // Initialize range dummy variable
\While {$z \leq \hat{z}^{[\alpha]}$}
\State // Calculate marginal measurement probability $p_z$
\State $p_1 \gets p(z|M=0)$ 
\State $p_2 \gets 0$
\For {$j \in \mathcal{I}^{[\alpha]}$}
\State $p_1 \gets p_1 (1-m^{[j]})$
\State $p_2 \gets p_2 + p(z|M=m^{[j]}) m^{[j]} \displaystyle\prod_{l < j}{(1-m^{[l]})}$
\EndFor
\State $p_z \gets p_1 + p_2$
\State // Map prediction at point $i$ along beam $\alpha$ using inverse sensor model
\If{$\texttt{isFree}(\bar{m}^{[i]})$} \label{line:invmappredb}
\State $\bar{m}^{[i]} \gets \max(p_{sat}-\epsilon, b_{free}*\bar{m}^{[i]})$
\Else
\State $\bar{m}^{[i]} \gets \min(1-p_{sat}+\epsilon, b_{occ}*\bar{m}^{[i]})$
\EndIf \label{line:invmapprede}
\State $\bar{h} \gets \bar{h} + p_z [\bar{m}^{[i]} \log(\bar{m}^{[i]}) + (1-\bar{m}^{[i]}) \log(1-\bar{m}^{[i]})]$
\State $z \gets z + s_z^{-1}$ // Increase range along the beam
\EndWhile \label{line:cente}
\State $I \gets I + \bar{h} s_z^{-1}$ \label{line:miupdate}
\EndFor
\EndFor
\Return $I$ (total information gain), $\bar{m}$ (updated map)
\end{algorithmic}}
\end{algorithm}

\begin{algorithm}[t!]
\caption[InformationMIUB]{\texttt{InformationMIUB}()}
\label{alg:funcmiub}
\small{
\begin{algorithmic}[1]
\Require Robot pose or desired location, current map estimate $m$, sensor model, saturation probability $p_{sat}$, free point belief $b_{free}$, occupied point belief $b_{occ}$, numerical integration resolution $s_z$, near node information $I_{near}$;
\State $\bar{m} \gets m$ // Initialize updated map as the current map
\If {$I_{near}$ is not empty} // Initialize information gain
\State $I_{UB} \gets I_{near}$ 
\Else
\State $I_{UB} \gets 0$
\EndIf
\State // Compute saturation entropy
\State $h_{sat} \gets -[p_{sat} \log(p_{sat}) + (1-p_{sat}) \log(1-p_{sat})]$
\For {all $\alpha$} // Loop over all sensor beams
\State Compute $\hat{z}^{[\alpha]}$ and $\mathcal{I}^{[\alpha]}$ using ray casting in $m$
\For {$i \in \mathcal{I}^{[\alpha]}$}
\State // Entropy of point $i$
\State $h_i \gets -[\bar{m}^{[i]} \log(\bar{m}^{[i]}) + (1-\bar{m}^{[i]}) \log(1-\bar{m}^{[i]})]$
\If{$h_i < h_{sat}$}
\State \textbf{continue}
\Else
\State $I_{UB} \gets I_{UB} + h_i$ // Add to information gain
\EndIf
\State // Map prediction at point $i$ along beam $\alpha$ using inverse sensor model
\If{$\texttt{isFree}(\bar{m}^{[i]})$}
\State $\bar{m}^{[i]} \gets \max(p_{sat}-\epsilon, b_{free}*\bar{m}^{[i]})$
\Else
\State $\bar{m}^{[i]} \gets \min(1-p_{sat}+\epsilon, b_{occ}*\bar{m}^{[i]})$
\EndIf
\EndFor
\EndFor
\Return $I_{UB}$ (total information gain), $\bar{m}$ (updated map)
\end{algorithmic}}
\end{algorithm}

It is also interesting to calculate an upper bound for the information gain. Given Algorithm~\ref{alg:funcmi2}, it is trivial to calculate MI upper bound using the total amount of map entropy in the current perception field of the robot. It is faster to compute the upper bound as it only shows the uncertainty from the current map and it does not consider any gain from future measurements. However, in practice, it can be useful for fast and online predictions. More details regarding the difference between maximizing mutual information and entropy are discussed in~\cite{guestrin2005near,krause2008near}.

\begin{lemma}[Information gain upper bound]
\label{lem:miub}
 For any location in the map, the information gain upper bound is given by the total map entropy calculated using map points in the perception field of the robot at the same location.
\end{lemma}
\begin{proof}
 From Theorem~\ref{th:condentineq}, $0 \leq I(M;Z) = H(M) - H(M|Z) \leq H(M)$ which extends to any sub-map $\mathcal{M}_{sub}$ that is in the perception field of the robot. Note that MI beyond the perception field is zero as $\forall \ M \in \mathcal{M}\setminus\mathcal{M}_{sub}$, $M \bot Z$.
\end{proof}
Algorithm~\ref{alg:funcmiub} shows the MI Upper Bound (MIUB) information function in which the integration over predicted measurement is omitted. However, to avoid information double-counting, it is still required to update the map estimate after each function call. Furthermore, the MIUB calculation can be integrated into Algorithm~\ref{alg:funcmi2} with an insignificant computational load through calculation of $I_{UB}$ alongside $I$ in lines~\ref{line:initinfogain}-~\ref{line:initinfogaine} and \ref{line:addtoinfo}. As we show later, in the early stage of the search the behavior of $I_{UB}$ is similar to that of $I$, therefore, it is possible to use the upper bound at the beginning of the sampling to speed up the search.

\subsection{GP variance reduction}
\label{subsec:gpvr}

Variance reduction is the essence of information gathering. Since predictive variance calculation in Equation~\eqref{eq:gp_cov}, does not depend on observations, we can come up with a non-parametric algorithm to estimate variance reduction throughout dense belief representation of the map. For the problem of informative path planning, a similar approach is used in~\cite{binney2013optimizing} where the reduction in the trace of the covariance function is considered as the objective function (A-optimality). Here, we are interested in approximating the mutual information through entropy reduction, i.e.{\@} using determinant of the covariance matrix (D-optimality)~\citep{pukelsheim1993optimal}. This is mainly to keep the proposed IIG framework agnostic about the choice of information functions. We treat each map point as a continuous random variable that is normally distributed. Therefore we use differential entropy formulation for mutual information approximation. Differential entropy of a Gaussian random variable, \mbox{$X \sim \mathcal{N}(\mu, \sigma^2)$}, can be derived as \mbox{$h(X) = \frac{1}{2}\log(2\pi e \sigma^2)$}; and where \mbox{$X \sim \mathcal{N}(\boldsymbol \mu,\boldsymbol \Sigma)$} is a Gaussian random vector of dimension $n$, the differential entropy can be derived as \mbox{$h(X) = \frac{1}{2}\log((2\pi e)^n \lvert \boldsymbol \Sigma \rvert)$}. Now in a Bayesian setup, let \mbox{$X \sim \mathcal{N}(\boldsymbol \mu_X,\boldsymbol \Sigma_X)$} and \mbox{$X\mid Z \sim \mathcal{N}(\boldsymbol \mu_{X\mid Z},\boldsymbol \Sigma_{X\mid Z})$} be the prior and posterior distribution of the random vector $X$. Subsequently, it follows that the mutual information after receiving observation $Z$ can be derived as
\begin{equation}
 I(X;Z) = \frac{1}{2} [\log(\lvert \boldsymbol \Sigma_X \rvert) - \log(\lvert \boldsymbol \Sigma_{X\mid Z} \rvert)]
\end{equation}

Where possible, the mutual information should be computed using the full covariance matrix or its block-diagonal approximation. However, typically for large problems with dense belief representations, maintaining and updating the full covariance matrix for thousands of random variables is not tractable. Therefore, to approximate the mutual information, we suggest a trade-off approach between the tractability and accuracy based on the problem at hand. In the following, we propose an approximation of the mutual information using the marginalization property of normal distribution. We also discuss the relation of this approximation with the exact mutual information.

\begin{lemma}[Marginalization property of normal distribution~\citep{von2014mathematical}]
\label{lem:margprop}
 Let $\boldsymbol x$ and $\boldsymbol y$ be jointly Gaussian random vectors
 \begin{equation}
  \begin{bmatrix}
	\boldsymbol x \\
	\boldsymbol y
 \end{bmatrix} \sim \mathcal{N}(
 \begin{bmatrix}
  \boldsymbol\mu_x \\
  \boldsymbol\mu_y
 \end{bmatrix},
 \begin{bmatrix}
	\boldsymbol A & \boldsymbol C \\
	\boldsymbol C^T & \boldsymbol B 
 \end{bmatrix})
 \end{equation}
 then the marginal distribution of $\boldsymbol x$ is
 \begin{equation}
  \boldsymbol x \sim \mathcal{N}(\boldsymbol\mu_x, \boldsymbol A)
 \end{equation}
\end{lemma}
\begin{proposition}
\label{prop:miapprox}
 Let $X_1, X_2, ..., X_n$ have a multivariate normal distribution with covariance matrix $\boldsymbol K$. 
 The mutual information between $X$ and observations $Z$ can be approximated as
 \begin{equation}
  \hat{I}(X;Z) = \frac{1}{2} [\sum_{i=1}^{n}{\log(\sigma_{X_i})} - \sum_{i=1}^{n}{\log(\sigma_{X_i|Z})}]
 \end{equation}
 where $\sigma_{X_i}$ and $\sigma_{X_i|Z}$ are marginal variances for $X_i$ before and after incorporating observations $Z$, 
 i.e.{\@} prior and posterior marginal variances.
\end{proposition}
\begin{proof}
 Using marginalization property of normal distribution, Lemma~\ref{lem:margprop}, for every $X_i$ we have $\Var{X_i} = \boldsymbol K^{[i,i]}$. Using differential entropy of $X_i$, the mutual information for $X_i$ can be written as
 \begin{equation}
  \hat{I}^{[i]}(X_i;Z) = \frac{1}{2} [\log(\sigma_{X_i}) - \log(\sigma_{X_i|Z})]
 \end{equation}
 and the total mutual information can be calculated as $\hat{I}(X;Z) = \sum_{i=1}^{n}\hat{I}^{[i]}(X_i;Z)$. Alternatively, one could build a new covariance matrix by placing marginal variances on its diagonal and use the fact the the determinant of a diagonal matrix is the product of its diagonal elements.
\end{proof}

This approximation makes the information gain calculation for a class of problem with dense belief representation tractable. However, it is interesting to study the effect this approximation. Intuitively, the determinant of the covariance matrix corresponds to the hypervolume of the subspace spanned by the columns of the covariance matrix. When ignoring the correlation between random variables, Proposition~\ref{prop:miapprox}, the spanned subspace becomes larger; as a result, the determinant grows which corresponds to higher entropy. Regarding the information gain, the following statement holds true.

\begin{proposition}
\label{prop:miineq}
$ \hat{I}(X;Z) \le I(X;Z)$.
\end{proposition}
\begin{proof}
 See Appendix~\ref{appx:miapprox}.
\end{proof}

\begin{algorithm}[t!]
\caption[InformationGPVR]{\texttt{InformationGPVR}()}
\label{alg:gpvr}
\small{
\begin{algorithmic}[1]
\Require Robot pose or desired location $\boldsymbol p$, current map/state estimate $m$, covariance function $k(\cdot,\cdot)$, sensor noise $\sigma_n^2$, near node information $I_{near}$;
\State $\bar{\sigma} \gets \sigma$ // Initialize updated map variance as the current map variance
\If {$I_{near}$ is not empty} // Initialize information gain
\State $I \gets I_{near}$ 
\Else
\State $I \gets 0$
\EndIf
\State $\boldsymbol z \gets$ Predict future measurements using $\boldsymbol p$ and $m$\label{line:gpvrknnb}
\State $\mathcal{D} \gets$ Construct training set using $\boldsymbol z$ and $\boldsymbol p$\label{line:trainingset}
\State // Find the corresponding nearest sub-map
\State $\mathcal{M}_{\mathcal{D}} \gets \varnothing$
\For{all $\boldsymbol x \in \mathcal{D}$}
\State $\boldsymbol x_{nearest} \gets \texttt{Nearest}(\boldsymbol x, \mathcal{M})$ 
\State $\mathcal{M}_{\mathcal{D}} \gets \mathcal{M}_{\mathcal{D}} \cup \{\boldsymbol x_{nearest}\}$
\EndFor\label{line:gpvrknne}
\State // Calculate self-covariance and cross-covariance matrices
\State $\boldsymbol C \gets \boldsymbol{K}(\boldsymbol X, \boldsymbol X), \boldsymbol C_* \gets \boldsymbol{K}(\boldsymbol X, \boldsymbol X_*)$ // $\boldsymbol X \in \mathcal{D}$ and \mbox{$\boldsymbol X_* \in \mathcal{M}_{\mathcal{D}}$} \label{line:gpvrkovb}
\State // Calculate vector of diagonal variances for test points
\State $\boldsymbol c_{**} \gets \texttt{diag}(\boldsymbol{K}(\boldsymbol X_*, \boldsymbol X_*))$
\State $\boldsymbol L \gets \texttt{Cholesky}(\boldsymbol C + \sigma_n^2 \boldsymbol I)$, $\boldsymbol V \gets \boldsymbol L \backslash \boldsymbol C_*$
\State $\boldsymbol v \gets \boldsymbol c_{**} - \texttt{dot}(\boldsymbol V, \boldsymbol V)^{T}$ // dot product \label{line:gpvrkove}
\For {all $i \in \mathcal{M}_{\mathcal{D}}$} \label{line:gpvrmib}
\State $\bar{\sigma}^{[i]} \gets ((\sigma^{[i]})^{-1}+ (v^{[i]})^{-1})^{-1}$ // BCM fusion
\State $I \gets I + \log(\sigma^{[i]}) - \log(\bar{\sigma}^{[i]})$
\EndFor \label{line:gpvrmie}
\Return $I$ (total information gain), $\bar{\sigma}$ (updated map variance)
\end{algorithmic}}
\end{algorithm}

Algorithm~\ref{alg:gpvr} shows the details of GPVR information function~\footnote{The algorithm uses MATLAB-style operations for matrix inversion, Cholesky factorization, and dot product with matrix inputs.}. Based on the training points generated from predicted measurement $\boldsymbol z$, a sub-map from the current map estimate using nearest neighbor search is found, line~\ref{line:gpvrknnb}-\ref{line:gpvrknne}. In line~\ref{line:gpvrkovb}-\ref{line:gpvrkove}, GP predictive variances, \mbox{$\boldsymbol v = \mathrm{vec}(v^{[1]},\dots,v^{[\lvert \mathcal{M}_{\mathcal{D}} \rvert]})$}, are computed using covariance function $k(\cdot,\cdot)$ with the same hyperparameters learned for the map inference. In line~\ref{line:gpvrmib}-\ref{line:gpvrmie}, using Bayesian Committee Machine (BCM) fusion~\citep{tresp2000bayesian} the predictive marginal posterior variance is calculated and the information gain is updated consequently~\footnote{Note that the constant factor $\frac{1}{2}$ is removed since it does not have any effect in this context.}. BCM combines estimators which were trained on different data sets and is shown to be suitable for incremental map building~\citep{jadidi2016gaussian}. The Cholesky factorization is the most computationally expensive operation of the algorithm. However, it is possible to exploit a sparse covariance matrix such as the kernel in~\cite{melkumyan2009sparse} or use a cut-off distance for the covariance function~\footnote{The positive semidefinite property of the covariance matrix needs to be preserved.} to speed up the algorithm.

We emphasize that to use GPVR Algorithm, the underlying process needs to be modeled using GPs, i.e. $y(\boldsymbol x) \sim \mathcal{GP}(f_m(\boldsymbol x), k(\boldsymbol x,\boldsymbol x'))$ where $f_m(\boldsymbol x)$ is the GP mean function. Therefore, for any map point we have \mbox{$m^{[i]} = y(\boldsymbol x^{[i]}_*) \sim \mathcal{N}(\mu^{[i]},\sigma^{[i]})$}. Furthermore, construction of the training set, in line~\ref{line:trainingset}, is part of the GP modeling. For the particular case of occupancy mapping using a range-finder sensor see~\cite[Chapter 4]{ghaffari2017gaussian}. For the case where the robot only receives point measurements at any location, such as wireless signal strength, we explain in Subsection~\ref{subsec:lakemon}.

\subsection{Uncertain GP variance reduction}
\label{subsec:ugpvr}

Thus far, the developed information functions do not incorporate uncertainties of other state variables that are jointly distributed with the map (such as the robot pose) in information gain calculation. We define the modified kernel $\tilde{k}$ as follows.
\begin{definition}[Modified kernel]
 Let $k(x,x_*)$ be a kernel and $X \in \mathcal{X}$ a random variable that is distributed according to a probability distribution function $p(x)$. The modified kernel is defined as
 its expectation with respect to $p(x)$, therefore we can write
 \begin{equation}
 \label{eq:modkernel}
  \tilde{k} = \EV k = \int_{\Omega} kdp
 \end{equation}
\end{definition}

Through replacing the kernel function in Algorithm~\ref{alg:gpvr} with the modified kernel we can propagate the robot pose uncertainty in the information gain calculation. Intuitively, under the presence of uncertainty in other state variables that are correlated with the map, the robot does not take greedy actions as the amount of available information calculated using the modified kernel is less than the original case. Therefore, the chosen actions are relatively more conservative. The integration in Equation~\eqref{eq:modkernel} can be numerically approximated using Monte-Carlo or Gauss-Hermite quadrature techniques~\citep{davis2007methods,press1996numerical}. In the case of a Gaussian assumption for the robot pose, Gauss-Hermite quadrature provides a better accuracy and efficiency trade-off and is preferred.

Algorithm~\ref{alg:ugpvr} shows UGPVR information function. The difference with GPVR is that the input location is not deterministic, i.e.{\@} it is approximated as a normal distribution $\mathcal{N}(\boldsymbol p, \boldsymbol\Sigma)$, and the covariance function is replaced by its modified version. Given the initial pose belief, the pose uncertainty propagation on the IIG graph can be performed using the robot motion model, i.e.{\@} using Equations~\eqref{eq:reom} and \eqref{eq:predcov}.

The UGPVR estimate at most the same amount of mutual information as GPVR. This is because of taking the expectation of the kernel with respect to the robot pose posterior. If we pick the mode of the robot pose posterior, GPVR only uses that input point for mutual information computation. In contrast, UGPVR averages over all possible values of the robot pose within the support of its distribution. Now it is clear that having an estimate of the robot pose posterior with long tails (yet exponentially bounded) reduces the mutual information even further due to averaging. Furthermore, if the robot pose is not known and we have only access to its estimate, ignoring the distribution can lead to overconfident or inconsistent inference/prediction~\citep[Figure 2]{maaniwgpom}.

\begin{algorithm}[t!]
\caption[InformationUGPVR]{\texttt{InformationUGPVR}()}
\label{alg:ugpvr}
\small{
\begin{algorithmic}[1]
\Require Robot pose or desired location $\mathcal{N}(\boldsymbol p, \boldsymbol\Sigma)$, current map/state estimate $m$, modified covariance function $\tilde{k}(\cdot,\cdot)$, sensor noise $\sigma_n^2$, near node information $I_{near}$;
\State $\bar{\sigma} \gets \sigma$ // Initialize updated map variance as the current map variance
\If {$I_{near}$ is not empty} // Initialize information gain
\State $I \gets I_{near}$ 
\Else
\State $I \gets 0$
\EndIf
\State $\boldsymbol z \gets$ Predict future measurements using $\boldsymbol p$ and $m$
\State $\mathcal{D} \gets$ Construct the training set using $\boldsymbol z$ and $\boldsymbol p$
\State // Find the corresponding nearest sub-map
\State $\mathcal{M}_{\mathcal{D}} \gets \varnothing$
\For {all $\boldsymbol x \in \mathcal{D}$}
\State $\boldsymbol x_{nearest} \gets \texttt{Nearest}(\boldsymbol x, \mathcal{M})$
\State $\mathcal{M}_{\mathcal{D}} \gets \mathcal{M}_{\mathcal{D}} \cup \{\boldsymbol x_{nearest}\}$
\EndFor
\State // Calculate self-covariance and cross-covariance matrices using $\tilde{k}(\cdot,\cdot)$ with respect to $\mathcal{N}(\boldsymbol p, \boldsymbol\Sigma)$
\State $\boldsymbol C \gets \tilde{\boldsymbol{K}}(\boldsymbol X, \boldsymbol X), \boldsymbol C_* \gets \tilde{\boldsymbol{K}}(\boldsymbol X, \boldsymbol X_*)$ // $\boldsymbol X \in \mathcal{D}$ and \mbox{$\boldsymbol X_* \in \mathcal{M}_{\mathcal{D}}$}
\State // Calculate vector of diagonal variances for test points 
\State $\boldsymbol c_{**} \gets \texttt{diag}(\tilde{\boldsymbol{K}}(\boldsymbol X_*, \boldsymbol X_*))$
\State $\boldsymbol L \gets \texttt{Cholesky}(\boldsymbol C + \sigma_n^2 \boldsymbol I)$, $\boldsymbol V \gets \boldsymbol L \backslash \boldsymbol C_*$
\State $\boldsymbol v \gets \boldsymbol c_{**} - \texttt{dot}(\boldsymbol V, \boldsymbol V)^{T}$ // dot product
\For {all $i \in \mathcal{M}_{\mathcal{D}}$}
\State $\bar{\sigma}^{[i]} \gets ((\sigma^{[i]})^{-1}+ (v^{[i]})^{-1})^{-1}$ // BCM fusion
\State $I \gets I + \log(\sigma^{[i]}) - \log(\bar{\sigma}^{[i]})$
\EndFor
\Return $I$ (total information gain), $\bar{\sigma}$ (updated map variance)
\end{algorithmic}}
\end{algorithm}

\section{Path extraction and selection}
\label{sec:pathselect}

In the absence of artificial targets such as frontiers, in general, there is no goal to be found by the planner. IIG searches for traversable paths within the map, and the resulting tree shows feasible trajectories from the current robot pose to each leaf node, expanded using the maximum information gathering policy. Therefore, any path in the tree starting from the robot pose to a leaf node is a feasible action. For robotic exploration scenarios, it is not possible to traverse all the available trajectories since, after execution of one trajectory, new measurements are taken, and the map (and the robot pose) belief is updated; therefore, previous predictions are obsolete, and the robot has to enter the planning state~\footnote{Here we assume the robot remains committed to the selected action, i.e. there is no replanning while executing an action.}. As such, once the RIG/IIG tree is available, next step can be seen as decision-making where the robot selects a trajectory as an executable action. One possible solution is finding a trajectory in the IIG tree that maximizes the information gain.

\begin{algorithm}[t!]
\caption[PathSelection]{\texttt{PathSelection}()}
\label{alg:pathselect}
\small{
\begin{algorithmic}[1]
\Require RIG/IIG tree $\mathcal{T}$, path similarity ratio $s_{ratio}$;
\State // The path length equals the number of nodes in the path.
\State // Find all leaves using depth first search
\State $\mathcal{V}_{leaves} \gets \texttt{DFSpreorder}(\mathcal{T})$ \label{line:path_cuts}
\State // Find all paths by starting from each leaf and following parent nodes
\State $\mathcal{P}_{all} \gets \texttt{Paths2root}(\mathcal{T},\mathcal{V}_{leaves})$
\State $l_{max} \gets$ Find maximum path length in $\mathcal{P}_{all}$
\State $ l_{min} \gets \texttt{ceil}(\kappa l_{max})$ // Minimum path length, $0 < \kappa < 1$
\For {all $\mathcal{P} \in \mathcal{P}_{all}$}
\If {$\texttt{length}(\mathcal{P}) \leq l_{min}$}
\State \textbf{Delete} $\mathcal{P}$
\EndIf
\EndFor \label{line:path_cute}
\State $n_p \gets \lvert\mathcal{P}_{all}\rvert$ // Number of paths in set $\mathcal{P}_{all}$\label{line:path_votes}
\State $vote \gets \texttt{zeros}(n_p,1)$
\State // Find longest independent paths
\State {$i \gets 1$}
\While {$i \leq n_p-1$}
\State {$j \gets i+1$}
\While {$j \leq n_p$}
\State $l_i \gets \texttt{length}(\mathcal{P}_{i})$, $l_j \gets \texttt{length}(\mathcal{P}_{j})$
\State // Find number of common nodes between paths $i$ and~$j$, and the length ratio they share
\State $l_{ij} \gets \texttt{SimilarNodes}(\mathcal{P}_{i},\mathcal{P}_{j}) / \min(l_i, l_j)$
\If {$l_{ij} > s_{ratio}$}
\If {$l_i > l_j$} // Path $i$ is longer
\State $vote^{[i]} \gets vote^{[i]} + 1$
\State $vote^{[j]} \gets vote^{[j]} - 1$
\Else $\ $ // Path $j$ is longer
\State $vote^{[i]} \gets vote^{[i]} - 1$
\State $vote^{[j]} \gets vote^{[j]} + 1$
\EndIf
\Else $\ $ // Two independent paths
\State $vote^{[i]} \gets vote^{[i]} + 1$
\State $vote^{[j]} \gets vote^{[j]} + 1$
\EndIf
\State {$j \gets j+1$}
\EndWhile
\State {$i \gets i+1$}
\EndWhile \label{line:path_votee}
\State // Find paths with maximum vote and select the maximally informative path
\State $\mathcal{P}_{max} \gets \texttt{MaxVotePath}(\mathcal{P}_{all}, vote)$ \label{line:path_maxvote}
\State $\mathcal{P}_{I} \gets \texttt{MaxInformativePath}(\mathcal{P}_{max})$  \label{line:path_mip}
\Return $\mathcal{P}_{I}$
\end{algorithmic}}
\end{algorithm}

We provide a heuristic algorithm based on a voting method. Algorithm~\ref{alg:pathselect} shows the implementation of the proposed method. The algorithm first finds all possible paths using a preorder \emph{depth first search}, function \texttt{DFSpreorder}, and then removes paths that are shorter than a minimum length (using parameter $0 < \kappa < 1$), line~\ref{line:path_cuts}-\ref{line:path_cute}. Note that the path length and the length returned by function \texttt{length} are integers and correspond to the number of nodes in the path; consequently, the path length in Algorithm~\ref{alg:pathselect} is independent of the actual path scale. Then each path is compared with others using the following strategy. If two paths have more than a specified number nodes in common, then we penalize the shorter path by a negative vote and encourage the longer path by a positive vote. However, if two paths do not have many common nodes, then they are considered as two independent paths, and they receive positive votes, line~\ref{line:path_votes}-\ref{line:path_votee}. The function \texttt{SimilarNodes} returns the number of overlapping nodes between two paths. In line~\ref{line:path_maxvote}, the function \texttt{MaxVotePath} returns all paths that have the maximum number of votes. There is usually more than one path with the maximum vote, therefore, in line~\ref{line:path_mip}, the function \texttt{MaxInformativePath} selects the path that overall has the maximum information gain.

  Note that path selection algorithm is independent of the IIG algorithm and is a necessarily step to choose an action (trajectory) from the IIG tree. In other words, the IIG tree (graph) can be seen as the space of all feasible informative trajectories whereas the path selection step finds a trajectory from the graph as an executable action, i.e. the final output of Equation~\eqref{eq:incinfmp}. If one decides to include a target (goal) while running IIG, once the target has been reached by the planner the decision has been made and, therefore, there is no need to use Algorithm~\ref{alg:pathselect}. Moreover, if the target is moving, this can be seen as an instance of the target tracking problem~\citep{levine2010information}.

\section{Information-theoretic robotic exploration}
\label{sec:theoretic}

In this section, we present the information-theoretic basis for applying the IIG-tree algorithm to solve the autonomous robotic exploration problem. Since the developed algorithm does not rely on geometric features (frontiers) for map exploration, an alternative criterion is required for mission termination. We use the entropy independence bound theorem to leverage such a criterion. 

\begin{theorem}[Independence bound on entropy]
\label{th:indboundent}
Let $X_1,X_2,...,X_n$ be drawn according to $p(x_1,x_2,...,x_n)$. Then
 \begin{equation}
 \label{jentineq}
  H(X_1,X_2,...,X_n) \leq \sum_{i=1}^n H(X_i) 
 \end{equation}
with equality if and only if the $X_i$ are Independent.
\end{theorem}
\begin{proof}
 The proof follows directly from Theorem~\ref{th:entchainrule} and~\ref{th:condentineq}.
\end{proof}

This theorem states that the joint entropy is always smaller than sum of entropies independently and both sides are equal if and only if the random variables are independent. We start from this inequality and prove that for robotic information 
gathering or map exploration the least upper bound of the average map entropy that is calculated by assuming independence between map points can be a threshold for mission termination. This is formally stated in the following theorem.

\begin{theorem}[The least upper bound of the average map entropy]
\label{th:lupent}
 Let $n \in \mathbb N$ be the number of map points. In the limit, for a completely explored occupancy map, the least upper bound of the average map entropy is given by \mbox{$h_{sat} = H(p_{sat})$}.
\end{theorem}
\begin{proof}
 From Theorem~\ref{th:indboundent} and through multiplying each side of the inequality by $\frac{1}{n}$, we can write the average map entropy as
 \begin{equation}
 \label{eq:avemapentineq}
   \frac{1}{n} H(M) < \frac{1}{n} \sum_{i=1}^n H(M = m^{[i]}) 
 \end{equation}
 by taking the limit as $p(m) \rightarrow p_{sat}$, then
 \begin{equation}
  \nonumber \underset{p(m) \rightarrow p_{sat}}\lim \ \frac{1}{n} H(M) < \underset{p(m) \rightarrow p_{sat}}\lim \ \frac{1}{n} \sum_{i=1}^n H(M = m^{[i]}) 
 \end{equation}
 \begin{equation}
  \nonumber \underset{p(m) \rightarrow p_{sat}}\lim \ \frac{1}{n} H(M) < H(p_{sat}) 
 \end{equation}
 \begin{equation}
  \sup \ \frac{1}{n} H(M) = H(p_{sat})
 \end{equation}
\end{proof}
The result from Theorem~\ref{th:lupent} is useful because the calculation of the right hand side of the inequality~\eqref{eq:avemapentineq} is trivial. In contrast, calculation of the left hand side assuming the map belief is represented by a multi-variate Gaussian, requires maintaining the full map covariance matrix and computation of its determinant. This is not practical, since the map often has a dense belief representation and can be theoretically expanded unbounded (to a very large extent). In the following, we present some notable remarks and consequences of Theorem~\ref{th:lupent}.

\begin{remark}
 The result from Theorem~\ref{th:lupent} also extends to continuous random variables and differential entropy.
\end{remark}
\begin{remark}
 Note that we do not assume any distribution for map points. The entropy can be calculated either with the assumption that
 the map points are normally distributed or treating them as Bernoulli random variables.
\end{remark}
\begin{remark}
 Since $0 < p_{sat} < 1$ and $H(p_{sat}) = H(1 - p_{sat})$, one saturation entropy can be set for the entire map.
\end{remark}
\begin{corollary}[information gathering termination]
\label{coro:infogatherterm}
 Given a saturation entropy $h_{sat}$, the problem of search for information gathering for desired random variables $X_1, X_2,...,X_n$ whose 
 support is alphabet $\mathcal{X}$, can be terminated when $\frac{1}{n} \sum_{i=1}^n H(X_i) \leq h_{sat}$.
\end{corollary}
\begin{corollary}[Map exploration termination]
\label{coro:expterm}
 The problem of autonomous robotic exploration for mapping can be terminated when $\frac{1}{n} \sum_{i=1}^n H(M = m^{[i]}) \leq H(p_{sat})$.
\end{corollary}
The Corollary~\ref{coro:infogatherterm} generalizes the notion of exploration in the sense of information gathering. Therefore, regardless of the quantity
of interest, we can provide a stopping criterion for the exploration mission.
The Corollary~\ref{coro:expterm} is of great importance for the classic robotic exploration for map completion problem as there is no need to resort to geometric
frontiers with a specific cluster size to detect map completion. Another advantage of setting a threshold in the information space is the natural
consideration of uncertainty in the estimation process before issuing a mission termination signal.

\begin{remark}
 Note that MI-based information functions and map exploration termination condition have different saturation probabilities/entropies that are independent and do not necessarily have similar values. Where it is not clear from the context, we make the distinction explicitly clear.
\end{remark}

\section{Results and discussion}
\label{sec:iigresults}

In this section, we examine the proposed algorithms in several scenarios. We use MATLAB implementations of the algorithms that are also made publicly available on:~\textcolor{BrickRed}{\url{https://github.com/MaaniGhaffari/sampling_based_planners}}.

We first design experiments for comparison of information functions under various sensor parameters such as the number of beams and the sensor range, and their effects on the convergence of IIG. Although the primary objective of this experiment is to evaluate the performance of IIG using each information function, we note that such an experiment can also facilitate sensor selection. In other words, given the map of an environment and a number of sensors with different characteristics, how to select the sensor that has a reasonable balance of performance and cost.

In the second experiment a robot explores an unknown environment; it needs to solve SLAM incrementally, estimate a dense occupancy map representation suitable for planning and navigation, and automatically detect the completion of an exploration mission. Therefore, the purpose of information gathering is map completion while maintaining accurate pose estimation, i.e.{\@} planning for estimation. In these experiments the robot pose is estimated through a pose graph algorithm such as Pose SLAM~\citep{ila2010information}, and the map of the unknown environment is computed using the Incremental Gaussian Processes Occupancy Mapping (I-GPOM)~\citep{maani2014com,jadidi2016gaussian}. Therefore, the state which includes the robot pose and the map is partially observable. 

In the third scenario, we demonstrate another possible application of IIG in a lake monitoring experiment using experimental data collected by an Autonomous Surface Vehicle (ASV). The dataset used for this experiment is publicly available and also used in the original RIG article~\citep{hollinger2014sampling}~\footnote{The dataset is available on:~\textcolor{BrickRed}{\url{http://research.engr.oregonstate.edu/rdml/software}}}. 

We conclude this section by discussing the limitations of the work including our observations and conjectures.

\begin{table}[th!]
\footnotesize
\centering
\caption{Parameters for IIG-tree experiments. ``Online'' parameters are only related to the exploration experiments.}
\begin{tabular}{lll}
\toprule
Parameter			& Symbol	& Value \\ \midrule
\multicolumn{3}{l}{$-$ General parameters:} \\
Occupied probability  		& $p_{occ}$	& 0.65		\\
Unoccupied probability 		& $p_{free}$	& 0.35		\\  
Initial position 		& $x_{init}$	& [10,2] $\m$	\\ 
Map resolution			& $\delta_{map}$ & 0.2 $\m$	\\
$I_{RIC}$ threshold		& $\delta_{RIC}$ & 5e-4		\\
$I_{RIC}$ threshold (Online)	& $\delta_{RIC}$ & 1e-2		\\
\multicolumn{3}{l}{$-$ MI-based parameters:} \\
Hit std				& $\sigma_{hit}$	& 0.05 m	\\
Short decay			& $\lambda_{short}$	& 0.2 m	 	\\
Hit weight			& $z_{hit}$	& 0.7		\\
Short weight			& $z_{short}$	& 0.1		\\ 
Max weight			& $z_{max}$	& 0.1	 	\\
Random weight			& $z_{rand}$	& 0.1		\\ 
Numerical integration resolution & $s_z$	& 2 $\m^{-1}$	\\
Saturation probability 		& $p_{sat}$	& 0.05		\\
Saturation probability (Online)	& $p_{sat}$	& 0.3		\\
Occupied belief  		& $b_{occ}$	& 1.66		\\
Unoccupied belief 		& $b_{free}$	& 0.6		\\
\multicolumn{3}{l}{$-$ Covariance function hyperparameters:} \\
characteristic length-scale	& $l$		& 3.2623 $\m$	\\
Signal variance			& $\sigma_f^2$	& 0.1879	\\
\multicolumn{3}{l}{$-$ Robot motion model:} \\
Motion noise covariance		& \multicolumn{2}{l}{$\boldsymbol Q=\diag(0.1\m, 0.1\m, 0.0026\rad)^2$}	\\ 
Initial pose uncertainty 	& \multicolumn{2}{l}{$\boldsymbol\Sigma_{init}=\diag(0.4\m,0.1\m,0\rad)^2$}	\\ 
\multicolumn{3}{l}{$-$ Path selection:} \\
Minimum path length coefficient 	& $\kappa$	& 0.4 	\\
Path similarity ratio			& $s_{ratio}$	& 0.6  	\\
\multicolumn{3}{l}{$-$ Termination condition (Online):} \\
Saturation entropy		& $h_{sat} = H(p_{sat}=0.1)$	& 0.3251 $\mathrm{nats}$ 	\\	\bottomrule
\end{tabular}
\label{tab:iigparam}
\end{table}

\subsection{Experimental Setup}
\label{subsec:iigexpsetup}

We first briefly describe the experiment setup that is used in Subsections~\ref{subsec:inffunccomp} and \ref{subsec:oniigres}. The parameters for experiments and the information functions are listed in Table~\ref{tab:iigparam}. The robot is equipped with odometric and laser range-finder sensors. The environment is constructed using a binary map of obstacles. For GPVR-based algorithms, the covariance function is Mat\'ern ($\nu = 5/2$)
\begin{equation}
 k_{VR} = \sigma_f^2 k_{\nu=5/2}(r) = \sigma_f^2 (1+\frac{\sqrt{5}r}{l}+\frac{5r^2}{3l^2})\exp(-\frac{\sqrt{5}r}{l})
\end{equation}
with hyperparameters that were learned prior to the experiments and are available in Table~\ref{tab:iigparam}. The modified kernel in UGPVR algorithm was calculated using Gauss-Hermite quadrature with $11$ sample points. 

The path selection parameters and $\delta_{RIC}$ are found empirically; however, these quantities are non-dimensional, and we expect that one can tune them easily. In particular, $\delta_{RIC}$ which sets the planning horizon depends on the desired outcome. Setting $\delta_{RIC} = 0$ is the ideal case which makes the planner infinite-horizon, but in practice due to the numerical resolution of computer systems, a reasonably higher value (Table~\ref{tab:iigparam}) should be chosen. Furthermore, a very small value of $\delta_{RIC}$, analogous to numerical optimization algorithms, results in the late convergence of the planner which is usually undesirable.

The termination condition requires a saturation entropy. For the case of occupancy mapping, the problem is well-studied, and we are aware of desired saturation probability that leads to an occupancy map with sufficient confidence for the status of each point/cell. Also, when a process posterior is represented using mass probabilities such as occupancy maps, finding a saturation probability and, therefore the corresponding saturation entropy, is possible. Perhaps a hard case is when the posterior is a probability density function. In such situations, further knowledge of the marginal posterior distribution of the state variables is required. For example, if we wish to model state variables as Gaussian random variables with the desired variance, we can use differential entropy of the Gaussian distribution to compute the saturation entropy.

The robot pose covariance was approximated using the uncertainty propagation and local linearization of the robot motion model using Equations~\eqref{eq:reom} and \eqref{eq:predcov}. Each run was continued until the algorithm converges without any manual intervention. To obtain the performance of each method in the limit and also examine the convergence of IIG, we used an unlimited budget in all experiments. Moreover, the cost is computed as the Euclidean distance between any two nodes and the budget is the maximum allowable travel distance which is unlimited. The online parameters in Table~\ref{tab:iigparam} refer to exploration experiments and are applied in Subsection~\ref{subsec:oniigres}. The exploration experiments use Corollary~\ref{coro:expterm} as the termination condition for the entire mission.

\subsection{Comparison of information functions}
\label{subsec:inffunccomp}

We now compare the proposed information functions by running IIG-tree using different sensor parameters  in the Cave map~\citep{Radish_data_set}. As in this experiment the map of the environment is given, the map is initialized as an occupancy grid map by assigning each point the occupied or unoccupied probability according to its ground truth status. For GPVR-based methods, an initial variance map is set to the value of $1$ for all points. The experiments are conducted by increasing the sensor number of beams and range, and the results are collected in Tables~\ref{tab:ifuncompbeam} and ~\ref{tab:ifuncomprange}, respectively.

\begin{table}[t]
\footnotesize
\centering
\caption{Comparison of the information functions in offline IIG-tree experiments by increasing the sensor number of beams from $10$ to $50$. The figures  are averaged over $30$ experiments (mean $\pm$ standard error). For all experiments the following parameters are set in common $r_{max} = 5 \m$, $\delta_{RIC} = 5e-4$. The total information gain is reported in $\mathrm{nats}$.}
\begin{tabular}{lccc}
\toprule
No. of beams $n_z$	& \multicolumn{1}{c}{10} 	& \multicolumn{1}{c}{20} & \multicolumn{1}{c}{50} \\ \midrule
 \multicolumn{4}{c}{Mutual information (MI)} \\ \midrule
 
Plan. time (min)		& \textbf{6.88 $\pm$ 0.16}	& 8.44 $\pm$ 0.27	& 14.64 $\pm$ 0.81	 	\\ 
No. of samples			& 911 $\pm$ 19			& 588 $\pm$ 13		& \textbf{435 $\pm$ 20}		\\ 
No. of nodes			& 402 $\pm$ 5			& 280 $\pm$ 5		& \textbf{181 $\pm$ 8} 		\\  
Tot. info. gain			& \textbf{1.40e+04 $\pm$ 148}	& 1.36e+04 $\pm$ 173	& 1.25e+04 $\pm$ 313 		\\  
Tot. cost (m)			& 336.3 $\pm$ 3.8		& 258.5 $\pm$ 3.7	& \textbf{187.2 $\pm$ 7.7} 	\\  \midrule\midrule

\multicolumn{4}{c}{Mutual information upper bound (MIUB)} \\ \midrule
 
Plan. time (min)		& \textbf{4.53 $\pm$ 0.16}	& 6.95 $\pm$ 0.28	& 12.88 $\pm$ 0.51	 	\\ 
No. of samples			& 1014 $\pm$ 24			& 629 $\pm$ 17		& \textbf{486 $\pm$ 15}		\\ 
No. of nodes			& 444 $\pm$ 8			& 303 $\pm$ 7		& \textbf{212 $\pm$ 5} 		\\  
Tot. info. gain 		& \textbf{2.04e+04 $\pm$ 169}	& 1.96e+04 $\pm$ 207	& 1.98e+04 $\pm$ 403 		\\  
Tot. cost (m)			& 358.3 $\pm$ 5.6		& 272.0 $\pm$ 5.0	& \textbf{211.5 $\pm$ 4.4}	\\  \midrule\midrule

\multicolumn{4}{c}{GP variance reduction (GPVR)} \\ \midrule
 
Plan. time (min)		& \textbf{9.31 $\pm$ 0.39}	& 13.10 $\pm$ 0.41	& 15.07 $\pm$ 0.61		\\ 
No. of samples		& 1677 $\pm$ 53			& 1135 $\pm$ 36		& \textbf{874 $\pm$ 36}		\\ 
No. of nodes			& 616 $\pm$ 15			& 477 $\pm$ 13		& \textbf{382 $\pm$ 11}		\\  
Tot. info. gain 		& 7.84e+03 $\pm$ 62		& 9.123e+03 $\pm$ 80	& \textbf{1.31e+04 $\pm$ 176} 	\\  
Tot. cost (m)			& 490.5 $\pm$ 9.5		& 392.4 $\pm$ 7.9	& \textbf{320.6 $\pm$ 7.2} 	\\  \midrule\midrule

\multicolumn{4}{c}{Uncertain GP variance reduction (UGPVR)} \\ \midrule
 
Plan. time (min)		& \textbf{19.75 $\pm$ 0.48}	& 27.89 $\pm$ 0.94	& 47.79 $\pm$ 2.16		\\ 
No. of samples			& 4746 $\pm$ 136		& 2633 $\pm$ 102	& \textbf{1828 $\pm$ 81}	\\ 
No. of nodes			& 1344 $\pm$ 28			& 913 $\pm$ 25		& \textbf{704 $\pm$ 22} 	\\  
Tot. info. gain 		& 3.19e+03 $\pm$ 16		& 3.73e+03 $\pm$ 22	& \textbf{5.55e+03 $\pm$ 52}	\\  
Tot. cost (m)			& 894.0 $\pm$ 16.3		& 621.4 $\pm$ 15.6	& \textbf{497.8 $\pm$ 12.5}	\\  \bottomrule
\end{tabular}
\label{tab:ifuncompbeam}
\end{table}

\begin{table}[t]
\footnotesize
\centering
\caption{Comparison of the information functions in offline IIG-tree experiments by increasing the sensor range, $r_{max}$, from $5\m$ to $20\m$ (averaged over $30$ experiments, mean $\pm$ standard error). For all experiments the following parameters are set in common $n_z = 10$; $\delta_{RIC} = 5e-4$. The total information gain is reported in $\mathrm{nats}$.}
\begin{tabular}{lccc}
\toprule
Range $r_{max}$ (m)	& \multicolumn{1}{c}{5} 	& \multicolumn{1}{c}{10} 	& \multicolumn{1}{c}{20} \\ \midrule
 \multicolumn{4}{c}{Mutual information (MI)} \\ \midrule
 
Plan. time (min)	& \textbf{5.00 $\pm$ 0.13}	& 6.60 $\pm$ 0.10		& 7.88 $\pm$ 0.11 		\\ 
No. of samples		& 979 $\pm$ 22			& 782 $\pm$ 21			& \textbf{780 $\pm$ 20}		\\ 
No. of nodes		& 423 $\pm$ 7			& \textbf{352 $\pm$ 8}		& 360 $\pm$ 6 			\\  
Tot. info. gain		& 1.40e+04 $\pm$ 146		& 1.45e+04 $\pm$ 177		& \textbf{1.57e+04 $\pm$ 220} 	\\  
Tot. cost (m)		& 351.7 $\pm$ 4.8		& \textbf{307.8 $\pm$ 5.4}	& 309.2 $\pm$ 4.0 		\\  \midrule\midrule

\multicolumn{4}{c}{Mutual information upper bound (MIUB)} \\ \midrule
 
Plan. time (min)	& 3.67 $\pm$ 0.10		& \textbf{3.56 $\pm$ 0.08}	& 3.87 $\pm$ 0.11		\\ 
No. of samples		& 1015 $\pm$ 23			& 813 $\pm$ 15			& \textbf{811 $\pm$ 20}		\\ 
No. of nodes		& 444 $\pm$ 7			& \textbf{368 $\pm$ 6}		& 369 $\pm$ 7 			\\  
Tot. info. gain		& \textbf{2.04e+04 $\pm$ 161}	& 1.98e+04 $\pm$ 228		& 1.97e+04 $\pm$ 272 		\\  
Tot. cost (m)		& 360.1 $\pm$ 5.2		& \textbf{309.9 $\pm$ 4.2}	& 311.6 $\pm$ 4.5 		\\  \midrule\midrule

\multicolumn{4}{c}{GP variance reduction (GPVR)} \\ \midrule
 
Plan. time (min)	& \textbf{6.20 $\pm$ 0.22}	& 7.03 $\pm$ 0.24		& 6.60 $\pm$ 0.22	 	\\ 
No. of samples		& 1745 $\pm$ 71			& 1393 $\pm$ 41			& \textbf{1248 $\pm$ 48}	\\ 
No. of nodes		& 636 $\pm$ 18			& 559 $\pm$ 13			& \textbf{523 $\pm$ 16}		\\  
Tot. info. gain		& \textbf{7.97e+03 $\pm$ 56}	& 7.80e+03 $\pm$ 71		& 7.83e+03 $\pm$ 62 		\\  
Tot. cost (m)		& 501.7 $\pm$ 11.2		& 445.0 $\pm$ 7.5		& \textbf{422.7 $\pm$ 9.8} 	\\  \midrule\midrule

\multicolumn{4}{c}{Uncertain GP variance reduction (UGPVR)} \\ \midrule
 
Plan. time (min)	& \textbf{20.52 $\pm$ 0.65}	& 26.3 $\pm$ 0.68	& 26.68 $\pm$ 0.54			\\ 
No. of samples		& 4994 $\pm$ 167		& 3661 $\pm$ 109	& \textbf{3385 $\pm$ 87}		\\ 
No. of nodes		& 1404 $\pm$ 35			& 1139 $\pm$ 25		& \textbf{1078 $\pm$ 19} 		\\
Tot. info. gain		& 3.17e+03 $\pm$ 19		& 3.25e+03 $\pm$ 21	& \textbf{3.27e+03 $\pm$ 21}		\\  
Tot. cost (m)		& 934.4 $\pm$ 22.3		& 765.2 $\pm$ 15.5	& \textbf{724.1 $\pm$ 12.3}		\\  \bottomrule
\end{tabular}
\label{tab:ifuncomprange}
\end{table}

In the first experiment, we use $10$, $20$, and $50$ sensor beams with the maximum sensor range fixed at $r_{max} = 5 \m$. The information functions used are MI (Algorithm~\ref{alg:funcmi2}), MIUB (Algorithm~\ref{alg:funcmiub}), GPVR (Algorithm~\ref{alg:gpvr}), and UGPVR (Algorithm~\ref{alg:ugpvr}), and the results are presented in Table~\ref{tab:ifuncompbeam}. Convergence is detected when the average of penalized relative information contribution $I_{RIC}$ over a window of size $30$ drops below the threshold $\delta_{RIC} = 5e-4$. The total information gain/cost is calculated using the sum of all edges information/costs. Therefore, it denotes the total information/cost over the searched space and not a particular path. This makes the results independent of the path selection algorithm.

\begin{figure}[t]
  \centering  
  \subfloat{
    \includegraphics[width=0.33\columnwidth,trim={1.5cm 1.5cm 1.5cm 1.5cm},clip]{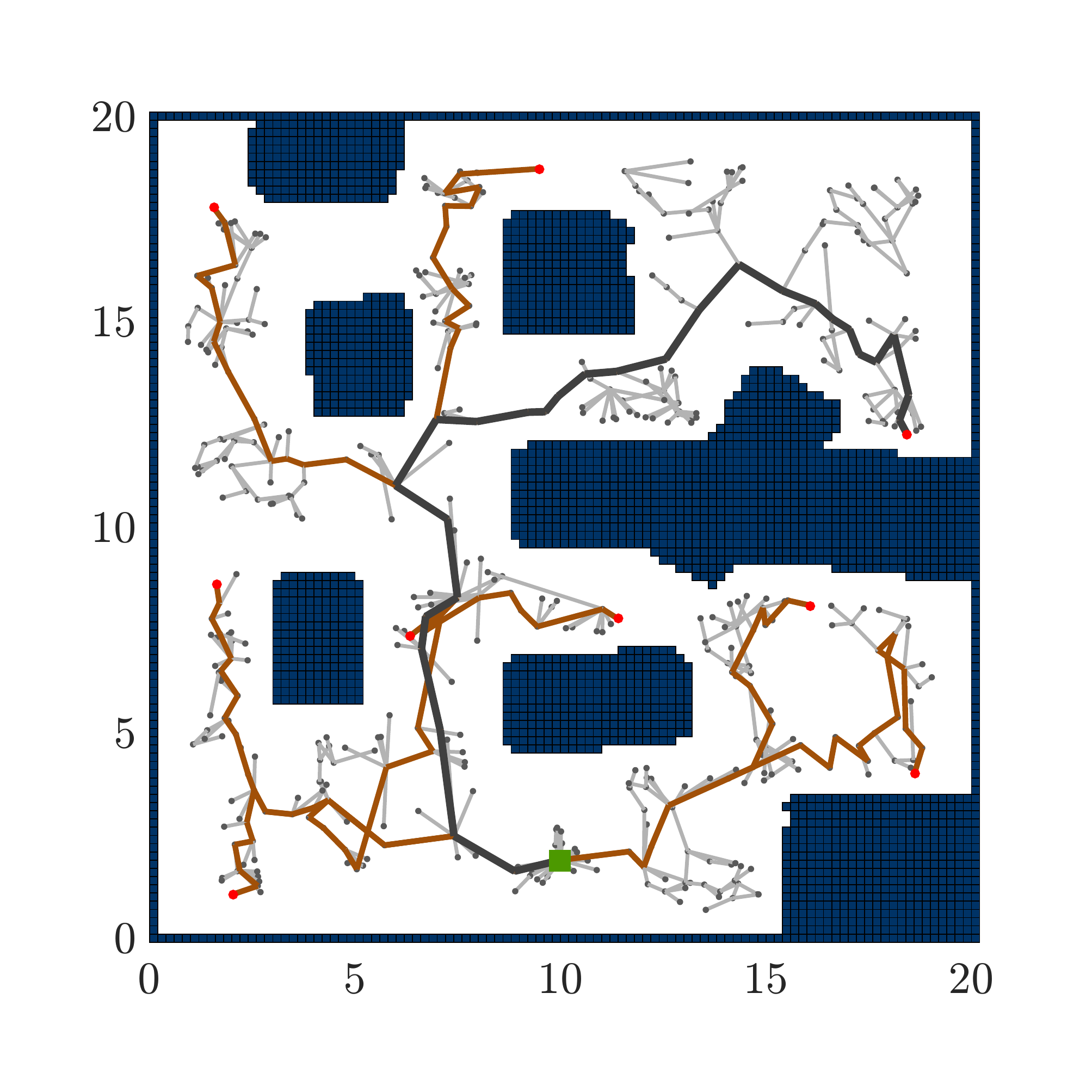}
    \label{fig:mi_01}
    }
  \subfloat{
    \includegraphics[width=0.33\columnwidth,trim={1.5cm 1.5cm 1.5cm 1.5cm},clip]{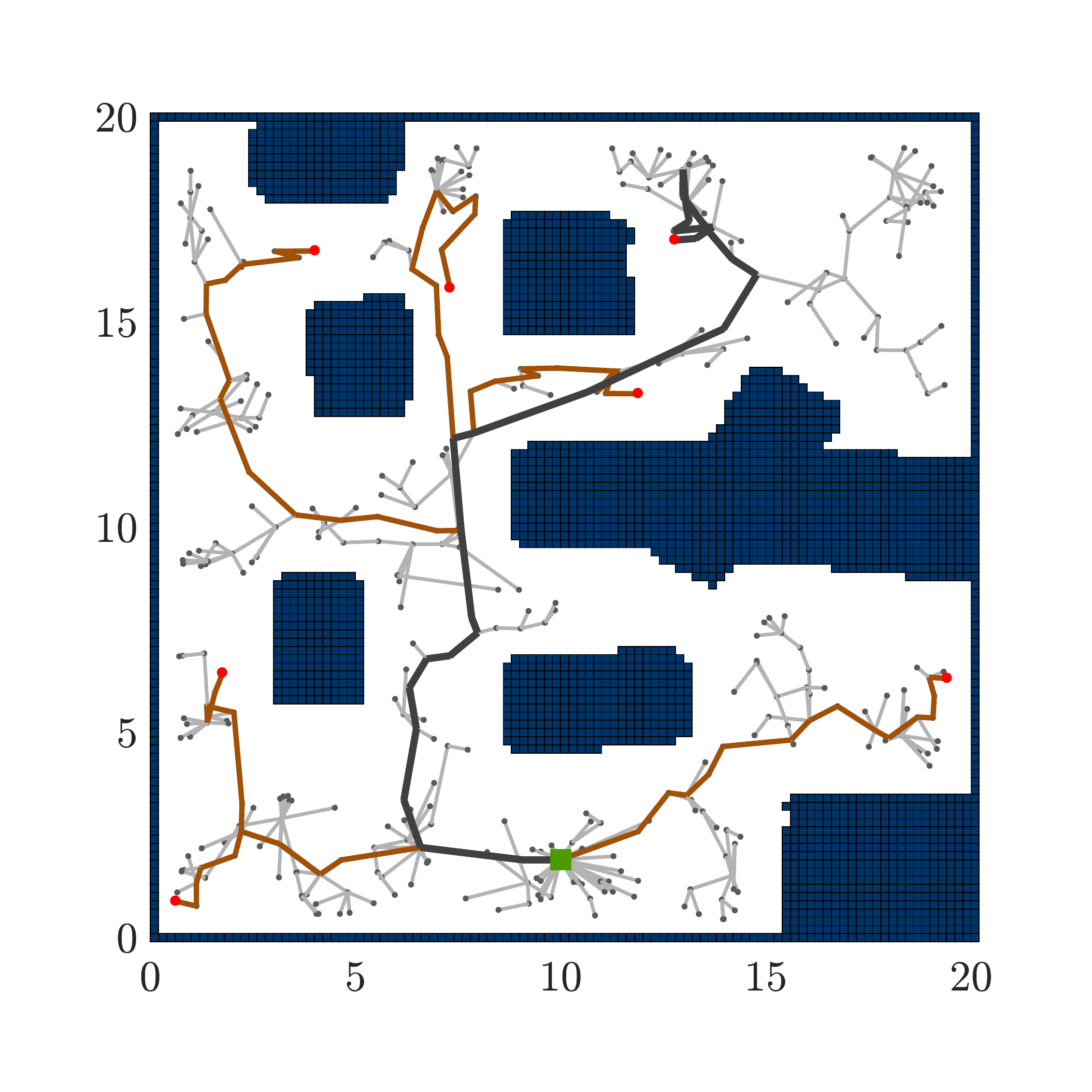}
    \label{fig:miub_01}
    }
  \subfloat{
    \includegraphics[width=0.33\columnwidth,trim={1.5cm 1.5cm 1.5cm 1.5cm},clip]{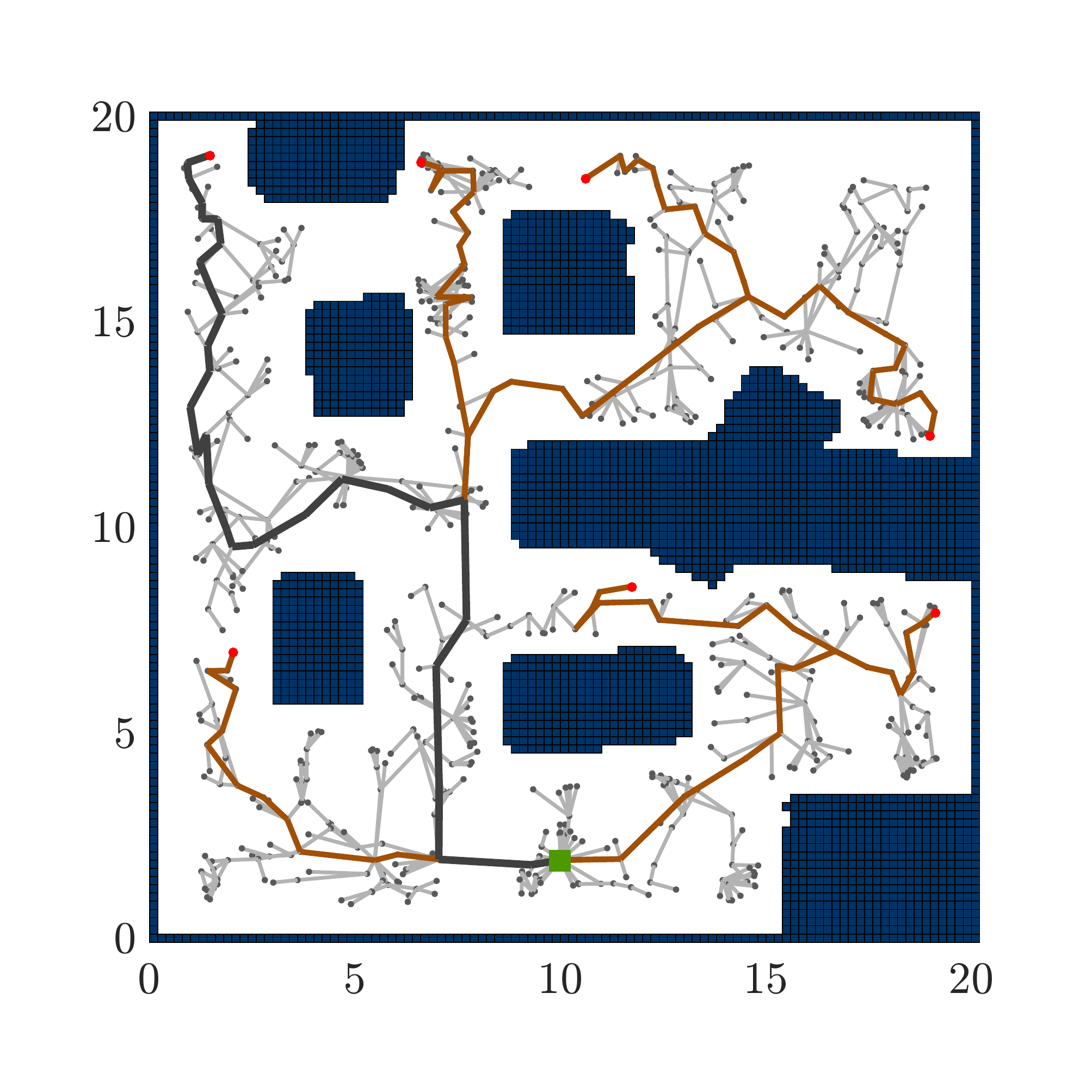}
    \label{fig:gpvr_01}
    }\\
  \subfloat{
    \includegraphics[width=0.33\columnwidth,trim={1.5cm 1.5cm 1.5cm 1.5cm},clip]{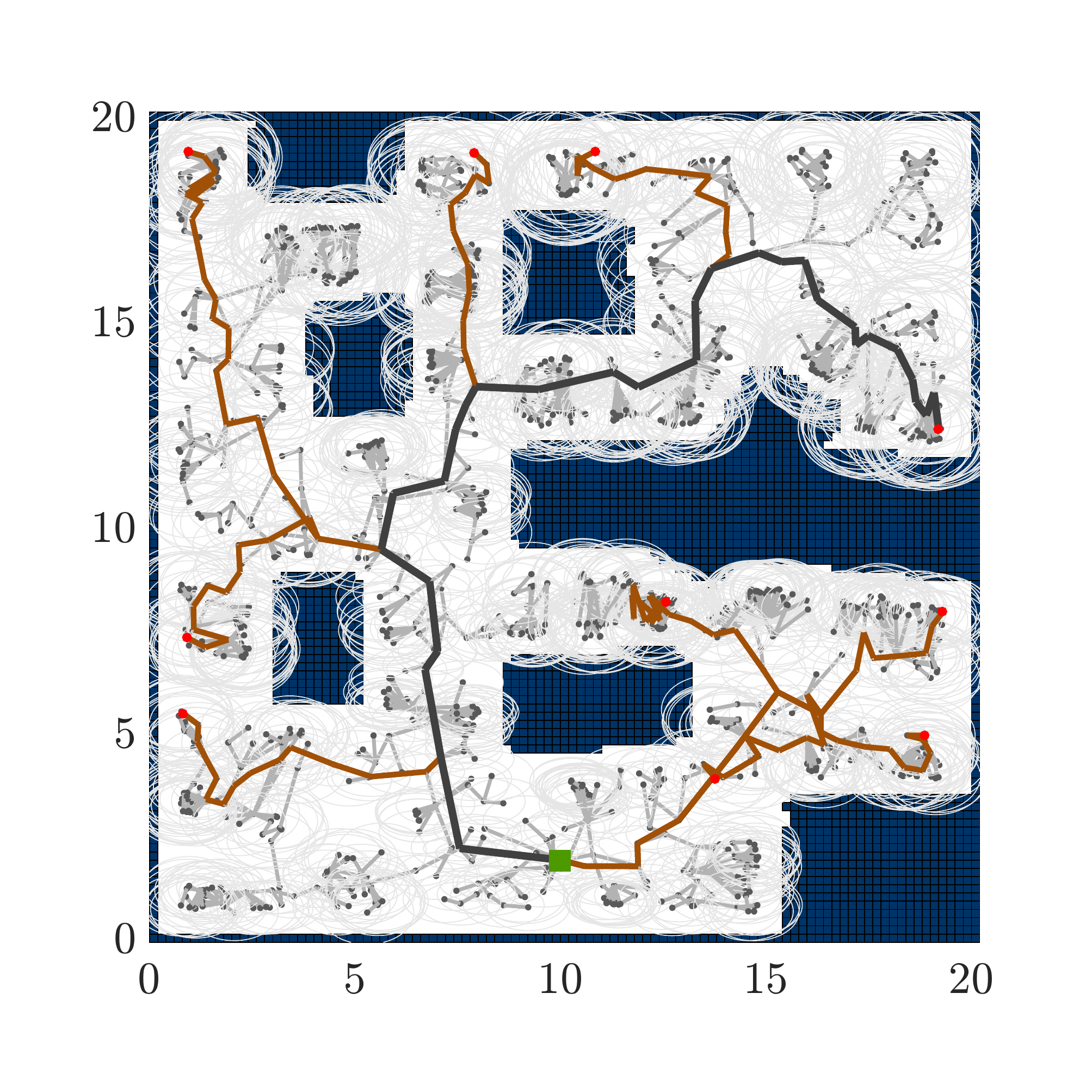}
    \label{fig:ugpvr_01}
    }
  \subfloat{
    \includegraphics[width=0.33\columnwidth,trim={0cm 0cm 1.5cm 1.cm},clip]{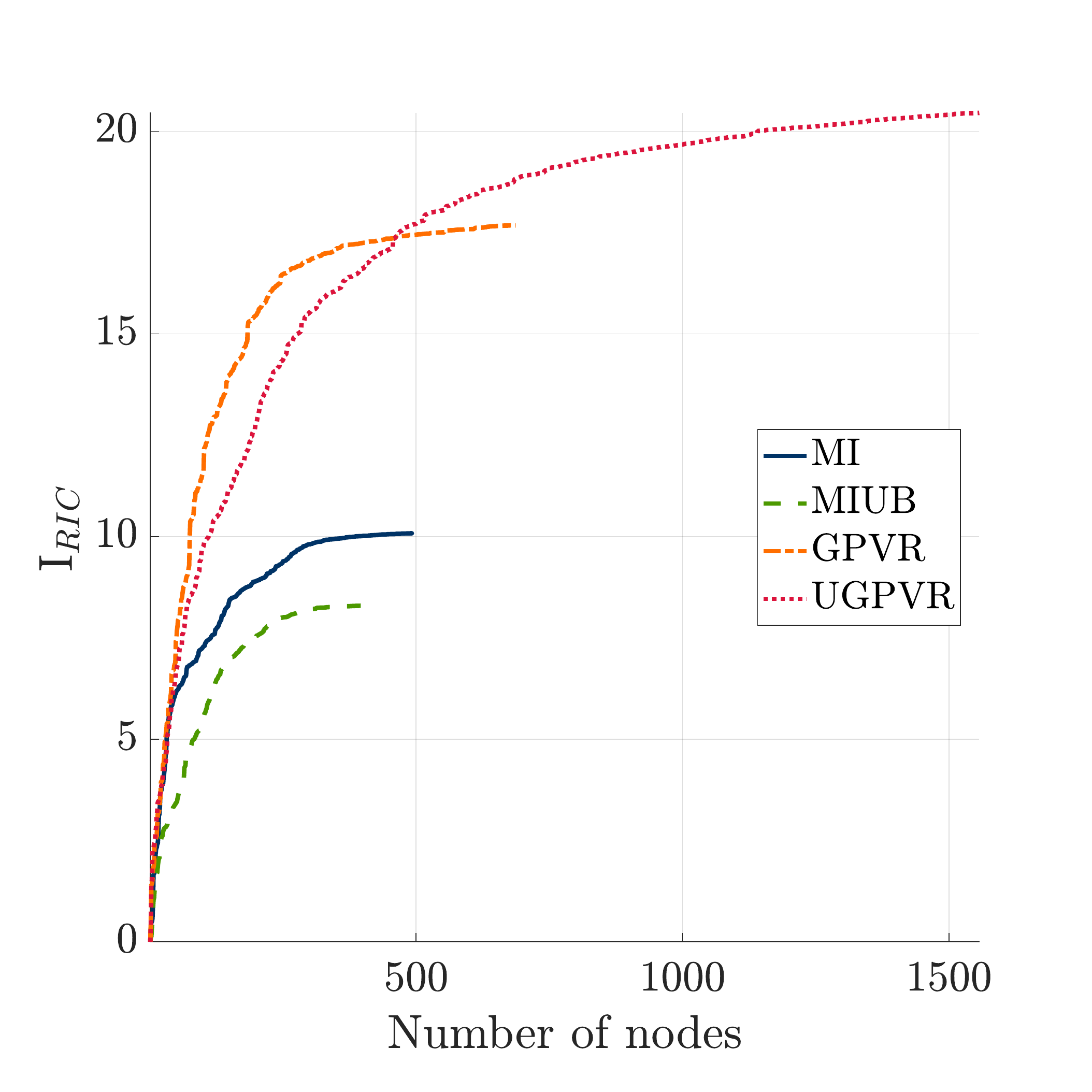}
    \label{fig:Iric_offline}
    }
  \subfloat{
    \includegraphics[width=0.33\columnwidth,trim={0cm 0cm 1.5cm 1.cm},clip]{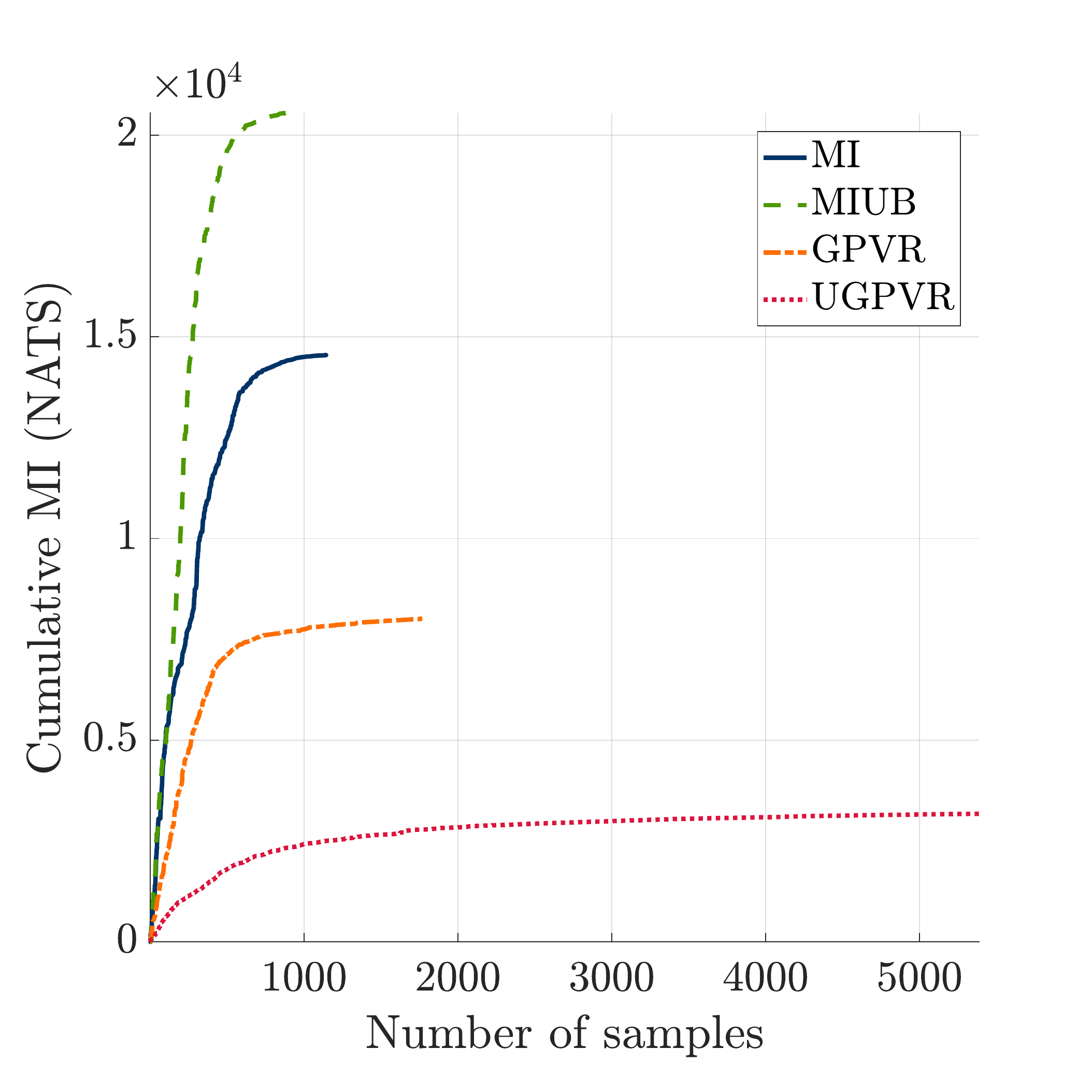}
    \label{fig:mi_offline}
    }
  \caption{An example of results from running the IIG-tree algorithm in the Cave map. From top left, the plots show the IIG-tree graph for MI, MIUB, GPVR, and UGPVR MI information functions. The selected paths computed using Algorithm~\ref{alg:pathselect} are also shown where the most informative path is shown using a darker color. MI and UGPVR show similar behaviors despite different quantitative outcomes. The last two plots in the bottom row, from left, show the convergence graph of the penalized relative information contribution $I_{RIC}$ and the evolution of the cumulative information gain for all information functions, respectively. The sensor range and the number of beams are set to $5\m$ and $10$, respectively.}
  \label{fig:offiig}
\end{figure}

From Table~\ref{tab:ifuncompbeam}, MIUB has the lowest runtime as expected, followed by MI, GPVR, and UGPVR, respectively. The calculation of MI can be more expensive than GPVR-based algorithms, but with sparse sensor observations (only $10$ beams) and a very coarse numerical integration resolution it performs faster. In particular, small number of sensor beams, given a sufficiently fine map resolution, satisfies the common assumption of independence between the noise in each individual measurement value~\citep{thrun2005probabilistic}. The conditional independence assumption of measurements for calculation of mutual information is also discussed in~\citet{charrow2014cooperative,charrow2015information}. Empirically, for MI function, we can observe in Table~\ref{tab:ifuncompbeam} that as the number of beams increases, the total information gain decreases. We note that increasing the number of beams lead to measurements becoming more correlated, thereby violating the conditional independence assumption between measurements. From Proposition~\ref{prop:miineq}, we know that the estimated mutual information is always less than the actual mutual information; therefore, increasing the number of beams, given the map resolution used in the experiments, reduces the MI estimation accuracy. Furthermore, while our proof of Proposition~\ref{prop:miineq} is for the case of Gaussian random variables, \emph{data-processing inequality}~\citep{cover2012elements} states that processing of measurements, such as those with simplified correlation models, deterministic or random, can not increase the information gain. It is important to note that observations presented above stem from the assumptions used and developed models and algorithms. Therefore these results should not be interpreted to indicate that in general using a sparse set of measurements the information gain calculation accuracy increases.

For all the compared algorithms, as the number of beams increases, i.e.{\@} taking more observations, the computations take longer, and the number of samples/nodes reduces. However, MI has the fastest convergence speed with, approximately, half of the number of samples taken by GPVR. Incorporating the pose uncertainty in UGPVR leads to a slower convergence. This can be explained as a result of the reduction in the information content of each set of observations by adding uncertainty to the information gain calculation. Another interpretation is that the UGPVR algorithm is, relatively speaking, less greedy or more conservative for information gathering. For GPVR-based information functions, increasing the number of sensor beams, leads to higher total information gain due to a larger training set $\mathcal{D}$. MI in comparison with MIUB has more realistic information gain estimation which is reflected in less total cost on average. In other words, the upper bound of the estimated MI, most likely, overestimate the actual information gain (Lemma~\ref{lem:miub}). The latter results also confirm the advantage of maximizing the mutual information rather than minimizing the entropy which is discussed in~\cite{krause2008near}.

Figure~\ref{fig:offiig} illustrates examples of IIG-tree results using different information functions in the Cave map. The sensor range and the number of beams are set to $5\m$ and $10$, respectively. The extracted paths are also shown and the most informative path is separated using a darker color. The convergence of $I_{RIC}$ for all information is shown in bottom middle plot, and the bottom right plot shows the cumulative information gain demonstrating its diminishing return over time (as the number of samples grows). The higher value of $I_{RIC}$ can be correlated with a denser IIG graph. Also, the MIUB essentially minimizes the entropy which tends to explore the boundary of the space, unlike MI that takes into account the perception field of the sensor. This behavior is also observed and discussed in~\citet[Figure 4]{krause2008near}. The GPVR and UGPVR curves have similar trends; however, due to the incorporated pose uncertainty the amount of information gain is remarkably lower which explains the longer tail of information gain evolution before the convergence of the algorithm and the higher relative information contribution from each node. In other words, farther nodes have higher pose uncertainties. Therefore the discrimination between farther and nearer nodes from relative information contribution is less than the case where the pose uncertainty is ignored.

In the second experiment, the number of beams is kept fixed at $n_z = 10$, but the sensor range is increased from $5\m$ to $20\m$. The results shown in Table~\ref{tab:ifuncomprange} are consistent with the previous test. The variations are smaller as the growth in the size of the observation set (training set for GPVR-based functions), due to increased range, is much smaller, mostly as a result of the geometry of the environment. Note that the first column of Table~\ref{tab:ifuncomprange} and first column of Table~\ref{tab:ifuncompbeam} are for two different runs with same parameters. The differences in planning time are due to the fact that the experiments were carried out in a high-performance computing facility where the reported time is affected by other users activities. Furthermore, the difference between $r_{max} = 10$ and $r_{max} = 20$ cases is marginal which shows increasing the sensor range more than $10 \m$ does not improve the information gathering process in this particular environment. 

Finally, a practical conclusion can be that increasing the sensor range and the number of beams (observations) increases the computational time but does not necessarily make the corresponding information function superior. As long as the approximation can capture the essence of information gain estimation consistently, the search algorithm performs well. Clearly the shape of the environment affects the result, e.g.{\@} the maximum and minimum perception field at different areas of the map. For the experiments in the following section, we select $n_z = 10$ and $r_{max} = 5 \m$ which lead to the fastest computational time with reasonable total information gain and cost values.

\subsection{Robotic exploration in unknown environments}
\label{subsec:oniigres}

In this section, we examine use of the proposed algorithms for autonomous robotic exploration in an unknown environment (Figure~\ref{fig:cave_iig_baxplot}). This requires solving the SLAM problem for localization possibly using observations to features in the environment. It is also necessary to build a dense map representation that shows occupied and unoccupied regions for planning and navigation. We solve the localization problem using Pose SLAM, the occupancy mapping using \mbox{I-GPOM}, and the planning using IIG-tree with different information functions. The \mbox{I-GPOM} models the occupancy map of the environment as GPs which can be used by IIG to implement the proposed GPVR-based information functions.

\begin{figure}[t]
  \centering 
  \includegraphics[width=0.5\columnwidth,trim={1.5cm 1.5cm 1.5cm 1.5cm},clip]{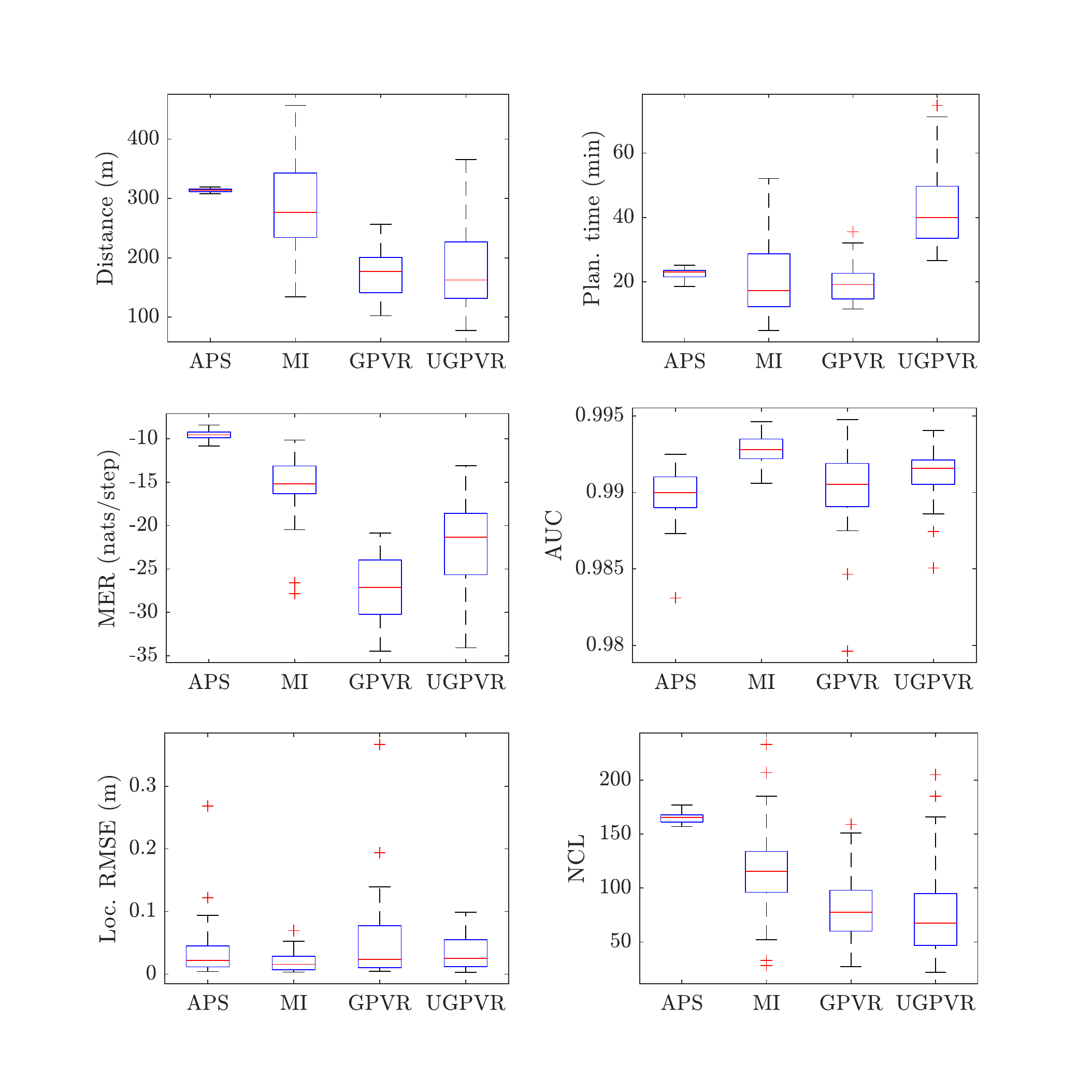}
  \caption{The box plots show the statistical summary of the comparison of different exploration strategies in the Cave dataset. The compared methods are APS, IIG-MI, IIG-GPVR, and IIG-UGPVR. The results are from $30$ independent exploration rounds for each technique. The time comparison is based on the planning time. Generally, APS is faster since the OGM is faster than I-GPOM; however, IIG-MI can also be implemented using OGMs.}
  \label{fig:cave_iig_baxplot}
\end{figure}

We also compare our results with Active Pose SLAM (APS)~\citep{valencia2012active} which is an information gain-based technique that considers explicit loop-closures by searching through nearby poses in the pose graph. As APS uses an Occupancy Grid Map (OGM)~\citep{moravec1985high,elfes1987sonar} for mapping and frontiers~\citep{yamauchi1997frontier} as exploration target candidates, we generate the equivalent \mbox{I-GPOM} using all poses and observations at the end of an experiment in order to compare mapping performance. Pose SLAM parameters were set and fixed regardless of the exploration method. The localization Root Mean Square Error (RMSE) was computed at the end of each experiment by the difference in the robot estimated and ground truth poses. The occupancy maps are compared using the Area Under the receiver operating characteristic Curve (AUC) as AUC is known to be more suitable for domains with skewed class distribution and unequal classification error costs~\citep{fawcett2006introduction}. The probability that the classifier ranks a randomly chosen positive instance higher than a randomly chosen negative instance can be understood using the AUC of the classifier. The AUC values of $0.5$ and $1$ correspond to random and ideal performance, respectively.

\begin{figure}[t]
  \centering 
  \includegraphics[width=0.5\columnwidth,trim={0cm 0cm 1.5cm 1cm},clip]{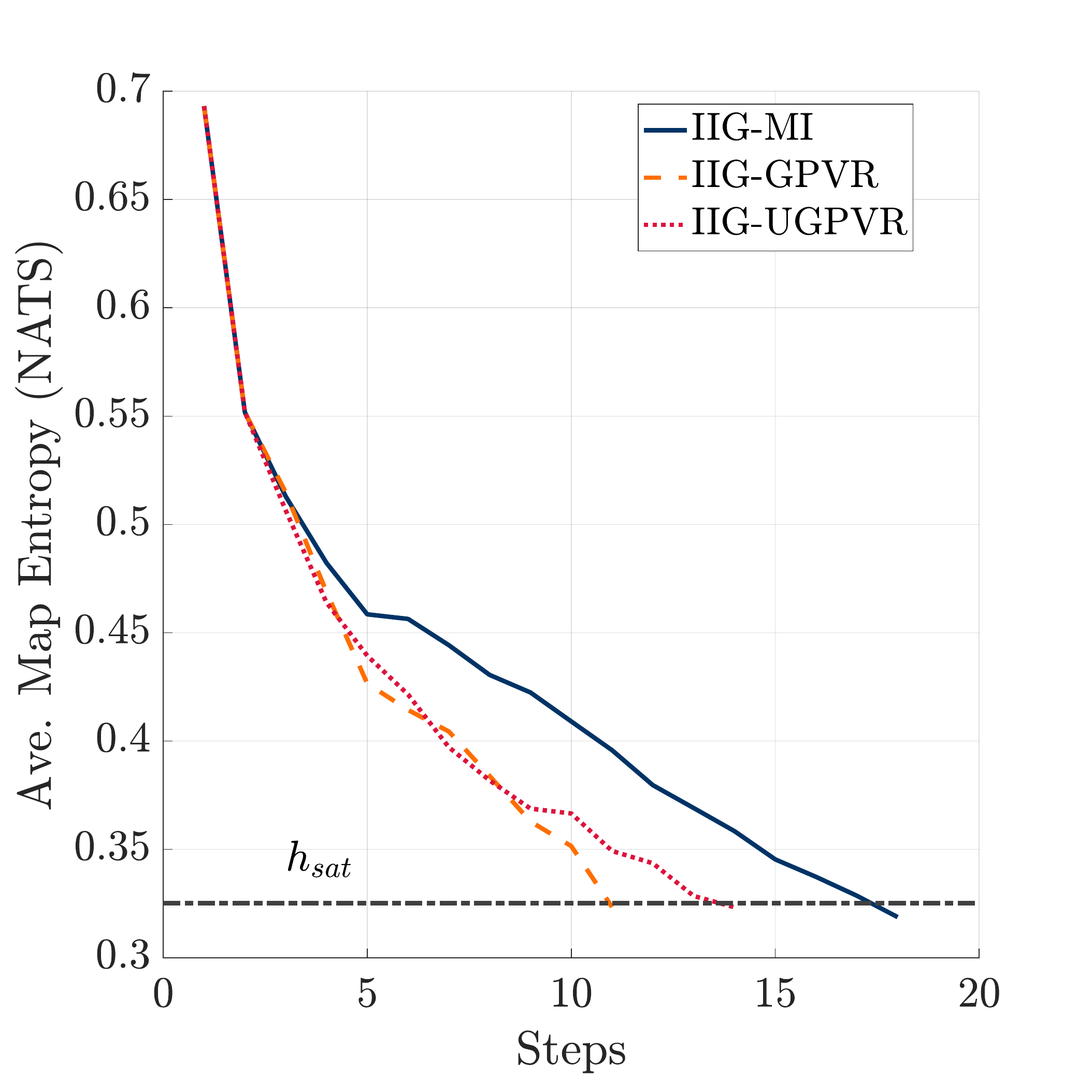}
  \caption{The evolution of the termination condition of different IIG exploration strategies in the Cave dataset. The curves are strongly related to the map entropy rate, i.e. a higher map entropy reduction rate results in faster termination of the mission. Unlike IIG, APS uses a geometric condition as the stopping criterion, and the mission was terminated when there was no cluster of geometric frontiers with the size larger than $12$ cells.}
  \label{fig:termcond_cave}
\end{figure}

Figure~\ref{fig:cave_iig_baxplot} shows the statistical summary of the results from exploration experiments in the Cave environment. The results are from $30$ exploration rounds in which each mission was terminated automatically using Corollary~\ref{coro:expterm} and saturation probability $p_{sat} = 0.1$ ($h_{sat} = 0.3251~\mathrm{nats}$). For APS, the mission was terminated when there was no cluster of geometric frontiers with the size larger than $12$ cells; in addition, any frontier closer than $2 \m$ to the robot was ignored to speed up the map exploration. The IIG-GPVR has the lowest travel distance, highest map entropy reduction rate, but larger localization error. IIG-UGPVR demonstrates a more conservative version of IIG-GPVR where the pose uncertainty propagation provides better localization and mapping performance at the expense of more planning time. This behavior is expected and shows the consistency between our problem definition and algorithmic development. We note that our time comparison is focused on the planning time with mapping implemented using I-GPOM. Mapping with OGM will speed up both APS and IIG-MI.

IIG-MI has the lowest localization error with a similar planning time on average. IIG-MI makes direct use of sensor model for information gain estimation. The fact that the sensor model for range-finders is an accurate model combined with the correlation between the map and robot pose leads to implicit pose uncertainty reduction. This result reveals the fundamental difference between mutual information approximation using a direct method (taking the expectation over future measurements) and GPs. However, it is important to note the ability of IIG using the GPVR-based functions to aggressively reduce the state estimation (map inference) uncertainty with fewer measurements (less loop-closures) without a significant undesirable effect on the robot localization.

\begin{figure}[t]
  \centering 
  \includegraphics[width=0.5\columnwidth,trim={0cm 0cm 1.5cm 1cm},clip]{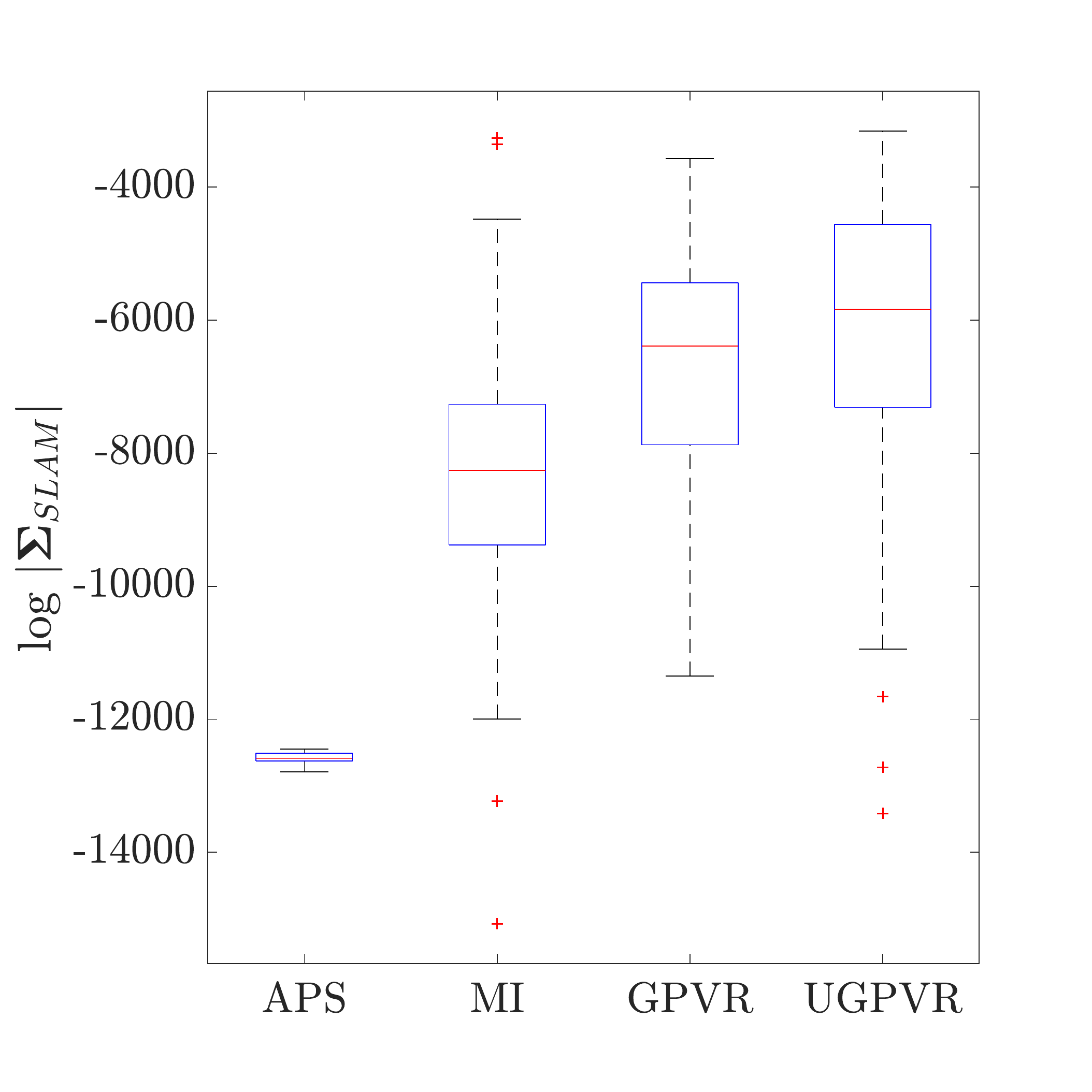}
  \caption{The statistical summary of the log of the determinant of the Pose SLAM covariance matrix at the end of each exploration run in the Cave dataset. The compared methods are, from left, APS, IIG-MI, IIG-GPVR, and IIG-UGPVR. IIG algorithms can terminate the information gathering process when there is sufficient information available for inference according to the requested saturation entropy (termination condition). Such an ability is not available in the APS framework since the robot has to complete the map based on its geometrical shape.}
  \label{fig:logdet_boxplot_cave}
\end{figure}

Figure~\ref{fig:termcond_cave} shows an example of the evolution of the average map entropy during the exploration mission. The robot terminated the exploration when the average map entropy dropped below the saturation entropy (Corollary~\ref{coro:expterm}). It is obvious that a technique with a faster map entropy reduction rate has achieved a faster convergence for the entire map exploration task. APS continues to explore until no frontier with ``significant size'' is left, which results in exhaustively searching many areas of the map multiple times. This results in a higher travel distance, more loop-closers, and lower map entropy reduction rate (higher number of steps). In contrast, IIG stops collecting measurements (information) when there is sufficient information available for inference, i.e. task completion. Therefore, IIG leads to more efficient management of robotic exploration and information gathering missions. Figure~\ref{fig:logdet_boxplot_cave} further demonstrates this effect by plotting the statistical summary of the log of the determinant of the Pose SLAM covariance matrix for the entire robot trajectory at the end of each exploration run.

\begin{figure}[t]
  \centering  
  \subfloat[]{
    \includegraphics[width=0.33\columnwidth]{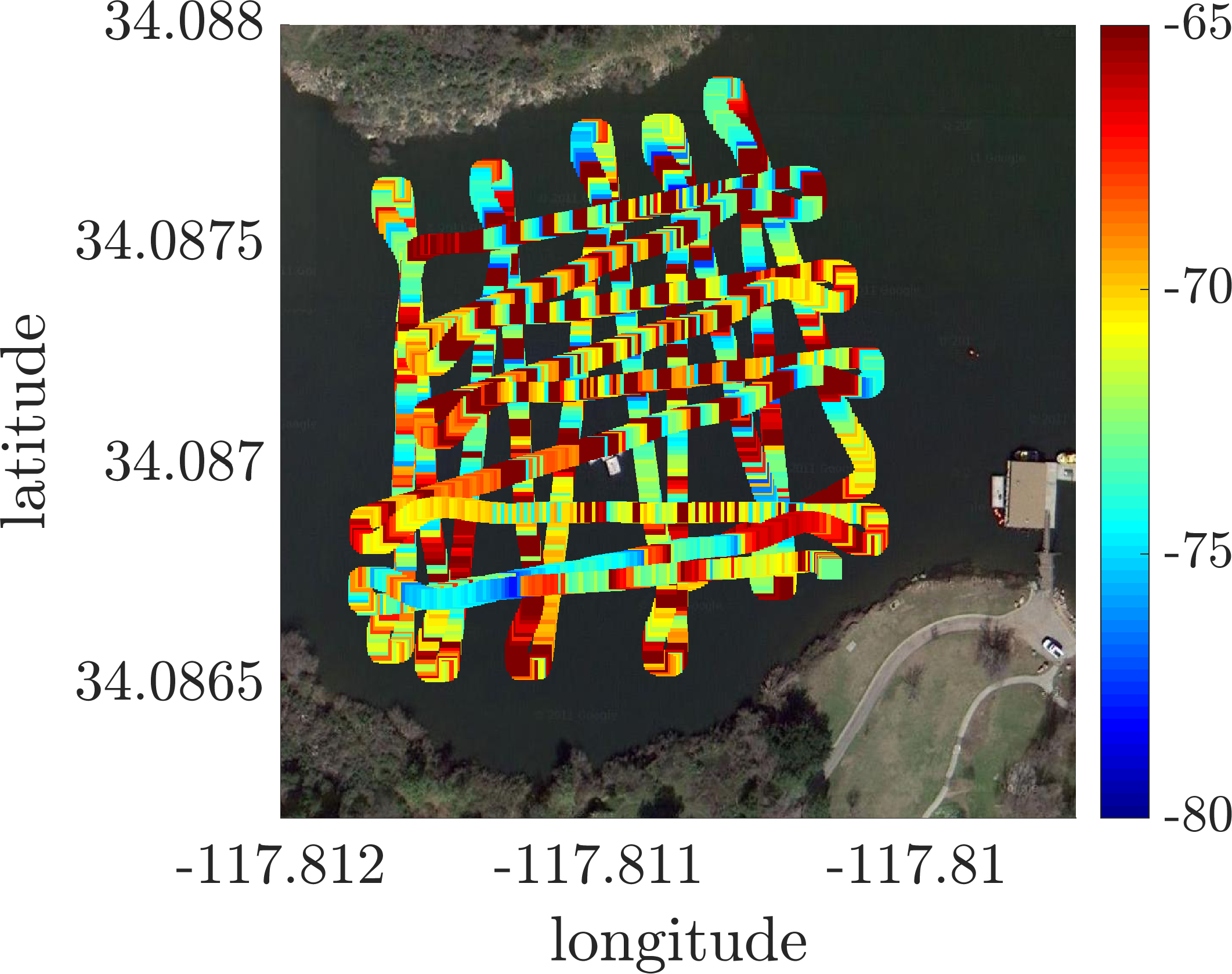}
    \label{fig:lake_sat}
    }
  \subfloat[]{
    \includegraphics[width=0.33\columnwidth]{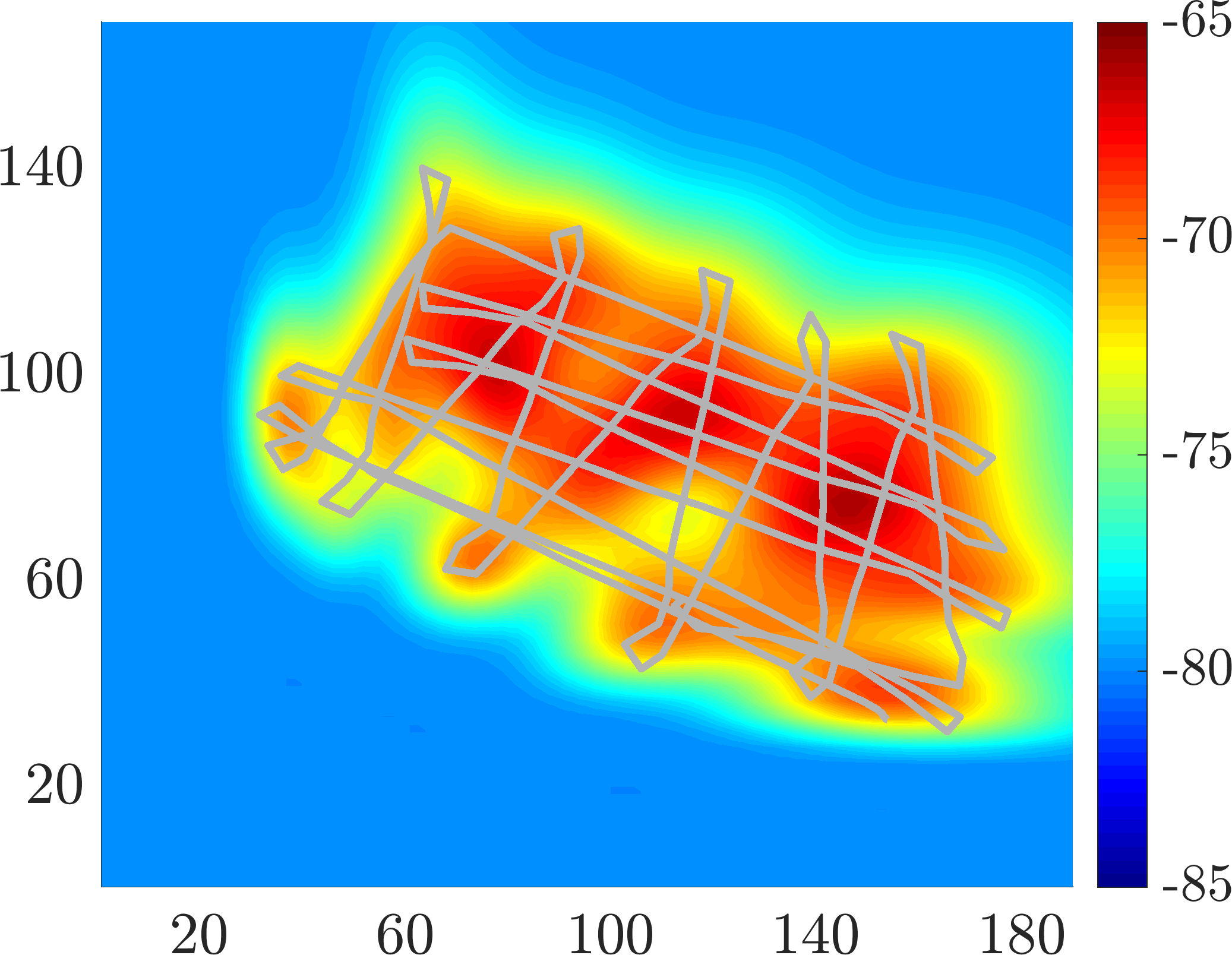}
    \label{fig:lake_gprp}
    }
  \subfloat[]{
    \includegraphics[width=0.33\columnwidth]{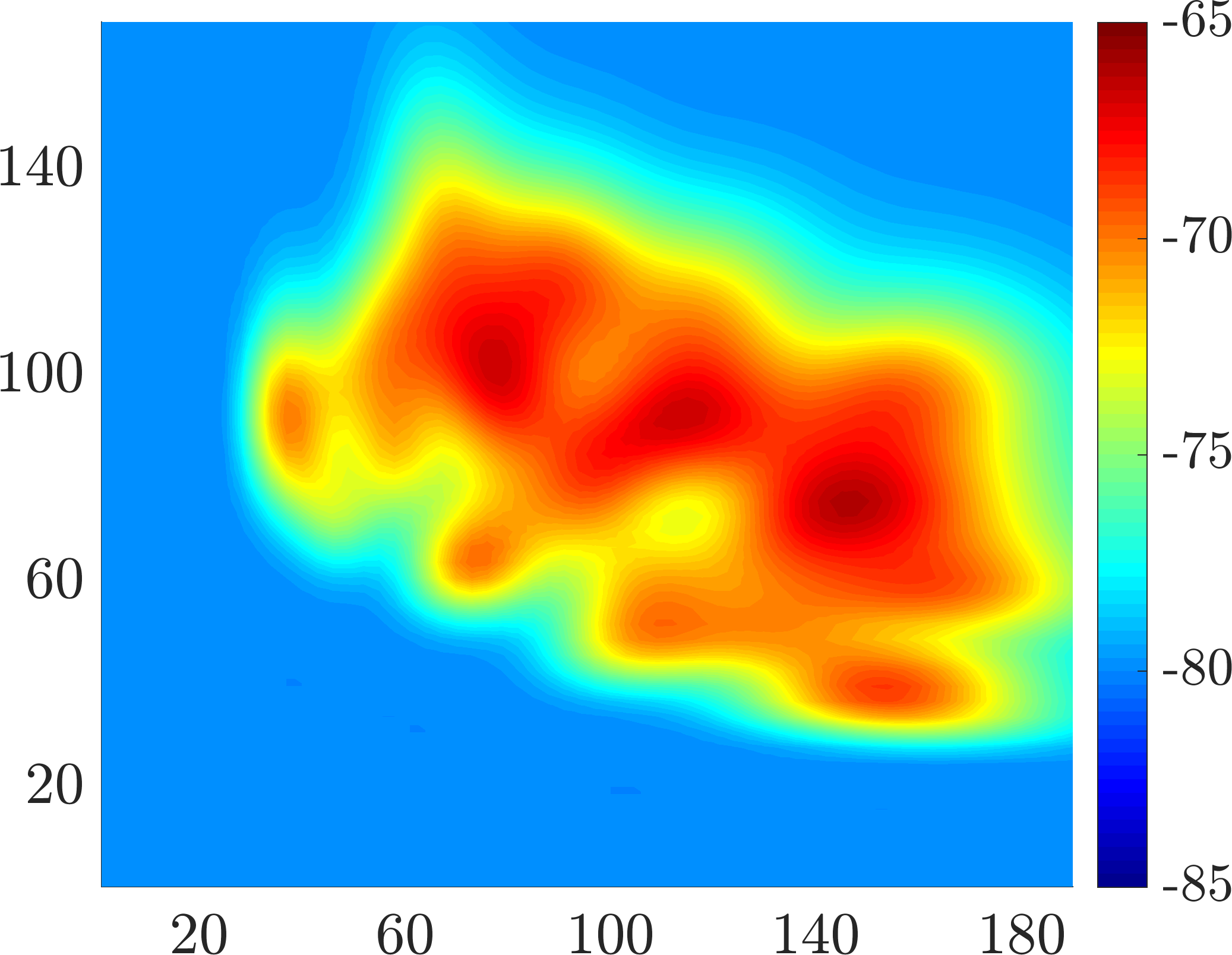}
    \label{fig:lake_gpr}
    }\\
  \subfloat[]{
    \includegraphics[width=0.33\columnwidth]{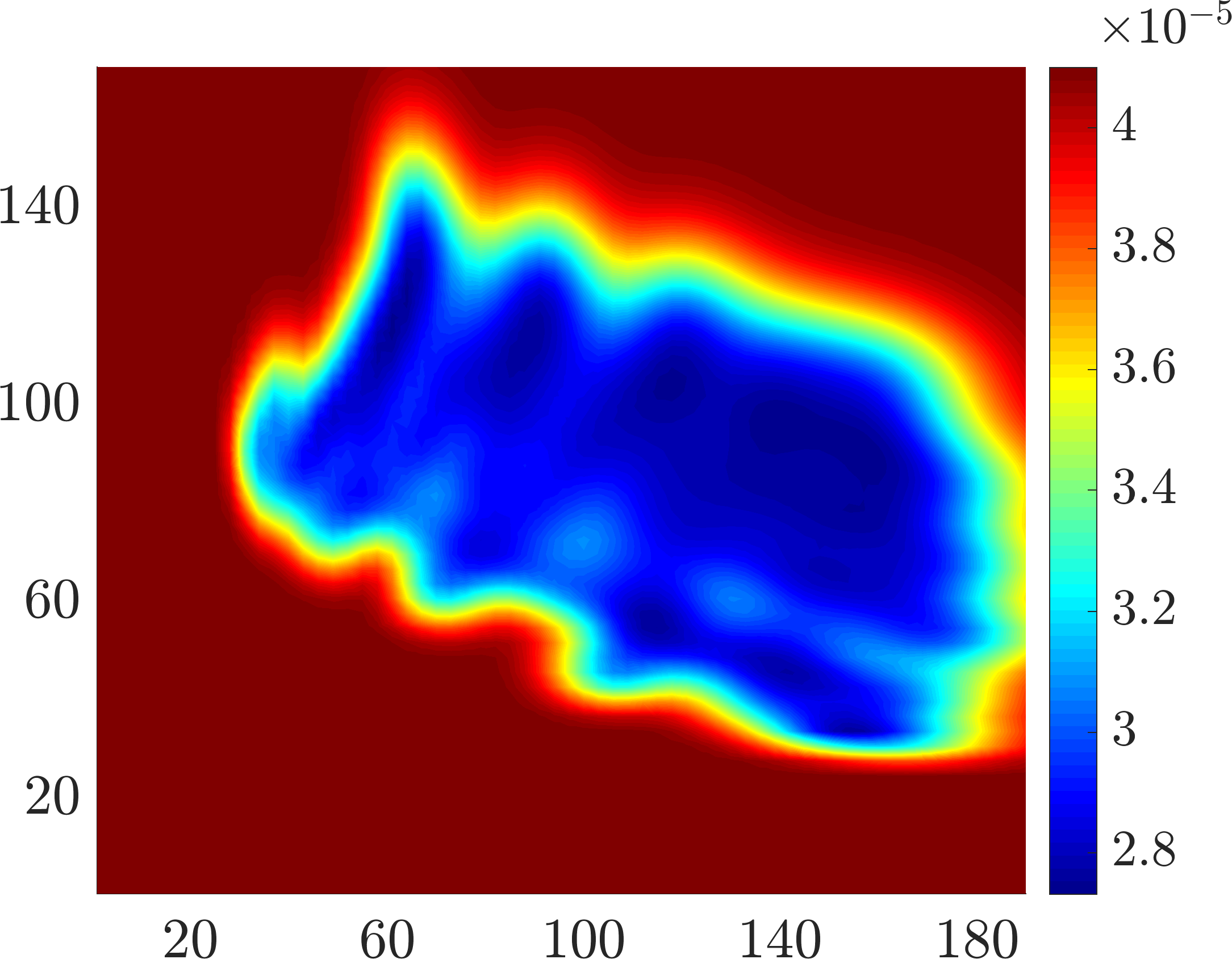}
    \label{fig:lake_cov}
    }
  \subfloat[]{
    \includegraphics[width=0.33\columnwidth]{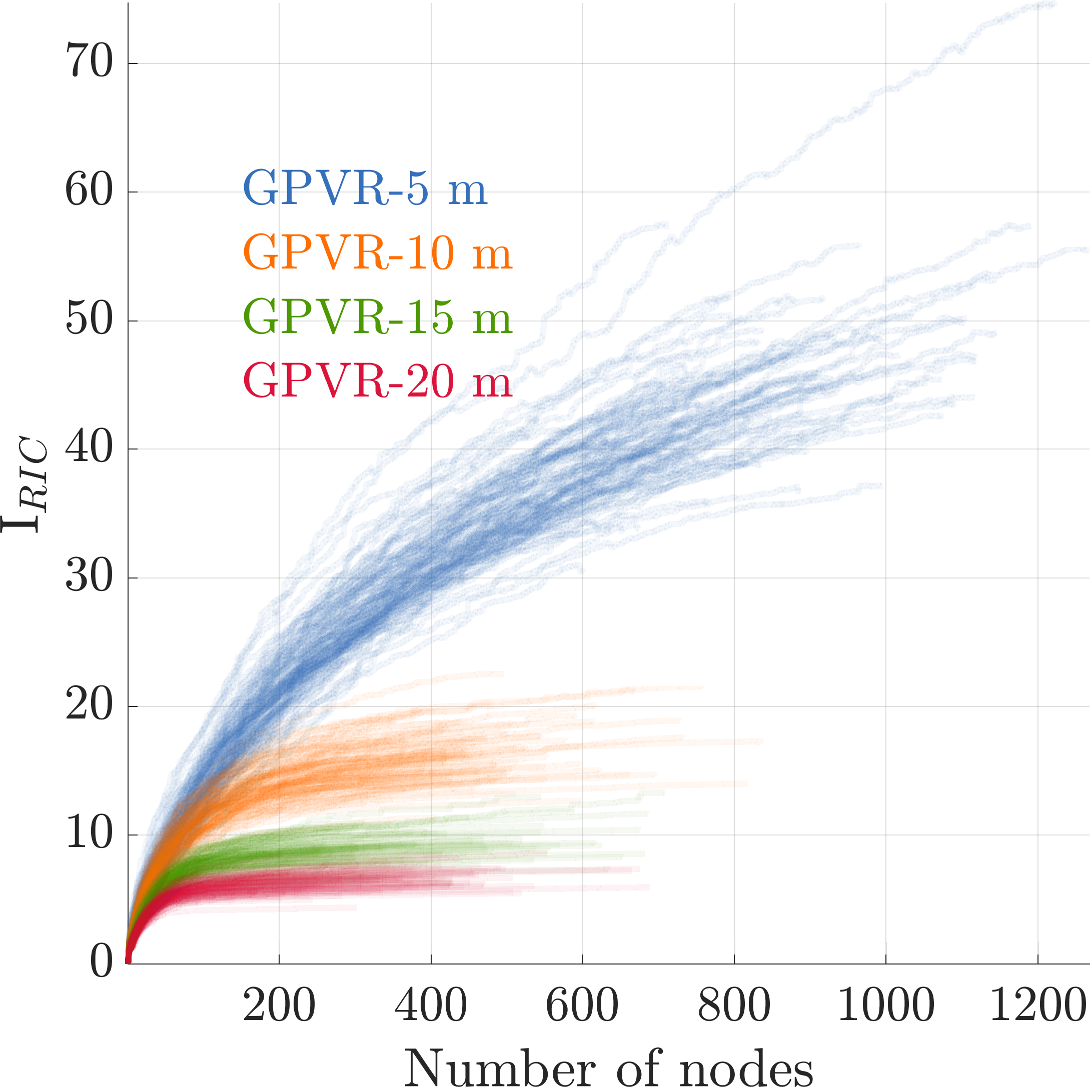}
    \label{fig:iric_gpvr_wifi}
    }
  \subfloat[]{
    \includegraphics[width=0.33\columnwidth]{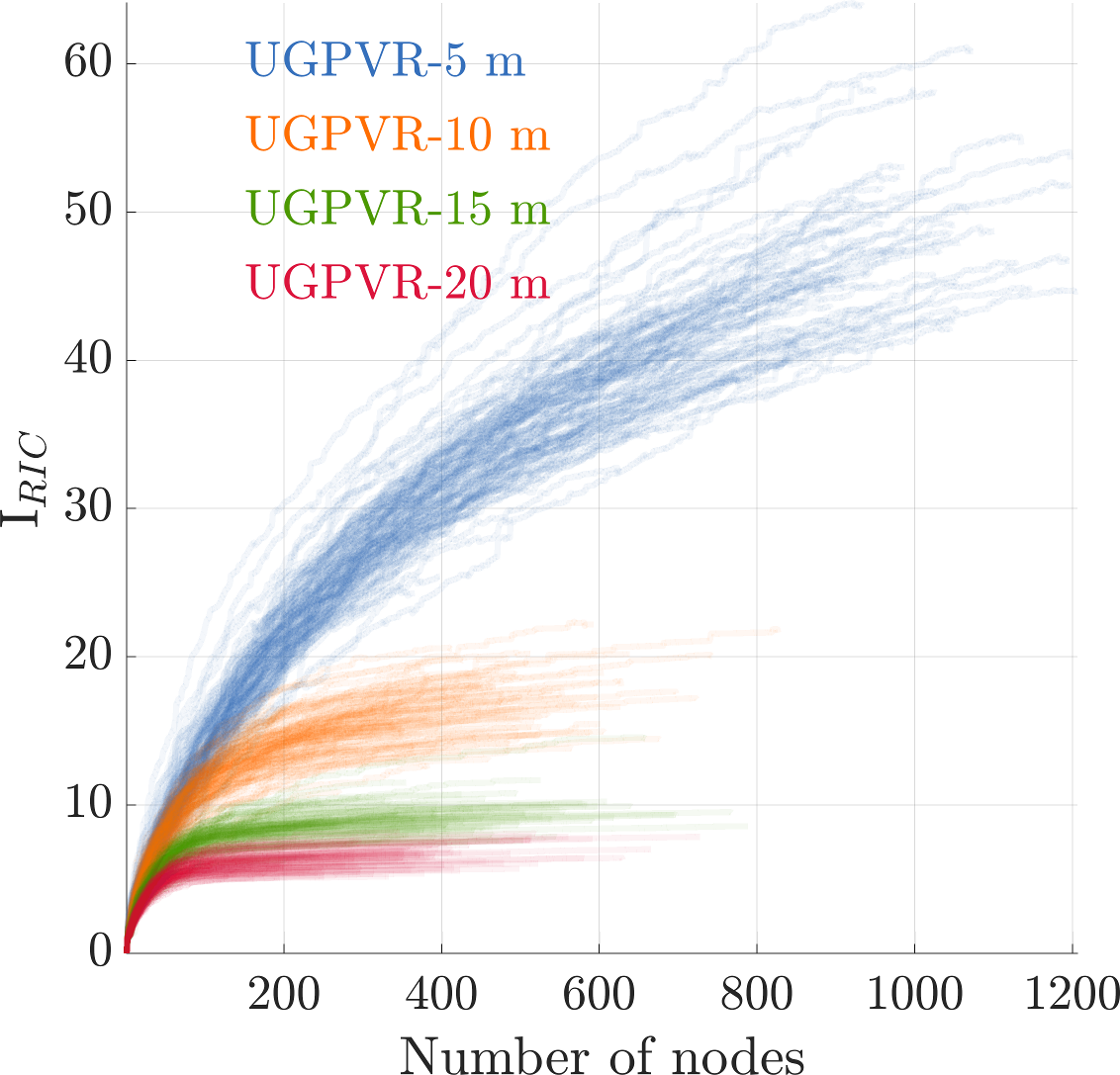}
    \label{fig:iric_ugpvr_wifi}
    }
  \caption{Lake monitoring scenario (a) satellite view of the lake and the survey trajectories using ASV, (b) the robot trajectories in metric scale on WSS map, (c) GP WSS mean surface (in $\dBm$), and (d) GP WSS covariance surface. The convergence results of IIG-GPVR and IIG-UGPVR using $5 \m$, $10 \m$, $15 \m$, and $20 \m$ sensing ranges are shown in (e) and (f), respectively. As the sensing range increases, the planner converges faster. The results are generated using $\delta_{RIC} = 5e-4$ and an unlimited budget (travel distance).}
  \label{fig:lakemap}
\end{figure}

All the proposed techniques achieve a lower number of loop-closures, while at the same time, on average, better localization and mapping performance as well as travel distance. APS explicitly searches for loop-closures by performing data association using the current measurement (laser scan) and previous scans with nearby poses available in the pose graph. This is a particularly expensive process and often limits the search area to small regions, especially, when the SLAM graph becomes less sparse due to the large number of loop-closures.
Lastly, we would like to emphasize that IIG-UGPVR is a fully integrated information-theoretic planner that takes full state estimate (the robot pose and map) uncertainty into account, it offers an automatic termination condition for the entire mission, has an adjustable planning horizon, $\delta_{RIC}$, and treats future measurements as random by taking expectations over them. In constrast, APS simplifies the information gain computation by resorting to a conditional independence assumption between the robot pose and the map. The true strength of APS for achieving comparable outcomes is through shifting the computational load towards the data association for predicting informative loop-closures along a planned trajectory. Therefore, the work presented in this paper is an alternative that is fundamentally different. Furthermore, the termination condition proposed in this article is general and can be easily incorporated into APS or other exploration techniques. For examples of qualitative results incluidng Pose SLAM, I-GPOM, and OGM plots see~\cite{jadidi2016gaussian,ghaffari2017gaussian}.

\subsection{Lake monitoring experiment}
\label{subsec:lakemon}

In this experiment, we demonstrate the performance of IIGs in a lake monitoring scenario. The ASV can localize using a GPS unit and a Doppler Velocity Log. The communication with the ground station is through a 802.11 wireless connection and at any location the Wireless Signal Strength (WSS) can be measured in $\dBm$. The Puddingstone Lake dataset is a publicly available dataset that includes about 2700 observations. The dataset is collected through a full survey of the lake area located at Puddingstone Lake in San Dimas, CA (Lat. $34.088854$\textdegree, Lon. $-117.810667$\textdegree)~\citep{hollinger2014sampling}. The objective is to find a trajectory for the robot to maintain a strong connectivity with the base station while taking physical samples in the lake.

Figure~\ref{fig:lakemap} shows the satellite view of the lake area together with the survey trajectories and regressed maps
that are used as a proxy for ground truth. Figure~\ref{fig:lake_gprp} shows the survey trajectories on the WSS surface where the longitudes and latitudes are converted to their corresponding distances using the haversine formula. Figures~\ref{fig:lake_gpr} and \ref{fig:lake_cov} illustrate the GP WSS mean and covariance surfaces, respectively. The ground truth map of WSS is built using GP regression with a constant mean function, SE covariance function with automatic relevance determination~\citep{neal1996bayesian}, and a Gaussian likelihood function which makes the exact inference possible. Furthermore, it is well-known that, in line of sight scenarios, radio signals propagation can be characterized based on Friis free space model~\citep{rappaport1996wireless,goldsmith2005wireless}. In this model, the signal attenuation is proportional to the logarithm of the distance. Therefore, to improve the regression accuracy we use logarithmic scales for input points during GP training and inference phases. The number of training points was down-sampled to 267 observations and the surface was inferred using 3648 query points. 

We store the GP output using a $k$d-tree data structure to be able to perform fast online nearest neighbor inquiries within a sensing range at any location in the map. In order to simulate the WSS monitoring experiment, at any location the robot can take measurements within a sensing range from the ground truth GP maps using \texttt{Near} function. This step is the measurement prediction in line~\ref{line:gpvrknnb} of Algorithms~\ref{alg:gpvr} and \ref{alg:ugpvr} and the training set, $\mathcal{D}$ is the output of \texttt{Near} function. As training points are already part of the map, the queried sub-map coincides with the training set, i.e. $\mathcal{D} = \mathcal{M}_{\mathcal{D}}$. For information functions, we used Algorithms~\ref{alg:gpvr} and \ref{alg:ugpvr} as they are natural choices for scenarios involving spatial phenomenas and environmental monitoring. The modified kernel in Algorithm~\ref{alg:ugpvr} was calculated using Gauss-Hermite quadrature with 11 sample points. The robot pose covariance was approximated using the uncertainty propagation and local linearization of the robot motion model~\footnote{The pose uncertainties are not part of the original dataset. Hence, we calculate them by simulation.}. 

\begin{table}[t]
\footnotesize
\centering
\caption{Lake monitoring experiments using IIG with GPVR and UGPVR information functions. For the comparison, the sensing range is varied from $5\m$ to $20\m$. The results are averaged over $100$ runs (mean $\pm$ standard error).
For all experiments $\delta_{RIC} = 5e-4$.}
\begin{tabular}{lcccc}
\toprule
Range (m)		& \multicolumn{1}{c}{5} & \multicolumn{1}{c}{10} & \multicolumn{1}{c}{15}	& \multicolumn{1}{c}{20} \\ \midrule
\multicolumn{5}{c}{IIG-GPVR} \\ \midrule
Time (sec)			& 5.26 $\pm$ 0.16		& 2.92 $\pm$ 0.08		& 2.38 $\pm$ 0.09		& \textbf{2.32 $\pm$ 0.07}		\\ 
RMSE (dBm)			& 4.54 $\pm$ 0.05		& 4.28 $\pm$ 0.06		& 3.97 $\pm$ 0.06		& \textbf{3.61 $\pm$ 0.07}		\\ 
No. of samples			& 830.6 $\pm$ 18.7		& 507.9 $\pm$ 13.0		& 437.0 $\pm$ 15.6		& \textbf{428.1 $\pm$ 12.9}		\\  
No. of nodes			& 826.1 $\pm$ 18.7		& 454.5 $\pm$ 12.8		& 359.9 $\pm$ 14.5 		& \textbf{331.6 $\pm$ 12.7}		\\  
Tot. info. gain			& 2.08e+04 $\pm$ 160		& 2.42e+04 $\pm$ 111		& 2.62e+04 $\pm$ 133 		& \textbf{2.77e+04 $\pm$ 130}	\\  
Tot. cost (Km)			& 6.77 $\pm$ 0.11		& 4.50 $\pm$ 0.09		& 3.67 $\pm$ 0.11 		& \textbf{3.42 $\pm$ 0.10}		\\  \midrule\midrule

\multicolumn{5}{c}{IIG-UGPVR} \\ \midrule 
Time (sec)			& 33.98 $\pm$ 1.30		& \textbf{25.93 $\pm$ 0.67}	& 29.26 $\pm$ 0.98		& 53.69 $\pm$ 1.61			\\ 
RMSE (dBm)			& 4.55 $\pm$ 0.05		& 4.26 $\pm$ 0.05		& 4.00 $\pm$ 0.06		& \textbf{3.56 $\pm$ 0.07}		\\ 
No. of samples			& 844.8 $\pm$ 19.8		& 496.0 $\pm$ 12.3		& 467.4 $\pm$ 15.4		& \textbf{413.0 $\pm$ 13.3}		\\ 
No. of nodes			& 839.2 $\pm$ 19.7		& 447.9 $\pm$ 12.4		& 389.1 $\pm$ 14.9 		& \textbf{325.8 $\pm$ 12.8}		\\  
Tot. info. gain			& 2.12e+04 $\pm$ 176		& 2.49e+04 $\pm$ 107		& 2.69e+04 $\pm$ 120 		& \textbf{2.80e+04 $\pm$ 113}	\\  
Tot. cost (Km)			& 6.81 $\pm$ 0.13		& 4.46 $\pm$  0.09		& 3.91 $\pm$ 0.11 		& \textbf{3.39 $\pm$ 0.10}			\\ \bottomrule
\end{tabular}
\label{tab:lakecomp}
\end{table}

\begin{figure}[t!]
  \centering  
  \subfloat{
    \includegraphics[width=0.5\columnwidth]{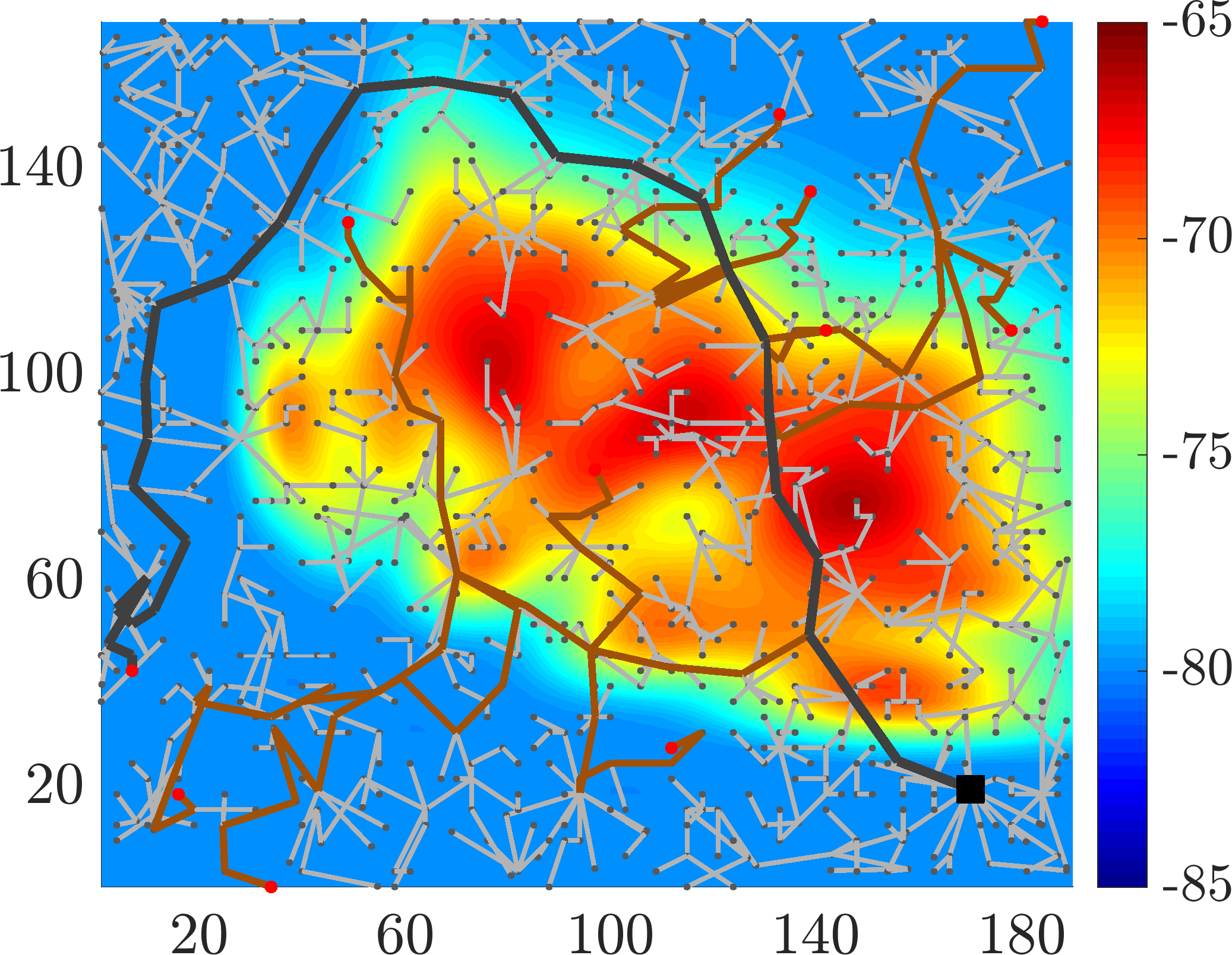}
    \label{fig:lake_gpvr_r5}
    }
  \subfloat{
    \includegraphics[width=0.5\columnwidth]{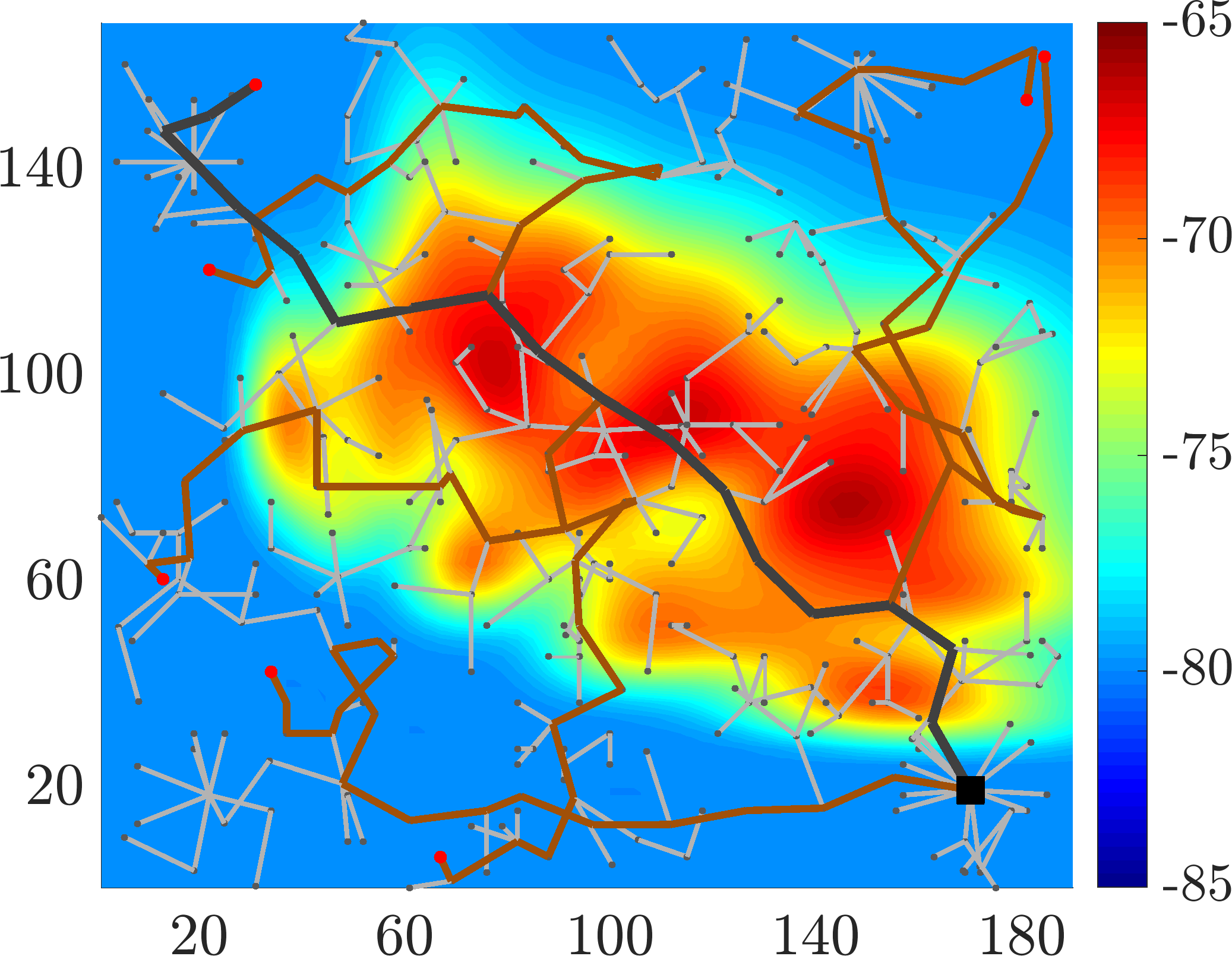}
    \label{fig:lake_gpvr_r10}
    }\\
  \subfloat{
    \includegraphics[width=0.5\columnwidth]{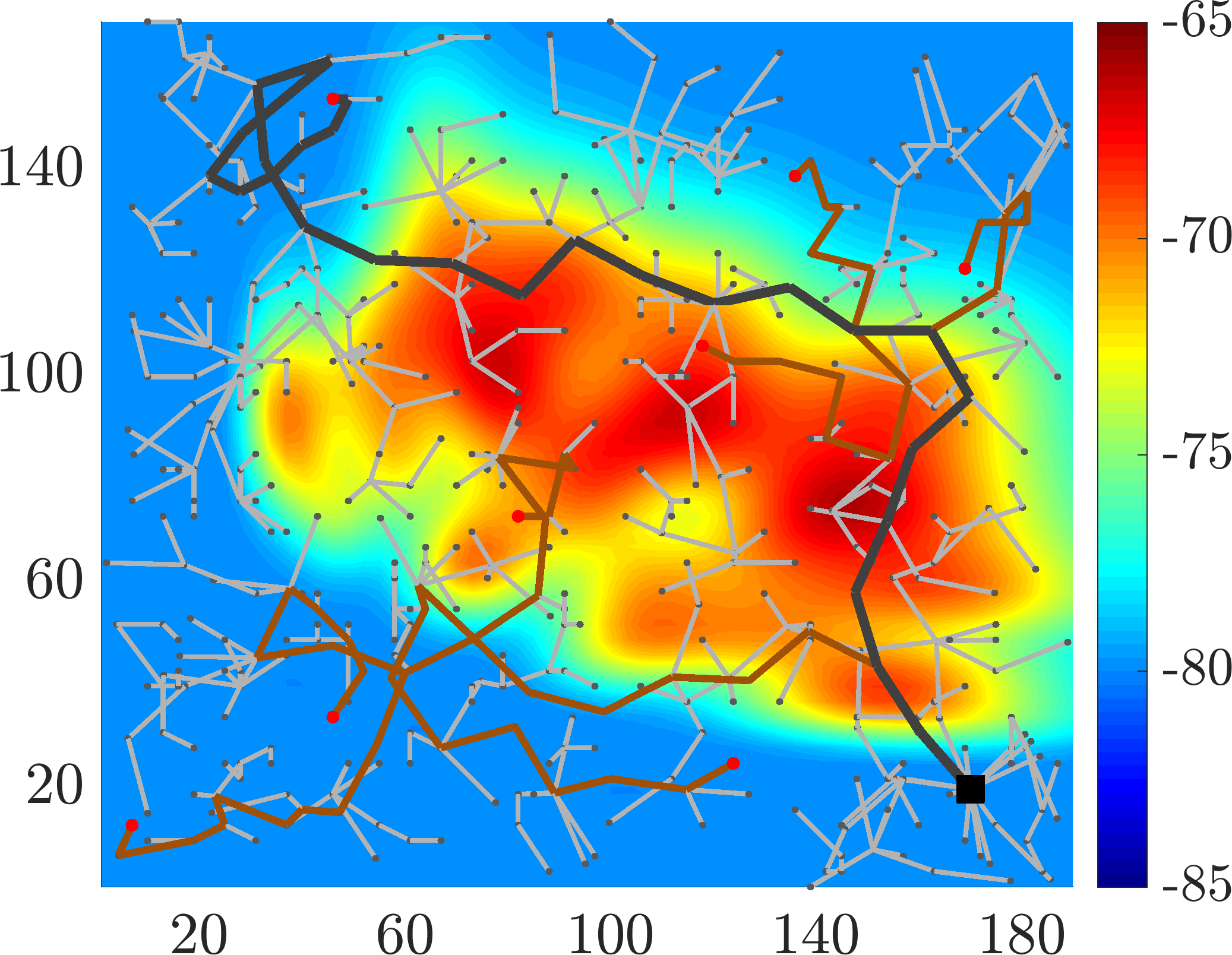}
    \label{fig:lake_gpvr_r15}
    }
  \subfloat{
    \includegraphics[width=0.5\columnwidth]{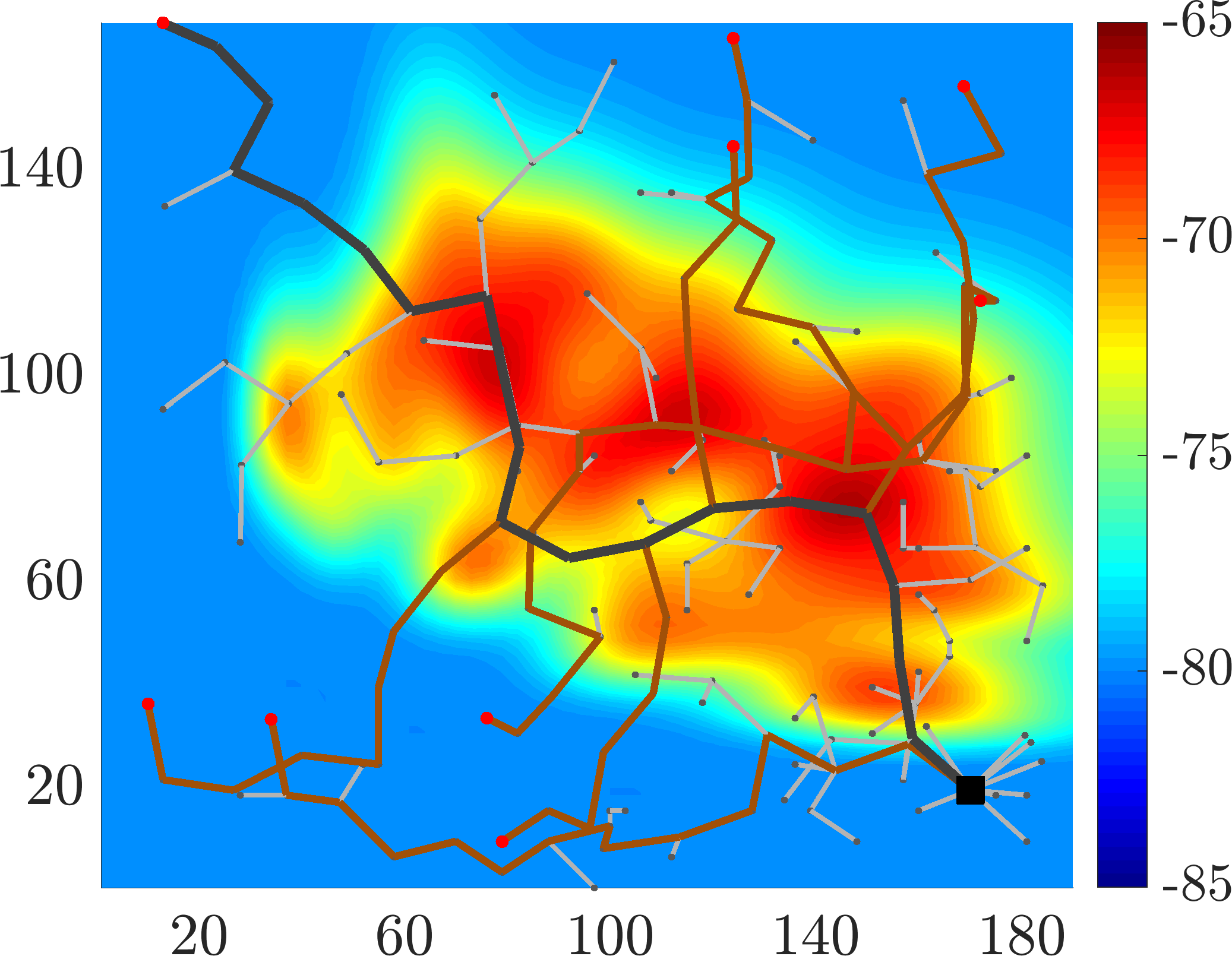}
    \label{fig:lake_gpvr_r20}
    }
  \caption{Informative motion planning for WSS monitoring in the lake area using IIG-GPVR with sensing range (a) $5\m$ (b) $10\m$ (c) $15\m$, and (d) $20\m$.}
  \label{fig:lakemap_gpvr}
\end{figure}

\begin{figure}[t!]
  \centering  
  \subfloat{
    \includegraphics[width=0.5\columnwidth]{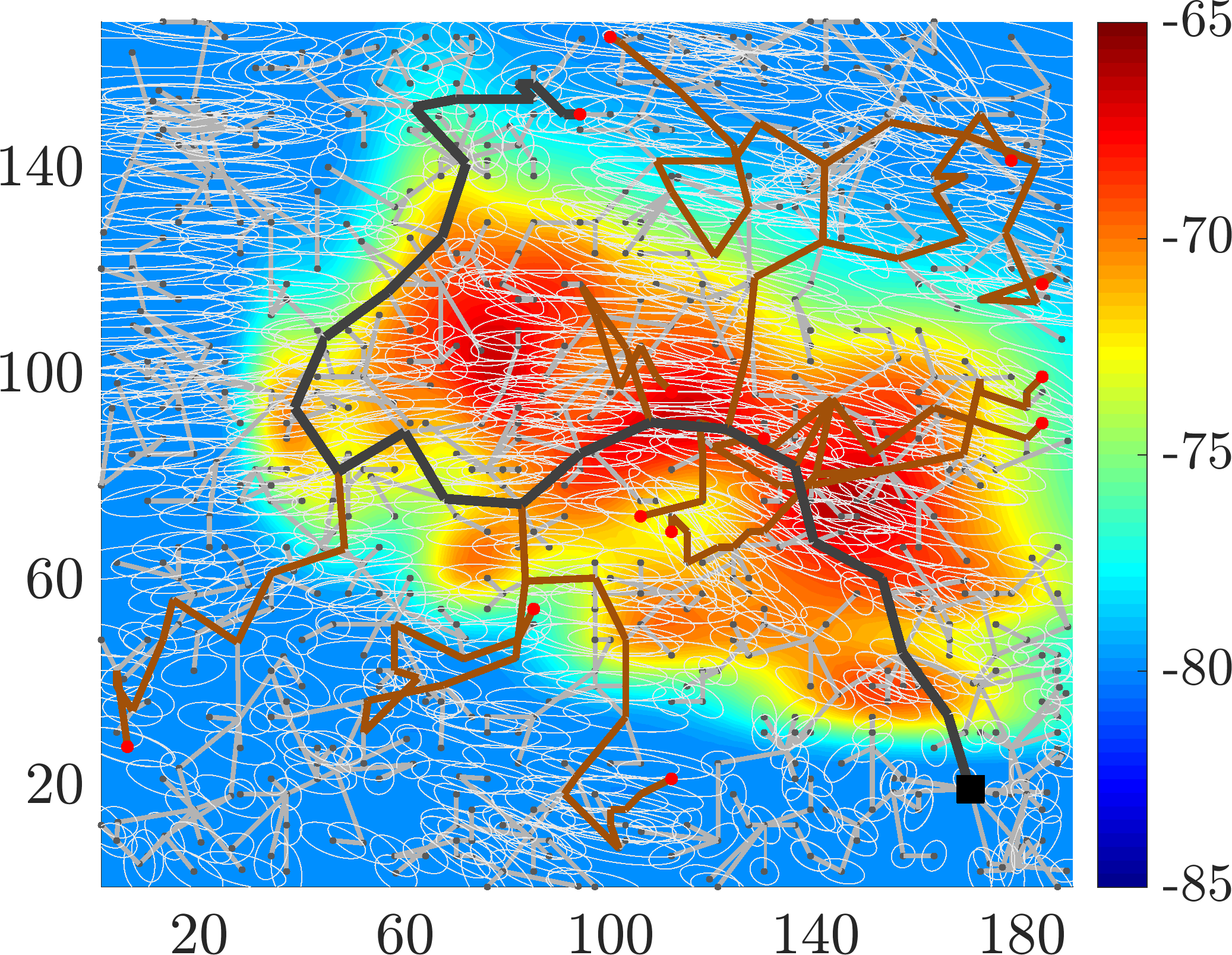}
    \label{fig:lake_ugpvr_r5}
    }
  \subfloat{
    \includegraphics[width=0.5\columnwidth]{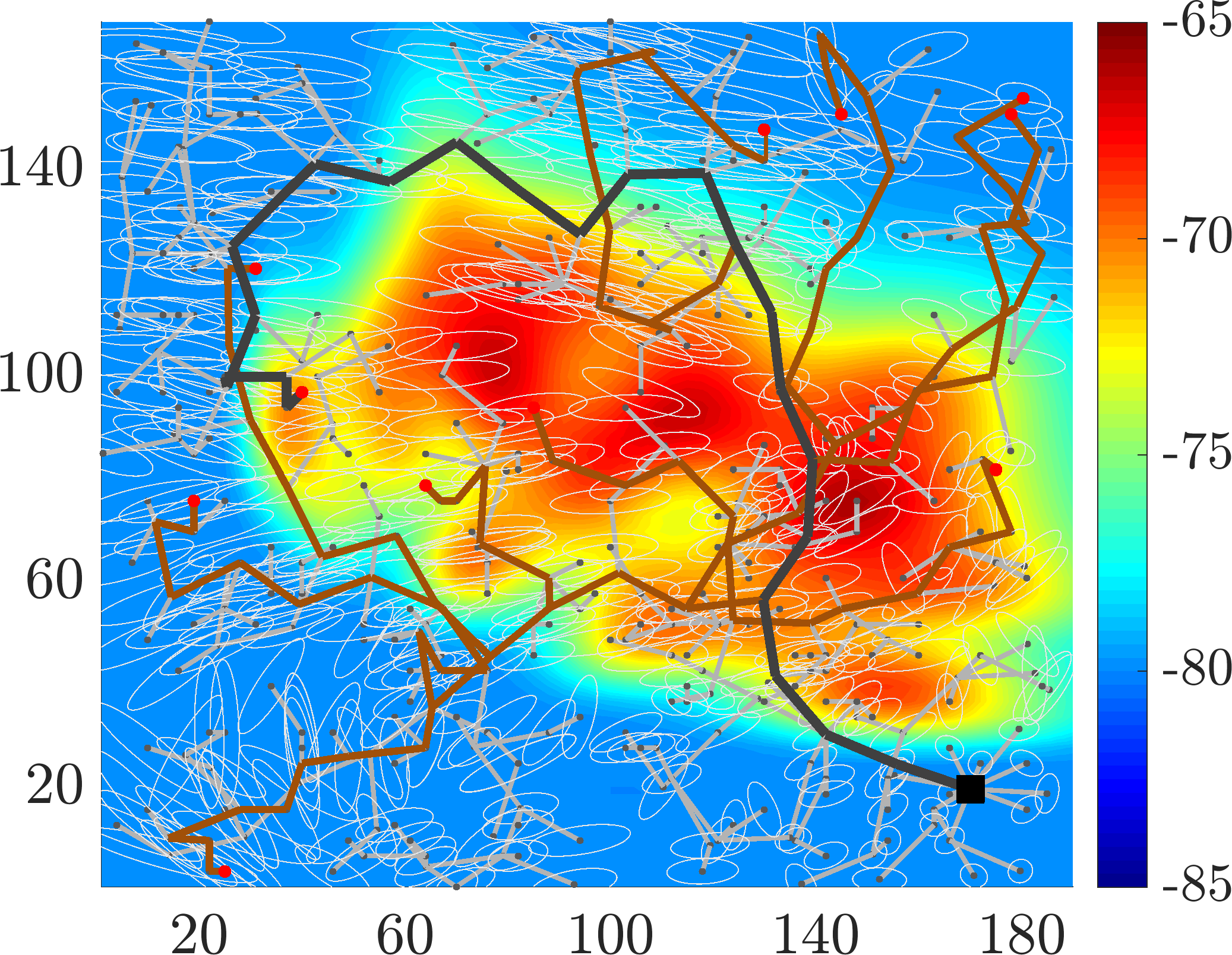}
    \label{fig:lake_ugpvr_r10}
    }\\
  \subfloat{
    \includegraphics[width=0.5\columnwidth]{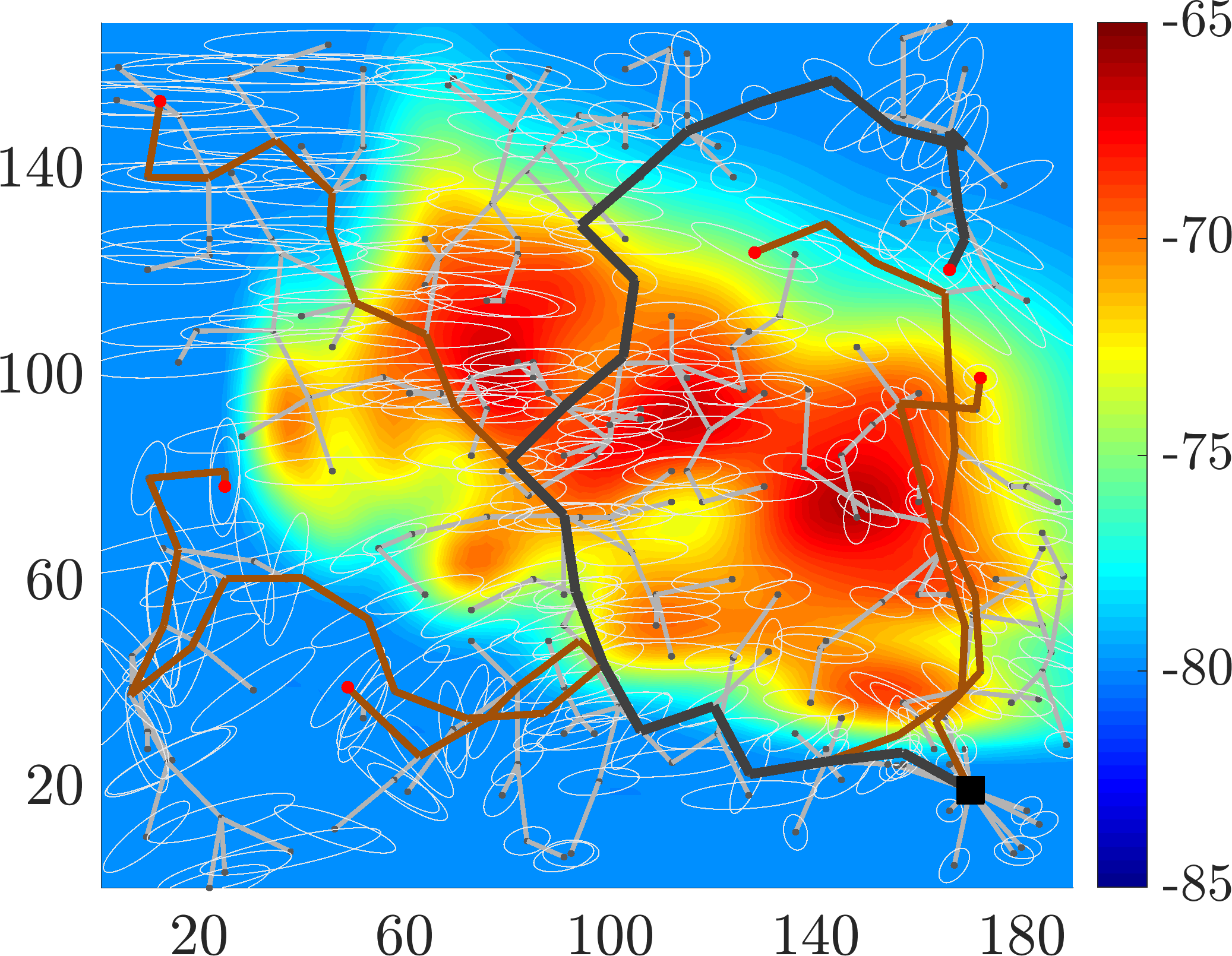}
    \label{fig:lake_ugpvr_r15}
    }
  \subfloat{
    \includegraphics[width=0.5\columnwidth]{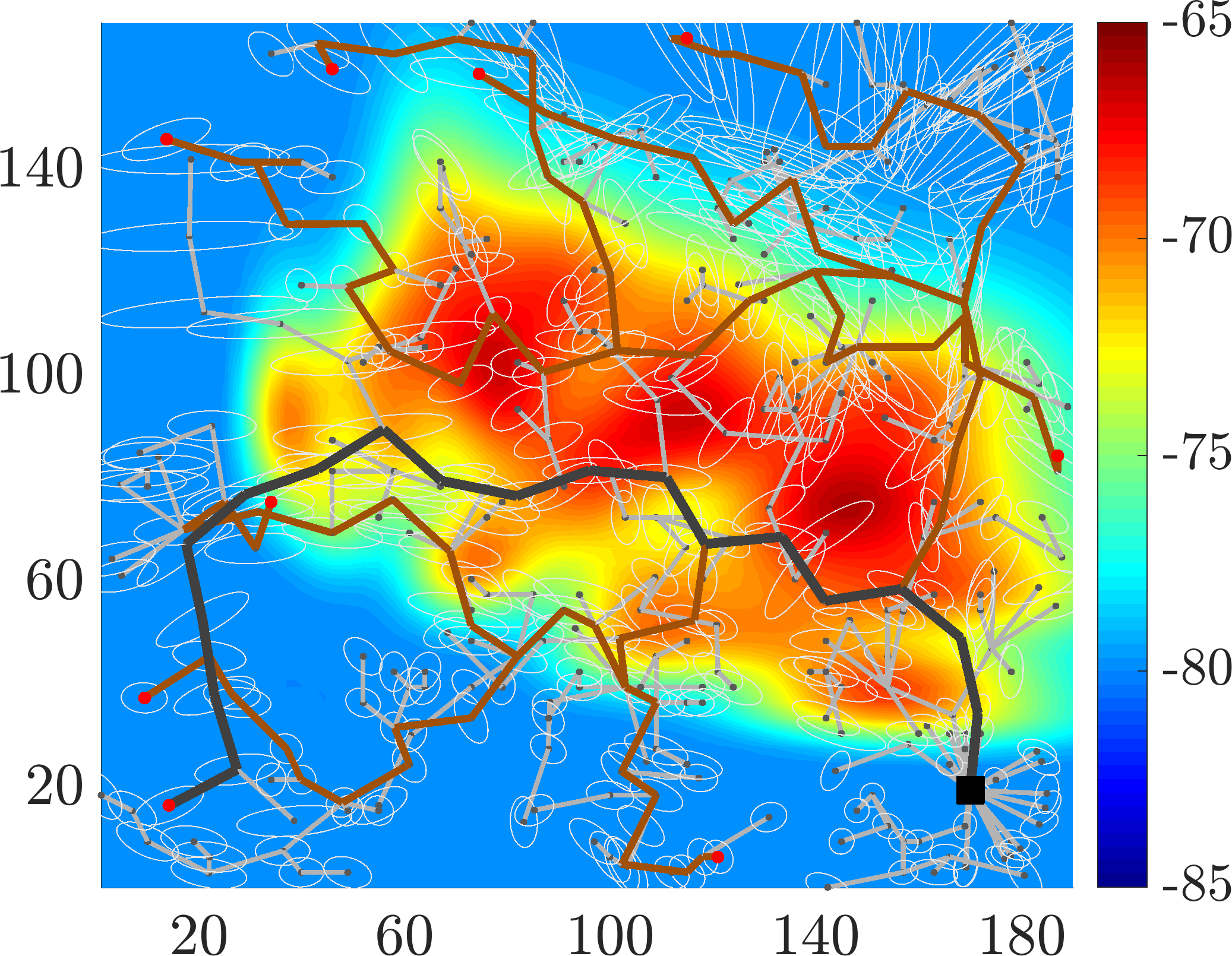}
    \label{fig:lake_ugpvr_r20}
    }
  \caption{Informative motion planning for WSS monitoring in the lake area using IIG-UGPVR with sensing range (a) $5\m$ (b) $10\m$ (c) $15\m$, and (d) $20\m$. By incorporating pose uncertainties in planning, the informative trajectories are those in which not only the wireless connectivity is strong, but the pose uncertainty along the trajectories is minimized.}
  \label{fig:lakemap_ugpvr}
\end{figure}

Each technique uses unlimited budget which is the total travel distance and continues to search the environment until convergence. Then the WSS field is reconstructed with the same resolution as the ground truth map using collected measurements along the most informative path which is extracted from the IIG tree using Algorithm~\ref{alg:pathselect}. The convergence results with $5 \m$, $10 \m$, $15 \m$, and $20 \m$ sensing ranges are shown in Figures~\ref{fig:iric_gpvr_wifi} and \ref{fig:iric_ugpvr_wifi} for 100 independent runs. Table~\ref{tab:lakecomp} shows the comparison results between IIG-GPVR and IIG-UGPVR using several criteria including RMSE. The IIG-GPVR algorithm is faster and can compute the most informative trajectory in only a few seconds. This is promising for online applications where the robot needs to replan along the trajectory. We note that, as expected from results in Subsection~\ref{subsec:inffunccomp}, by increasing the sensing range, the information gain rises and the total cost and RMSE decrease. However, in practice, a large sensing radius can be infeasible.

Figure~\ref{fig:lakemap_gpvr} shows the illustrative examples of IIG-GPVR with different sensing ranges. As the sensing range increases, the graph becomes gradually sparser. This can be understood from the fact that the information from a larger neighborhood is integrated into each node. Note that the objective is to explore the entire area while maintaining a strong wireless connectivity with the based station, therefore, the robot does not need to always travel in the area with high signal strengths. However, it travels most of the time in the areas with strong connectivity. Figure~\ref{fig:lakemap_ugpvr} shows the same scenario using IIG-UGPVR. The main difference is that the robot behaves more conservatively and tends to return to the regions with strong connectivity. The growth in pose uncertainty by traveling farther, reduces the information gain of those areas. In both IIG-GPVR and IIG-UGPVR results, we can see that by increasing the sensing range the robot tends to explore farther distances which is a natural behavior, in Figures~\ref{fig:lake_gpvr_r20} and \ref{fig:lake_ugpvr_r20}.

\subsection{Limitations and observations}
\label{subsec:limitiig}

The proposed algorithms can provide an approximate solution to the robotic exploration problem in unknown environments as a basic navigation task as well as environmental monitoring tasks. Although IIG can be implemented for online tasks, it has the limitations of its ancestors. Therefore, violating the main assumptions, such as availability of free area near any point for graph expansion, can result in failure of the task. More conceptually, most of the present robotic navigation algorithms are not truly adaptive as they commit to a decision once it is planned and replanning occurs only if the amount of computation is manageable, and often not with a high frequency. So there will be a window of time the robot acts based on the previously made decision. Although humans act in a similar fashion, we can make decisions more spontaneously (fully adaptive). This problem can be severe in environments that are highly dynamic as the robot is not as responsive as it is required.

To our experience, using a coarse approximation, i.e.{\@} less number of beams than the actual range-finder sensor results in not only faster search in the workspace and reducing the computational time, but giving the chance to more samples to be candidates as part of the potential most informative path. This is a promising feature that the algorithm works reasonably well without near exact estimation of information quality. Moreover, uniform sampling produces reasonable results in the early stage, however as the graph grows, this sampling strategy becomes less efficient. Biasing the sampling towards directions that are less explored may lead to a faster convergence.

\section{Conclusion and future Work}
\label{sec:conclusion}

In this work, we developed a sampling-based planning algorithm for incremental robotic information gathering. The proposed algorithm is based on RIG and offers an information-theoretic convergence criterion using the notion of penalized relative information contribution. We formulated the problem as an information maximization problem subject to budget and state estimate constraints. The unique feature of this approach is that it can handle dense belief representations as well as the robot pose uncertainty. We proposed MI-based and GPVR-based information functions that can be used in both RIG and IIG frameworks. The proposed IIG-tree algorithm using a UGPVR information function considers pose uncertainties in planning and provides an approximate solution for finding maximally informative paths under partially observable state variables.

We proved an information-theoretic termination condition for the entire information gathering mission based on the least upper bound of the average state variables entropy. The proposed algorithms can directly be applied to environmental monitoring and robotic exploration tasks in unknown environments. The proposed algorithms have also potential applications in solving the sensor configuration selection problem that we did not fully investigate as it is beyond the scope of this work. Furthermore, the planning horizon does not need to be set as a number of steps. Instead, the information-theoretic representation of the planning horizon used here brings more flexibility for managing the convergence of the algorithm. 

We would like to share some ideas that are natural extensions of this work. For the sake of clarity, we itemize them as follows.
\begin{itemize}
 \item[(i)] The algorithms in this article could be extended for multi-robot planning; exploiting the information gain available through wireless communication channels to manage the coordination between robots while a feedback controller satisfied kinodynamic constraints.
 \item[(ii)] We developed IIG-tree based on the RIG-tree algorithm. IIG could also be used to extend RIG-roadmap and RIG-graph. RIG-graph is asymptotically optimal and can asymptotically explore all possible budget-constraint trajectories, it produces a fully connected graph and is an interesting case to study.
 \item[(iii)] While we provide a procedural implementation of the proposed algorithms, and it suffices for research and comparison purposes, we believe the integration of the algorithms in open source libraries such as the Open Motion Planning Library (OMPL)~\citep{sucan2012open} has advantages for possible applications of this work.
 \item[(iv)] The GPVR-based algorithms could be used for active learning to model the quantity of interest online using GPs.
 \item[(v)] IIG/RIG frameworks could be used for maximizing a multi-objective information function to perform multiple tasks concurrently. In a na\"ive approach, the objective function can be defined as the sum of information functions; however, if random variables from different information functions are correlated, e.g.{\@} through the robot pose, the total information gain calculation will not be accurate. Hence, developing a more systematic approach is an interesting avenue to follow.
 \item[(vi)] A sparse covariance matrix could be constructed using the Sparse covariance function. Exploiting the sparse structure of the covariance matrix for large-scale problems could be an interesting extension of this work for online implementations.
\end{itemize}

\begin{appendices}
\section{Proof of mutual information approximation inequality}
\label{appx:miapprox}

In this Appendix, we provide the proof for Proposition~\ref{prop:miineq}. First, we present the required preliminaries and then our proof by contradiction. 
\begin{lemma}[Hadamard's Inequality]
\label{lem:hadamard}
 Let $\boldsymbol A = [\boldsymbol a_1\ \boldsymbol a_2\ \cdots\ \boldsymbol a_n]$ be an $n\times n$ matrix such that $\boldsymbol a_i \in \mathbb{R}^n$ for all $i \in \{1:n\}$. Hadamard's inequality asserts that
 \begin{equation}
  \label{eq:hadamardineq}
  \lvert \boldsymbol A \rvert \leq \prod_{i=1}^{n} \lVert \boldsymbol a_i \rVert
 \end{equation}
\end{lemma}

\begin{corollary}
 \label{cor:hadamard}
 Let $\boldsymbol S \in \mathbb{R}^{n\times n}$ be a positive definite matrix whose diagonal entries are $s_{11}, \dots, s_{nn}$. From Lemma~\ref{lem:hadamard}, it follows that 
 \begin{equation}
  \lvert \boldsymbol S \rvert \leq \prod_{i=1}^{n} s_{ii}
 \end{equation}
\end{corollary}
\begin{proof}
 Since $\boldsymbol S$ is positive definite, using Cholesky factorization, we have $\boldsymbol S = \boldsymbol L \boldsymbol L^T$ and $\lvert \boldsymbol S \rvert = \lvert \boldsymbol L \boldsymbol L^T \rvert = \lvert \boldsymbol L \rvert \lvert \boldsymbol L^T \rvert = \lvert \boldsymbol L \rvert^2$. For any diagonal elements of $\boldsymbol S$, we have $s_{ii} = \boldsymbol l_i \boldsymbol l_i^T = \lVert \boldsymbol l_i \rVert^2$ where $\boldsymbol l_i$ is the $i$-th column of $\boldsymbol L$. Using the Hadamard's inequality, it follows that
 \begin{equation}
  \lvert \boldsymbol S \rvert = \lvert \boldsymbol L \rvert^2 \leq \prod_{i=1}^{n} \lVert \boldsymbol l_i \rVert^2 = \prod_{i=1}^{n} s_{ii}
 \end{equation}
\end{proof}

We now derive a few useful inequalities before the final proof. Let \mbox{$X \sim \mathcal{N}(\boldsymbol \mu_X,\boldsymbol \Sigma_X)$} and \mbox{$X\mid Z \sim \mathcal{N}(\boldsymbol \mu_{X\mid Z},\boldsymbol \Sigma_{X\mid Z})$} be the prior and posterior distribution of the $n$-dimensional random vector $X$, respectively. Let $h(X)=\frac{1}{2}\log((2\pi e)^n \lvert \boldsymbol \Sigma_X \rvert)$ and $h(X\mid Z)=\frac{1}{2}\log((2\pi e)^n \lvert \boldsymbol \Sigma_{X\mid Z} \rvert)$ be the differential entropy and conditional entropy, respectively; and $\hat{I}(X;Z) = \hat{h}(X) - \hat{h}(X\mid Z)$ the corresponding entropies and mutual information by applying Proposition~\ref{prop:miapprox}, where $\hat{h}(X)=\frac{1}{2}\log((2\pi e)^n\prod_{i=1}^{n} \sigma_{X_i})$ and $\hat{h}(X\mid Z)=\frac{1}{2}\log((2\pi e)^n\prod_{i=1}^{n} \sigma_{X_i\mid Z})$. Using Corollary~\ref{cor:hadamard}, the following inequalities are immediate:
\begin{equation}
    h(X) \leq \hat{h}(X)
  \end{equation}
  \begin{equation}
   h(X\mid Z) \leq \hat{h}(X\mid Z)
  \end{equation}
  \begin{equation}
  \label{eq:miineq1}
   I(X;Z) \leq \hat{h}(X) - h(X\mid Z)
  \end{equation}
  \begin{equation}
  \label{eq:miineq2}
   h(X) - \hat{h}(X\mid Z) \leq \hat{I}(X;Z)
  \end{equation}
  \begin{equation}
  \label{eq:miineq3}
   h(X) - \hat{h}(X\mid Z) \leq \hat{h}(X) - h(X\mid Z)
  \end{equation}
  and summing~\eqref{eq:miineq1} and~\eqref{eq:miineq2} leads to
  \begin{equation}
  \label{eq:miineq4}
   I(X;Z) \leq \hat{I}(X;Z) + \xi(X;Z)
  \end{equation}
  where \mbox{$\xi(X;Z) \triangleq \hat{h}(X) + \hat{h}(X\mid Z) - h(X) - h(X\mid Z)$} and $\xi(X;Z) \ge 0$, with equality if and only if $X$ contains uncorrelated random variables, i.e.\@ $\hat{h}(X) = h(X)$ and $\hat{h}(X\mid Z) = h(X\mid Z)$.

\begin{proof}[Proof of Proposition~\ref{prop:miineq}]
 Suppose $\hat{I}(X;Z) \ge I(X;Z)$. By subtracting this inequality from~\eqref{eq:miineq4} we have
 \begin{equation}
  \nonumber \underbrace{I(X;Z) - \hat{I}(X;Z)}_{-a^2} \le \underbrace{\hat{I}(X;Z) - I(X;Z)}_{a^2} + \xi(X;Z)
 \end{equation}
 \begin{equation}
  \nonumber \xi(X;Z) \ge -2a^2
 \end{equation}
since $\xi(X;Z) \ge 0$, $a = 0$ and $\hat{I}(X;Z) = I(X;Z)$ are the only possibilities. Now suppose $\hat{I}(X;Z) \le I(X;Z)$; by summing this inequality and~\eqref{eq:miineq4} we have
 \begin{equation}
  \nonumber I(X;Z) + \hat{I}(X;Z) \le \hat{I}(X;Z) + I(X;Z) + \xi(X;Z)
 \end{equation}
 \begin{equation}
  \nonumber \nonumber \xi(X;Z) \ge 0 
 \end{equation}
 which shows that information can only be lost and not gained. 
\end{proof}

\end{appendices}

\bibliographystyle{SageH}
{\small
\bibliography{references}}

\end{document}